\documentclass[10pt,journal,compsoc]{IEEEtran}
\usepackage{amsmath}
\usepackage{amsfonts}
\usepackage{amsthm}
\usepackage{booktabs}
\usepackage{subcaption}
\usepackage{dsfont}
\usepackage{multirow}
\usepackage{enumerate}
\usepackage{pifont}
\usepackage{graphicx}
\usepackage{bbm}
\usepackage{color}
\usepackage{xcolor}
\usepackage{wrapfig}

\newcommand{\tfont}[1]{{\fontfamily{lmtt}\selectfont \textbf{#1}}}

\newcommand{\diag}{\textbf{diag\,}}
\newcommand{\re}{\textbf{Re\,}}

\newcommand{\head}{u}
\newcommand{\tail}{v}
\newcommand{\heademb}{\textbf{u}}
\newcommand{\headmat}{\textbf{U}}
\newcommand{\headconjemb}{\overline{\textbf{u}}_i}
\newcommand{\tailemb}{\textbf{v}}
\newcommand{\tailmat}{\textbf{V}}
\newcommand{\timemat}{\textbf{T}}

\newtheorem{definition}{Definition}
\newtheorem{theorem}{Theorem}
\newtheorem{lemma}{Lemma}

\newif\ifupdate\updatetrue

\newcommand{\modifyy}[1]{\ifupdate{\color{black}#1}\else{#1}\fi}

\newcommand{\zqzhang}[1]{{\color{black}#1}}
\newcommand{\zhshi}[1]{{\color{black}#1}}

\ifCLASSOPTIONcompsoc
  \usepackage[nocompress]{cite}
\else
  \usepackage{cite}
\fi
\ifCLASSINFOpdf
\else
\fi
\usepackage{url}

\begin{document}
%
\title{Duality-Induced Regularizer for Semantic Matching Knowledge Graph Embeddings}
%
%
%
%

\author{Jie~Wang,
        \textit{Senior Member, IEEE},
        Zhanqiu~Zhang,
        Zhihao~Shi,
        Jianyu~Cai,\\
        Shuiwang~Ji, 
        \textit{Senior Member, IEEE},
        and~Feng~Wu,
        \textit{Fellow, IEEE}
\IEEEcompsocitemizethanks{
\IEEEcompsocthanksitem J. Wang, Z. Zhang, Z. Shi, J. Cai, F. Wu are with: a) CAS Key Laboratory of Technology in GIPAS, University of Science and Technology of China, Hefei 230027, China; b) Institute of Artificial Intelligence, Hefei Comprehensive National Science Center, Hefei 230091, China. E-mail: jiewangx@ustc.edu.cn, zqzhang@mail.ustc.edu.cn, zhihaoshi@mail.ustc.edu.cn, jycai@mail.ustc.edu.cn, fengwu@ustc.edu.cn.
\IEEEcompsocthanksitem S. Ji is with the Department of Computer Science and Engineering, Texas A\&M University, College Station, TX77843 USA. E-mail: sji@tamu.edu.
}

\thanks{Manuscript received August, 2021.}}

%
%

\markboth{IEEE TRANSACTIONS ON PATTERN ANALYSIS AND MACHINE INTELLIGENCE,Vol. XX, No. X, August~2021}%
{Shell \MakeLowercase{\textit{et al.}}: Bare Demo of IEEEtran.cls for Computer Society Journals}
%



\IEEEtitleabstractindextext{%
\begin{abstract}
Semantic matching models---which assume that entities with similar semantics have similar embeddings---have shown great power in knowledge graph embeddings (KGE). Many existing semantic matching models use inner products in embedding spaces to measure the plausibility of triples and quadruples in static and temporal knowledge graphs. 
However, vectors that have the same inner products with another vector can still be orthogonal to each other, which implies that entities with similar semantics may have dissimilar embeddings. This property of inner products significantly limits the performance of semantic matching models.
To address this challenge, we propose a novel regularizer---namely, \textbf{DU}ality-induced \textbf{R}egul\textbf{A}rizer (DURA)---which effectively encourages the entities with similar semantics to have similar embeddings. The major novelty of DURA is based on the observation that, for an existing semantic matching KGE model (\textit{primal}), there is often another distance based KGE model (\textit{dual}) closely associated with it, which can be used as effective constraints for entity embeddings.
Experiments demonstrate that DURA consistently and significantly improves the performance of state-of-the-art semantic matching models on both static and temporal knowledge graph benchmarks.
\end{abstract}

\begin{IEEEkeywords}
Knowledge Graph Embeddings, Knowledge Graph, Link Prediction, Temporal Knowledge Graphs, Regularization.
\end{IEEEkeywords}
}

\maketitle

\IEEEdisplaynontitleabstractindextext

%
\IEEEpeerreviewmaketitle

\IEEEraisesectionheading{\section{Introduction}\label{sec:introduction}}

%
%
%
%

\IEEEPARstart{K}{nowledge}  graphs store human knowledge in a structured way, with nodes being entities and edges being relations. In the past few years, knowledge graphs have made great achievements in many areas, such as natural language processing \cite{ernie}, question answering \cite{KGQA}, recommendation systems \cite{KGRS}, and computer vision \cite{kg_cv}.

Some popular knowledge graphs, such as Freebase \cite{freebase}, Wikidata \cite{wikidata}, and Yago3 \cite{yago3}, consists of triples (subject, predicate, object). However, in many scenarios, triples cannot well represent knowledge as knowledge may evolve with time. For example, the fact (Donald Trump, president, USA) is only valid at the time interval 2017-2021. Therefore, temporal knowledge graphs that consists of quadruples (subject, predicate, object, timestamp) also attract increasing attention recently.
Although commonly used knowledge graphs usually contain billions of triples or quadruples, they still suffer from the incompleteness problem that a lot of factual triples or quadruples are missing. Due to the large scale of knowledge graphs, it is impractical to find all valid triples or quadruples manually. Therefore, knowledge graph completion (KGC)---which aims to predict missing links between entities based on known links automatically---has attracted much attention.

Knowledge graph embedding (KGE) is a powerful technique for KGC, which embeds entities, relations, and possible timestamps in KGs into low-dimensional continuous embedding spaces. KGE models define a score function for each triple/quadruple in embedding spaces. Valid triples/quadruples are expected to have higher scores than invalid ones. The success of many existing KGE models relies on an assumption that entities with similar semantics should have similar representations. For example, the triples $(\tfont{lions}, \tfont{is}, \tfont{mammals})$ and $(\tfont{tigers}, \tfont{is}, \tfont{mammals})$ will have similar scores if the entities $\tfont{lions}$ and $\tfont{tiger}$ have similar embeddings. In other words, given a known valid triple $(\tfont{lions}, \tfont{is}, \tfont{mammals})$, we can infer that the triple $(\tfont{tigers}, \tfont{is}, \tfont{mammals})$ is also valid with high confidence.

Semantic matching (SM) models are an important suite of knowledge graph embedding approaches. They usually model static and temporal knowledge graphs as partially observed third-order and fourth-order binary tensors, respectively. Then, they formulate KGC as a tensor completion problem and try to solve it by tensor factorization. 
From the perspective of scoring function, SM models use inner products between embeddings to define the plausibility of given triples and quadruples. 
Theoretically, these models are highly expressive and can well handle complex relation patterns, such as 1-N, N-1, and N-N relations \cite{transh,transr}. However, due to the property of inner products, entities with similar semantics may have dissimilar embeddings. For example, even if the embeddings of $\tfont{lions}$ and $\tfont{tiger}$ are orthogonal, the triples $(\tfont{lions}, \tfont{is}, \tfont{mammals})$ and $(\tfont{tigers}, \tfont{is}, \tfont{mammals})$ are likely to have the same scores when the two entity embeddings have proper norms (please refer to Section \ref{sec:motivation} for more detailed explanations). Similarly, we can observe the phenomenon in temporal semantic matching models (please refer to Section \ref{sec:motivation_temporal}). The performance of SM models suffer from this phenomenon and thus cannot achieve state-of-the-art.

To tackle the challenge, we propose a novel regularizer for semantic matching KGE models---namely, \textbf{DU}ality-induced \textbf{R}egul\textbf{A}rizer (DURA)---which effectively encourages entities with similar semantics to have similar embeddings. Specifically, by noticing that entities with similar semantics often connect to the same entity through the same relation, we propose to use a regularizer to constrain these entities' embeddings. Further, we propose DURA based on the observation called \textit{duality}---for an existing semantic matching KGE model (\textit{primal}), there is often another distance based KGE model closely associated with it (\textit{dual}), which uses Minkowski distances to define score functions. The duality can be derived by expanding the squared score functions of the associated distance based models, where the cross-term in the expansion is exactly a SM model and the squared terms in it give us a regularizer. In other words, adding DURA to a SM model is equivalently to a SM model with its associative distance based model being regularization. Since the minimization of the Minkowski distances between vectors force these vectors to be the same, DURA indeed encourages entities with similar semantics to have similar embeddings. Moreover, we show that when relation embeddings are diagonal matrices, the minimization of DURA has a form of the tensor nuclear 2-norm \cite{nunorm}. Experiments demonstrate that DURA is widely applicable to various static and temporal semantic matching KGE models and yields consistent and significant improvements on benchmark datasets.

An earlier version of this paper has been published at NeurIPS 2020 \cite{dura}. This journal manuscript significantly extends the initial version in several aspects. First, we propose a temporal version of DURA, named TDURA, which is designed for temporal knowledge graph embeddings. Second, we introduce a general semantic matching temporal KGE model, named TRESCAL, which is an extension of RESCAL to temporal KGs. Third, we rigorously analyze the diagonal relation embedding matrices and show that the minimization of TDURA also has a form of the tensor nuclear 2-norm \cite{nunorm}. Finally, we conduct experiments on temporal knowledge graphs and conduct more ablation studies to demonstrate the effectiveness of TDURA.

\section{Preliminaries}
We review the backgrounds of the knowledge graph, knowledge graph completion, and knowledge graph embeddings in Sections \ref{sec:bg_kg}, \ref{sec:bg_kgc}, and \ref{sec:bg_kge}, respectively. Then, we introduce the notations in Section \ref{sec:notation}.

\subsection{Knowledge Graphs}\label{sec:bg_kg}
\noindent \textbf{Static Knowledge Graphs}\hspace{1.5mm} Given a set $\mathcal{E}$ of entities and a set $\mathcal{R}$ of relations, a knowledge graph $\mathcal{K}=\{(\head_i,r_j,\tail_k)\}\subset \mathcal{E}\times\mathcal{R}\times\mathcal{E}$ is a set of triplets, where $\head_i$ and $r_j$ are the $i$-th entity and $j$-th relation, respectively. 

\noindent \textbf{Temporal Knowledge Graphs}\hspace{1.5mm} Given a set $\mathcal{E}$ of entities, a set $\mathcal{R}$ of relations, and a set $\mathcal{T}$ of timestamp, a temporal knowledge graph $\mathcal{K}_T=\{(\head_i,r_j,\tail_k,t_l)\}\subset \mathcal{E}\times\mathcal{R}\times\mathcal{E}\times\mathcal{T}$ is a set of quadruples, where $\head_i$, $r_j$ and $t_l$ are the $i$-th entity, $j$-th relation and $l$-th timestamp, respectively. 

\subsection{Knowledge Graph Completion}\label{sec:bg_kgc}
\noindent\textbf{Static Knowledge Graph Completion (KGC)} \hspace{1.5mm} The goal of KGC is to predict valid but unobserved triples based on the known triples in $\mathcal{K}$, which is often solved by knowledge graph embedding (KGE) models. KGE models associate each entity $\head_i\in\mathcal{E}$ and relation $r_j\in\mathcal{R}$ with an embedding (may be real or complex vectors, matrices, and tensors) $\heademb_i$ and $\textbf{r}_j$. Generally, they define a score function $s: \mathcal{E}\times\mathcal{R}\times\mathcal{E}\rightarrow\mathbb{R}$ to associate a score $s(\head_i,r_j,\tail_k)$ with each potential triplet $(\head_i,r_j,\tail_k)\in\mathcal{E}\times\mathcal{R}\times\mathcal{E}$. The scores measure the plausibility of triples. For a query $(\head_i,r_j,?)$, KGE models first fill the blank with each entity in the knowledge graphs and then score the resulted triples. Valid triples are expected to have higher scores than invalid triples. 

\noindent\textbf{Temporal Knowledge Graph Completion (TKGC)}\hspace{1.5mm} The goal of TKGC is to predict valid but unobserved quadruples based on the known quadruples in $\mathcal{K}_T$. TKGE models also contain two important categories: distance based models and semantic matching models, both of which are temporal knowledge graph embedding (TKGE) methods. TKGE models also associate each entity $\head_i\in\mathcal{E}$, relation $r_j\in\mathcal{R}$, and timestamp $t_l\in\mathcal{T}$ with an embedding (may be real or complex vectors, matrices, and tensors) $\heademb_i$, $\textbf{r}_j$ and $\textbf{t}_l$. Generally, they define a score function $s: \mathcal{E}\times\mathcal{R}\times\mathcal{E}\times\mathcal{T}\rightarrow\mathbb{R}$ to associate a score $s(\head_i,r_j,\tail_k,t_l)$ with each potential quadruples $(\head_i,r_j,\tail_k,t_l)\in\mathcal{E}\times\mathcal{R}\times\mathcal{E}\times\mathcal{T}$. The scores measure the plausibility of quadruples. For a query $(\head_i,r_j,?, t_l)$, TKGE models first fill the blank with each entity in the temporal knowledge graphs and then score the resulted quadruples. Valid quadruples are expected to have higher scores than invalid quadruples.

\subsection{Knowledge Graphs Embeddings}\label{sec:bg_kge}
Knowledge graph embeddings aim to embed entities, relations, and possible timestamps in static and temporal knowledge graphs into low-dimensional continuous vector spaces, which have two important categories---distance based models and semantic matching models.

\noindent\textbf{Distance Based (DB) Models for Static KGs}\hspace{1.5mm} DB models define the score function $s$ with Minkowski distances. That is, the score functions have the formulation of $s(\head_i,r_j,\tail_k)=-\|\Gamma(\head_i,r_j,\tail_k)\|_p$, where $\Gamma$ is a model-specific function.

\noindent\textbf{Semantic Matching (SM) Models for Static KGs}\hspace{1.5mm} SM models regard a knowledge graph as a third-order binary tensor $\mathcal{X}\in\{0,1\}^{|\mathcal{E}|\times|\mathcal{R}|\times|\mathcal{E}|}$. The $(i,j,k)$ entry $\mathcal{X}_{ijk}=1$ if $(\head_i, r_j, \tail_k)$ is valid otherwise $\mathcal{X}_{ijk}=0$. Suppose that $\mathcal{X}_j$ denotes the $j$-th frontal slice of $\mathcal{X}$, i.e., the adjacency matrix of the $j$-th relation. Generally, a SM KGE model factorizes $\mathcal{X}_j$ as $\mathcal{X}_j\approx\re(\overline{\headmat} \textbf{R}_j\tailmat^\top)$, where the $i$-th ($k$-th) row of $\headmat$ ($\tailmat$) is $\heademb_i$ ($\tailemb_k$), $\textbf{R}_j$ is a matrix representing relation $r_j$, $\re(\cdot)$ and $\overline{\cdot}$ are the real part and the conjugate of a complex matrix, respectively. That is, the score functions are defined as $s(\head_i,r_j,\tail_k)=\re(\headconjemb\textbf{R}_j\tailemb_k^\top)$. Note that the real part and the conjugate of a real matrix are itself.
The aim of SM models is to seek matrices $\headmat, \textbf{R}_1,\dots,\textbf{R}_{|\mathcal{R}|},\tailmat$, such that $\re(\overline{\headmat} \textbf{R}_j\tailmat^\top)$ can approximate $\mathcal{X}_j$.
Let $\hat{\mathcal{X}}_j=\re(\overline{\headmat} \textbf{R}_j\tailmat^\top)$ and $\hat{\mathcal{X}}$ be a tensor of which the $j$-th frontal slice is $\hat{\mathcal{X}}_j$. Then the regularized formulation of a semantic matching model can be written as 
\begin{align}\label{prob:t_optim}
    \min_{\hat{\mathcal{X}}_1,\dots,\hat{\mathcal{X}}_{|\mathcal{R}|}}\sum_{j=1}^{|\mathcal{R}|} L(\mathcal{X}_j,\hat{\mathcal{X}}_j)+\lambda g(\hat{\mathcal{X}}),
\end{align}
where $L(\mathcal{X}_j,\hat{\mathcal{X}}_j)$ measures the discrepancy between $\mathcal{X}_j$, $\hat{\mathcal{X}}_j$, and $g$ is the regularization function, and $\lambda>0$ is a fixed parameter.

\noindent\textbf{Distance Based (DB) Models for Temporal KGs}\hspace{1.5mm} Temporal DB models define the score function $s$ with the Minkowski distance. That is, the score functions have the formulation of $s(\head_i,r_j,\tail_k,t_l)=-\|\Gamma(\head_i,r_j,\tail_k,t_l)\|_p$, where $\Gamma$ is a model-specific function.

\noindent\textbf{Semantic Matching (SM) Models for Temporal KGs}\hspace{1.5mm} Temporal SM models regard a temporal knowledge graph as a fourth-order binary tensor $\mathcal{X}\in\{0,1\}^{|\mathcal{E}|\times|\mathcal{R}|\times|\mathcal{E}|\times|\mathcal{T}|}$. The $(i,j,k,l)$ entry $\mathcal{X}_{ijkl}=1$ if $(\head_i, r_j, \tail_k, t_l)$ is valid otherwise $\mathcal{X}_{ijkl}=0$. 
Suppose that $\mathcal{X}_{jl}$ denotes the $(j, l)$ slice of $\mathcal{X}$, i.e., the adjacency matrix given the $j$-th relation and the $l$-th timestamp. Generally, a SM temporal KGE model factorizes $\mathcal{X}_{jl}$ as $\mathcal{X}_{jl}\approx\re(\overline{\headmat} (\textbf{R}_j\odot \timemat_l)\tailmat^\top)$, where the $i$-th ($k$-th) row of $\headmat$ ($\tailmat$) is $\heademb_i$ ($\tailemb_k$), $\textbf{R}_j$ is a matrix representing relation $r_j$, $\timemat_l$ is a matrix representing timestamp $t_l$, $\re(\cdot)$ and $\overline{\cdot}$ are the real part and the conjugate of a complex matrix, and $\odot$ is the element-wise multiplication between two matrices. 
The aim of temporal SM models is to seek matrices $\headmat, \textbf{R}_1,\dots,\textbf{R}_{|\mathcal{R}|},\tailmat$ and $\timemat_1,\dots,\timemat_{|\mathcal{T}|}$, such that $\re(\overline{\headmat} (\textbf{R}_j\odot \timemat_l)\tailmat^\top)$ can approximate $\mathcal{X}_{jl}$.
Let $\hat{\mathcal{X}}_{jl}=\re(\overline{\headmat} (\textbf{R}_j\odot \timemat_l)\tailmat^\top)$ and $\hat{\mathcal{X}}$ be a tensor of which the $(j, l)$ slice is $\hat{\mathcal{X}}_{jl}$. Then the regularized formulation of a semantic matching model can be written as 
\begin{align}\label{prob:optim}
    \min_{\hat{\mathcal{X}}_{jl}}\sum_{j=1}^{|\mathcal{R}|}\sum_{l=1}^{|\mathcal{T}|} L(\mathcal{X}_{jl},\hat{\mathcal{X}}_{jl})+\lambda g(\hat{\mathcal{X}}),
\end{align}
where $\lambda>0$ is a fixed parameter, $L(\mathcal{X}_{jl},\hat{\mathcal{X}}_{jl})$ measures the discrepancy between $\mathcal{X}_{jl}$ and $\hat{\mathcal{X}}_{jl}$, and $g$ is the regularization function.

\subsection{Other Notations}\label{sec:notation}
We use $\head_i\in\mathcal{E}$ and $\tail_k\in\mathcal{E}$ to distinguish head and tail entities. Let $\|\cdot\|_1$, $\|\cdot\|_2$, and $\|\cdot\|_F$ denote the $L_1$ norm, the $L_2$ norm, and the Frobenius norm of matrices or vectors. We use $\langle \cdot, \cdot\rangle$ to represent the inner products of two real or complex vectors. 
\zqzhang{
\section{Related Work}\label{sec:related_work}
This work is related to knowledge graph embedding models and regularizer for semantic mathcing knowledge graph embeddings.

\subsection{Knowledge Graph Embedding Models}

Popular knowledge graph embeddings include distance-based, learning-based, and semantic matching models.

Distance based models describe relations as relational maps between head and tail entities. Then, they use the Minkowski distance to measure the plausibility of a given triplet or quadruple. Lower distances correspond to higher plausibility of given triplets or quadruples. 
For example, TransE \cite{transe} and its variants \cite{transh,transr} represent relations as translations in vector spaces. They assume that a valid triplet $(\head_i, r_j, \tail_k)$ satisfies $\heademb_{i,r_j}+\textbf{r}_j\approx \tailemb_{k,r_j}$, where $\heademb_{i,r_j}$ and $\tailemb_{k,r_j}$ mean that entity embeddings may be relation-specific.
Structured embedding (SE) \cite{se} uses linear maps to represent relations. Its score function is defined as $s(\head_i,r_j,\tail_k)=-\|\textbf{R}_j^1\heademb_i-\textbf{R}_j^2\tailemb_k\|_1$. RotatE \cite{rotate} defines each relation as a rotation in a complex vector space and the score function is defined as $s(\head_i,r_j,\tail_k)=-\|\heademb_i\circ \textbf{r}_j-\tailemb_k\|_1$, where $\heademb_i, \textbf{r}_j, \tailemb_k\in\mathbb{C}^k$ and $|[\textbf{r}]_i|=1$. ModE \cite{hake}
assumes that $\textbf{R}_j^1$ is diagonal and $\textbf{R}_j^2$ is an identity matrix. It shares a similar score function $s(\head_i,r_j,\tail_k)=-\|\heademb_i\circ \textbf{r}_j-\tailemb_k\|_1$ with RotatE but $\heademb_i, \textbf{r}_j, \tailemb_k\in\mathbb{R}^k$. Some works extend existing distance-based models to temporal KGs. For example, TTransE \cite{ttranse} extends TransE by learning timestamp embeddings. Inspired by TransH \cite{transh}, HyTE \cite{hyte} projects entity and relation embeddings into the timestamp-specific hyperplanes. 

Learning-based models use neural networks---such as convolutional neural networks \cite{conve}, capsule networks \cite{capse}, and graph neural networks \cite{kbgat}---to model the interaction among triplets or quadruples. Then, they use either inner products or distance functions to measure the plausibility of given triplets or quadruples.

Semantic matching models usually formulate the KGC task as a third-order binary tensor completion problem. 
RESCAL \cite{rescal} factorizes the $j$-th frontal slice of $\mathcal{X}$ as $\mathcal{X}_j\approx \textbf{A}\textbf{R}_j\textbf{A}^\top$. That is, embeddings of head and tail entities are from the same space. As the relation specific matrices contain lots of parameters, RESCAL is prone to be overfitting. DistMult \cite{distmult} simplifies the matrix $\textbf{R}_j$ in RESCAL to be diagonal, while it sacrifices the expressiveness of models and can only handle symmetric relations. In order to model asymmetric relations, ComplEx \cite{complex} extends DistMult to complex embeddings. Both DistMult and ComplEx can be regarded as variants of canonical decomposition (CP) \cite{cp}, which are in real and complex vector spaces, respectively. Some works also extend semantic mathcing models to temporal KGs. For example, ConT \cite{cont} extends TuckER \cite{tucker} by additionally learning timestamp embeddings. ASALSAN \cite{asalsan} use three-way DEDICOM \cite{dedicom} to express temporal relations, which shares similar ideas with RESCAL \cite{rescal}. Besides, TComplEx and TNTComplEx \cite{tcomplex} reduce memory complexity by extending ComplEx to decomposition of the 4-order tensor.

Compared with distance-based models, semantic matching knowledge graph embeddings are more time-efficient, as we can easily parallel the computation of linear transformations and inner products. Moreover, due to the property of inner products, semantic matching is good at modeling one-to-many/many-to-one/many-to-many relations. For example, given the entity embedding $\heademb$ and relation embedding $\textbf{r}$, we can find different entity embeddings $\tailemb_i$ such that $(\head,r,\tail_i)$ have the same scores. In order to handle complex relations, distance-based models usually turn to relation-specific projection and further increase the computational burden.
Compared with learning-based model, semantic matching knowledge graph embeddings are more time-efficient thanks to their simpler formulations.

However, also due to the property of inner products, in semantic matching, entities with similar semantics may have dissimilar embeddings, which will degrade the knowledge graph completion performance. As introduced in Section 4.1, even if the embeddings of $\tfont{Laue}$ and $\tfont{Schottky}$ are orthogonal, the triples $(\tfont{Laue}, \tfont{StudiedIn}, \tfont{HU Berlin})$ and $(\tfont{Schottky}, \tfont{StudiedIn}, \tfont{HU Berlin})$ are likely to have the same scores when the two entity embeddings have proper norms. Moreover, distance-based models are better at modeling properties of knowledge graphs than semantic matching. For example, RotatE \cite{rotate} can handle composition of relations; HAKE \cite{hake} and MuRP \cite{murp} can model semantic hierarchies.

\modifyy{
It is worth noting that many general embedding learning methods are closely related to knowledge graph embeddings, all of which aim to optimize similarity between objects. For example, some graph embedding models \cite{node2vec,deepwalk,line} define similarity functions between nodes, and more generally, some feature embedding approaches \cite{general_emb,van2018representation,simclr} learn discriminative feature representations by optimizing similarity between samples. Different from these works, knowledge graph embeddings need to model complex relations between entities, leading to a more challenging problem.
}

\subsection{Regularizers for Semantic Matching Knowledge Graph Embeddings}
Semantic matching (SM) KGE models usually suffer from overfitting problem seriously, which motivates various regularizers. In the original papers of SM models, the authors usually use the squared Frobenius norm ($L_2$ norm) regularizer \cite{rescal,distmult,complex}. This regularizer cannot bring satisfying improvements. Consequently, SM models do not gain comparable performance to distance based models \cite{rotate,hake}. More recently, \cite{n3} propose to use the tensor nuclear 3-norm \cite{nunorm} (N3) as a regularizer, which brings more significant improvements than the squared Frobenius norm regularizer. However, it is designed for the CP-like models, such as CP and ComplEx, and not suitable for more general models such as RESCAL. Moreover, some regularization methods aim to leverage external background knowledge. For example, to model equivalence and inversion axioms, \cite{bg_reg} impose a set of model-dependent soft constraints on the predicate embeddings. \cite{simp_cons} use non-negativity constraints on entity embeddings and approximate entailment constraints on relation embeddings to impose prior beliefs upon the structure of the embeddings space.

}

\section{DURA for Static KGE}\label{sec:sdura}
We introduce a novel regularizer---\textbf{DU}ality-induced \textbf{R}egulariz\textbf{A}er (\textbf{DURA})---for semantic matching knowledge graph embeddings. We first introduce the motivations in Section \ref{sec:motivation} and then introduce a basic version of DURA in Section \ref{sec:basic_dura}. Finally, we introduce the formulation of DURA for static KGE in Section \ref{sec:dura}.

\begin{figure}[t]
    \centering 
\begin{subfigure}{0.85\columnwidth}
  \includegraphics[width=220pt]{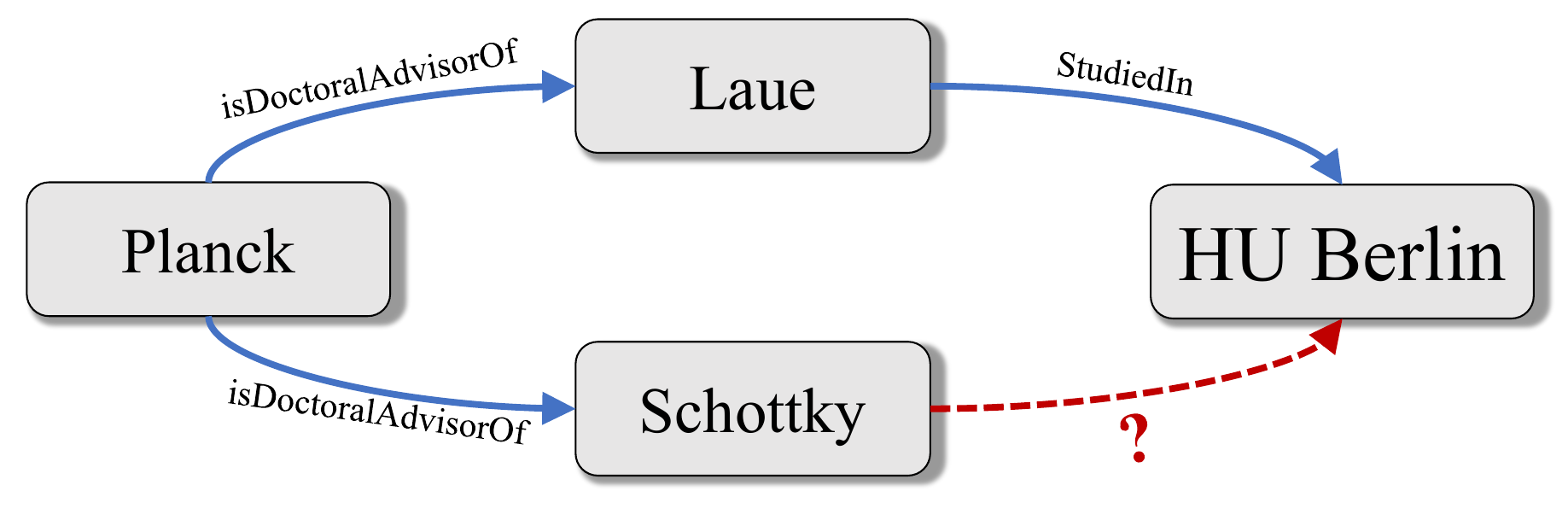}
  \label{fig:why-c}
  \caption{\modifyy{A KGC problem}.\label{fig:illu}}
\end{subfigure}
\medskip
\begin{subfigure}{0.48\columnwidth}
  \includegraphics[width=115pt]{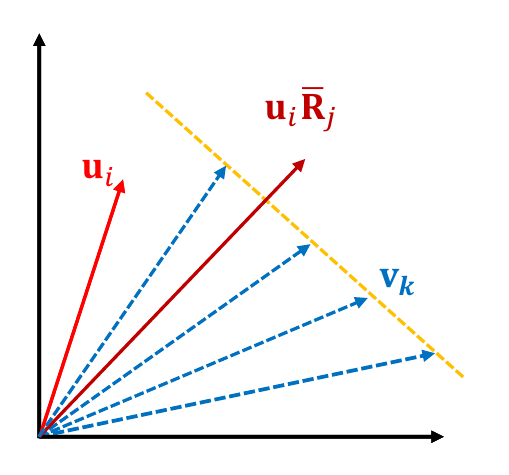}
  \label{fig:why-a}
  \caption{Without regularization.\label{fig:reg_na}}
\end{subfigure}\hfil 
\medskip
\begin{subfigure}{0.48\columnwidth}
  \includegraphics[width=115pt]{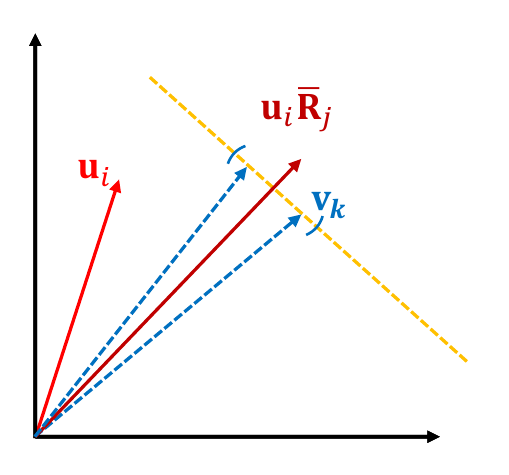}
  \label{fig:why-b}
  \caption{With DURA.}\label{fig:reg_dura}
\end{subfigure}
\vskip -0.15in
\caption{An illustration of why basic DURA can improve the performance of SM models when the embedding dimensions are $2$. Suppose that triples $(\head_i,r_j,\tail_k)$ ($k=1,2,\dots,n$) are valid. (a) Figure \ref{fig:illu} demonstrates that tail entities connected to a head entity through the same relation should have similar embeddings. (b) Figure \ref{fig:reg_na} shows that SM models without regularization can get the same score even though the embeddings of $\tail_k$ are dissimilar. (c) Figure \ref{fig:reg_dura} shows that with DURA, embeddings of $\tail_k$ are encouraged to locate in a small region.}
\label{fig:illustration}
\vskip -0.1in
\end{figure}

\subsection{Motivations}\label{sec:motivation}

To perform accurate link prediction, we expect tail entities connected to a head entity through the same relation to have similar embeddings.
First, we claim that tail entities connected to a head entity through the same relation should have similar semantics. Suppose that we know a head entity $\head_i$ and a relation $r_j$, and our aim is to predict the tail entity.  If $r_j$ is a one-to-many relation, i.e., there exist two entities $\tail_1$ and $\tail_2$ such that both $(\head_i,r_j,\tail_1)$ and $(\head_i,r_j,\tail_2)$ are valid, then we expect that $\tail_1$ and $\tail_2$ have similar semantics.
\modifyy{For example, if two triples (\tfont{Planck}, \tfont{isDoctoralAdvisorOf}, \tfont{Laue}) and (\tfont{Planck}, \tfont{isDoctoralAdvisorOf}, \tfont{Schottky}) are valid, then \tfont{Laue} and \tfont{Schottky} should have similar semantics. Further, we expect that entities with similar semantics have similar embeddings. In this way, if we have known that (\tfont{Laue}, \tfont{StudiedIn}, \tfont{HU Berlin}) is valid, then we can predict that (\tfont{Schottky}, \tfont{StudiedIn}, \tfont{HU Berlin}) is likely to be valid, which is consistent with the fact. See Figure \ref{fig:illu} for an illustration of the prediction process.}

However, SM models fail to achieve the above goal. As shown in Figure \ref{fig:reg_na}, suppose that we have known $\heademb_i\bar{\textbf{R}}_j$ when the embedding dimension is $2$. Then, we can get the same score $s(\head_i,r_j,\tail_k)$ for $k=1,2,\dots,n$, so long as $\tailemb_k$ lies on the same line perpendicular to $\heademb_i\bar{\textbf{R}}_j$. Generally, the entities $t_1$ and $t_2$ have similar semantics. However, their embeddings $\tailemb_1$ and $\tailemb_2$ can even be orthogonal, which means that the two embeddings are dissimilar. Therefore, the performance of SM models for knowledge graph completion is usually unsatisfying.

\subsection{Basic DURA}\label{sec:basic_dura}
Consider the static knowledge graph completion problem $(\head_i, r_j,?)$. That is, we are given the head entity and the relation, aiming to predict the tail entity. Suppose that $f_j(i,k)$ measures the plausibility of a given triplet $(\head_i,r_j,\tail_k)$, i.e., $f_j(i,k)=s(\head_i,r_j,\tail_k)$. Then the score function of a SM model is
\begin{align}\label{eqn:score}
    f_j(i,k)=\re(\overline{\heademb}_i \textbf{R}_j\tailemb_k^\top)=\re(\langle \heademb_i \overline{\textbf{R}}_j, \tailemb_k\rangle).
\end{align}
It first maps the entity embeddings $\heademb_i$ by a linear transformation $\overline{\textbf{R}}_j$, and then uses the real part of an inner product to measure the similarity between $\heademb_i \overline{\textbf{R}}_j$ and $\tailemb_k$. 
Notice that another commonly used similarity measure---Euclidean distance---can replace the dot product similarity in Equation \eqref{eqn:score}. We can obtain an associated distance based model formulated as
\begin{align*}
    f_j^E(i,k)=-\|\heademb_i\overline{\textbf{R}}_j-\tailemb_k\|_2^2.
\end{align*}
Note that we use the squared score function that is equivalent to the form without square. That is to say, there exists a \textit{\textbf{duality}}---for an existing semantic matching KGE model (\textit{\textbf{primal}}), there is often another distance based KGE model (\textit{\textbf{dual}}) closely associated with it.

Specifically, the relationship between the primal and the dual can be formulated as
\begin{align*}
    f_j^E(i,k)&=-\|\heademb_i\overline{\textbf{R}}_j-\tailemb_k\|_2^2\\
    &=-\|\heademb_i\overline{\textbf{R}}_j\|_2^2-\|\tailemb_k\|_2^2+2\re(\langle \heademb_i\overline{\textbf{R}}_j, \tailemb_k\rangle)\\
    &=2f_j(i,k)-\|\heademb_i\overline{\textbf{R}}_j\|_2^2-\|\tailemb_k\|_2^2.
\end{align*}
When we train a distance based model, the aim is to maximize $f_j^E(i,k)$ for all valid triples $(\head_i,r_j,\tail_k)$. Suppose that $\mathcal{S}$ is the set that contains all valid triples. Then, we have that 
\begin{align}\label{eqn:dual_rel}
    &\max_{(\head_i,r_j,\tail_k)\in\mathcal{S}} f_j^E(i,k)
    = \min_{(\head_i,r_j,\tail_k)\in\mathcal{S}} -f_j^E(i,k)\nonumber\\
    =&\min_{(\head_i,r_j,\tail_k)\in\mathcal{S}} -2f_j(i,k)+\|\heademb_i\overline{\textbf{R}}_j\|_2^2+\|\tailemb_k\|_2^2.
\end{align}
Therefore, the duality induces a regularizer for semantic matching KGE models, i.e.,
\begin{align}\label{reg:f}
   \sum_{(\head_i,r_j,\tail_k)\in\mathcal{S}}\|\heademb_i\overline{\textbf{R}}_j\|_2^2+\|\tailemb_k\|_2^2,
\end{align}
which is called basic DURA.

By Equation \eqref{eqn:dual_rel}, we know that basic DURA constrains the distance between $\heademb_i\bar{\textbf{R}}_j$ and $\tailemb_k$. If $\heademb_i$ and $\bar{\textbf{R}}_j$ are known, then $\tailemb_k$ will lie in a small region (see Figure \ref{fig:reg_dura}). Thus, tail entities  connected to a head entity through the same relation will have similar embeddings, which meets our expectation and is beneficial to the prediction of unknown triples.

\subsection{DURA}\label{sec:dura}
Basic DURA encourages tail entities with similar semantics to have similar embeddings. However, it cannot handle the case that head entities have similar semantics.

Suppose that  (\tfont{Laue}, \tfont{StudiedIn}, \tfont{HU Berlin}) and (\tfont{Schottky}, \tfont{StudiedIn}, \tfont{HU Berlin}) are valid. Similar to the discussion in Section \ref{sec:motivation}, we expect that \tfont{Laue} and \tfont{Schottky} have similar semantics and thus have similar embeddings. If we further know that (\tfont{DU Berlin}, \tfont{hasCitizen}, \tfont{Laue}) is valid, we can predict that (\tfont{DU Berlin}, \tfont{hasCitizen}, \tfont{Schottky}) is valid. However, basic DURA cannot handle the case.
Let $\heademb_1$, $\heademb_2$, $\tailemb_1$, and $\textbf{R}_1$ be the embeddings of \tfont{Laue}, \tfont{Schottky}, \tfont{HU Berlin}, and \tfont{StudiedIn}, respectively. Then, $\re(\overline{\heademb}_1 \textbf{R}_1\tailemb_1^\top)$ and $\re(\overline{\heademb}_2 \textbf{R}_1\tailemb_1^\top)$ can be equal even if $\heademb_1$ and $\heademb_2$ are orthogonal, as long as $\heademb_1\overline{\textbf{R}}_1=\heademb_2\overline{\textbf{R}}_1$.

To tackle the above issue, noticing that $\re(\overline{\heademb}_i\textbf{R}_j\tailemb_k^\top)=\re(\overline{\tailemb}_k \overline{\textbf{R}}_j^\top\heademb_i^\top)$, we define another dual distance based KGE model
\begin{align*}
    \tilde{f}_{j}^E(i,k)=-\|\tailemb_k \textbf{R}_j^\top-\heademb_i\|_2^2.
\end{align*}
Then, similar to the derivation in Equation \eqref{eqn:dual_rel}, the duality induces a regularizer given by 
\begin{align}\label{reg:b}
   \sum_{(\head_i,r_j,\tail_k)\in\mathcal{S}}\|\tailemb_k\textbf{R}_j^\top\|^2+\|\heademb_i\|^2.
\end{align}
When a SM model are incorporated with regularizer \eqref{reg:b}, head entities with similar semantics will have similar embeddings.

Finally, combining the regularizers \eqref{reg:b} and \eqref{reg:f}, DURA has the form of 
\begin{align}\label{reg:dura}
    \sum_{(\head_i,r_j,\tail_k)\in\mathcal{S}}\left[\|\heademb_i\overline{\textbf{R}}_j\|_2^2+\|\tailemb_k\|_2^2+\|\tailemb_k\textbf{R}_j^\top\|_2^2+\|\heademb_i\|_2^2\right].
\end{align}

\subsection{Theoretical Analysis for Diagonal Relation Matrices}
Ignoring the sampling frequency and summarizing DURA on all possible entities and relations, we can write the regularizer as:
\begin{align}
    |\mathcal{E}|\sum_{j=1}^{|\mathcal{R}|}(\|\headmat\overline{\textbf{R}}_j\|_F^2+\|\tailmat\|_F^2+\|\tailmat\textbf{R}_j^\top\|_F^2+\|\headmat\|_F^2),
\end{align}
where $|\mathcal{E}|$ and $|\mathcal{R}|$ are the number of entities and relations, respectively. 

In the rest of this section, we use the same definitions of $\hat{\mathcal{X}}_j$ and $\hat{\mathcal{X}}$ as in the problem \eqref{prob:optim}.  When the relation embedding matrices $\textbf{R}_j$ are diagonal in $\mathbb{R}$ or $\mathbb{C}$ as in CP or ComplEx, DURA gives an upper bound to the tensor nuclear 2-norm of $\hat{\mathcal{X}}$, which is an extension of trace norm regularizers in matrix completion. To simplify the notations, we take CP as an example, in which all involved embeddings are real numbers. The conclusion in complex space can be analogized accordingly.

\begin{definition}[\cite{nunorm}]
The nuclear 2-norm of a 3D tensor $\mathcal{A}\in\mathbb{R}^{n_1}\otimes\mathbb{R}^{n_2}\otimes\mathbb{R}^{n_3}$ is 
\begin{align*}
    \|\mathcal{A}\|_{*}=\min&\left\{\sum_{i=1}^r\|\textbf{u}_{1,i}\|_2\|\textbf{u}_{2,i}\|_2\|\textbf{u}_{3,i}\|_2:\right.\\
    &\left.\mathcal{A}=\sum_{i=1}^r\textbf{u}_{1,i}\otimes \textbf{u}_{2,i}\otimes \textbf{u}_{3,i},r\in\mathbb{N}\right\},
\end{align*}
where $\textbf{u}_{k,i} \in \mathbb{R}^{n_k}$ for $k=1, ..., 3$, $i=1,...,r$, and $\otimes$ denotes the outer product.
\end{definition}
For notation convenience, we define a relation matrix $\widetilde{\mathbf{R}} \in \mathbb{R}^{|\mathcal{R}| \times D}$, of which the $j$-th row
consists of the diagonal entries of $\textbf{R}_j$. That is,
 $\widetilde{\textbf{R}}(j, d) =  \textbf{R}_j (d, d),$
where $\textbf{A}(i, j)$ represents the  entry in the $i$-th row and $j$-th column of the matrix $\textbf{A}$.

In the knowledge graph completion problem, the tensor nuclear 2-norm of $\hat{\mathcal{X}}$ is
\begin{align*}
    \|\hat{\mathcal{X}}\|_{*}=\min&\left\{\sum_{d=1}^D\|\heademb_{:d}\|_2\|\textbf{r}_{:d}\|_2\|\tailemb_{:d}\|_2:\right.\\
    &\left.\hat{\mathcal{X}}=\sum_{d=1}^D\heademb_{:d}\otimes \textbf{r}_{:d}\otimes \tailemb_{:d}  \right\},\nonumber
\end{align*}
where $D$ is the embedding dimension, $\heademb_{:d}$, $\textbf{r}_{:d}$, and $\tailemb_{:d}$ are the $d$-th columns of $\headmat$, $\widetilde{\textbf{R}}$, and $\tailmat$.

For DURA, we have the following theorem, of which the proof are provided in Appendix A.
\begin{theorem}\label{thm:main}
Suppose that $\hat{\mathcal{X}}_j=\headmat\textbf{R}_j\tailmat^\top$ for $j=1,2,\dots,|\mathcal{R}|$, where $\headmat,\tailmat,\textbf{R}_j$ are real matrices and $\textbf{R}_j$ is diagonal. Then, the following equation holds
\begin{align*}
    \min_{\substack{\hat{\mathcal{X}}_j=\headmat \textbf{R}_j\tailmat^\top  }}&\frac{1}{4|\sqrt{\mathcal{R}}|}\sum_{j=1}^{|\mathcal{R}|}(\|\headmat\textbf{R}_j\|_F^2+\|\tailmat\|_F^2+\|\tailmat\textbf{R}_j^\top\|_F^2+\|\headmat\|_F^2)\\
=&\|\hat{\mathcal{X}}\|_*.
\end{align*}
The minimizers satisfy
$
  \|\heademb_{:d}\|_2\|\textbf{r}_{:d}\|_2=\sqrt{|\mathcal{R}|}\|\tailemb_{:d}\|_2
$
and
$
  \|\tailemb_{:d}\|_2\|\textbf{r}_{:d}\|_2=\sqrt{|\mathcal{R}|}\|\heademb_{:d}\|_2,
$
$\,\forall\,d\in\{1,2,\ldots, D\}$,
where $\heademb_{:d}$, $\textbf{r}_{:d}$, and $\tailemb_{:d}$ are the $d$-th columns of $\headmat$, $\widetilde{\textbf{R}}$, and $\tailmat$, respectively.
\end{theorem}
Therefore, the minimization of DURA has a form of the tensor nuclear 2-norm, which is a tensor analog to the matrix trace norm. Matrix trace norm regularization has shown great power in the matrix completion problem. In the experiments, we will show that DURA reproduces the success in the tensor completion problem.

\section{DURA for Temporal KGE}

We first introduce a general TKGE model TRESCAL---which is an extension of RESCAL to temporal KGs---and introduce the corresponding DURA regularizer. Then, we introduce the enhanced version of DURA for TComplEx \cite{tcomplex}.

\subsection{TRESCAL and DURA for TRESCAL}
Inspired by the classical static KGE model RESCAL, we propose its temporal extension TRESCAL as follows:
\begin{align}\label{eqn:trescal}
    s(\head_i,r_j,\tail_k,t_l)=\headconjemb(\textbf{R}_j \odot \timemat_l)\tailemb_k^\top,
\end{align}
where  $\textbf{R}_j$ and $\timemat_l$ are trainable real matrices. Note that the element-wise multiplication $\odot$ is commutative, i.e., $A\odot B=B\odot A$ for all matrices $A$ and $B$.

In Equation \eqref{eqn:trescal}, the product of relation and time matrices $\textbf{R}_j \odot \timemat_l$ can be seen as the representation of time-dependent relations, leading to $|\mathcal{R}|\times|\mathcal{T}|$ equivalent relations in a KG. By a similar derivation in Section \ref{sec:dura}, we propose DURA regularizer (DURA1) for TRESCAL as follows:
\begin{align}\label{reg:trescal}
    \sum_{(\head_i,r_j,\tail_k,t_l)\in\mathcal{S}}
    &[\|\overline{\heademb}_i(\textbf{R}_j \odot \timemat_l)\|_2^2+\|\tailemb_k\|_2^2\\
    &+\|\tailemb_k(\textbf{R}_j \odot \timemat_l)^\top\|_2^2+\|\heademb_i\|_2^2].
\end{align}

\subsection{DURA for TComplEx}\label{sec:dura_tcomplex}
\zhshi{By introducing representations of time-dependent relations, we can extend ComplEx \cite{complex} to temporal KGs. The resulting model, TComplEx \cite{tcomplex}}, is formulated as
\begin{align}\label{eqn:tcomplex}
    s(\head_i,r_j,\tail_k,t_l)=\re(\headconjemb(\textbf{R}_j \odot \timemat_l)\tailemb_k^\top).
\end{align}
\zqzhang{
As mentioned by Lacroix et al. \cite{tcomplex}, timestamp embeddings in temporal semantic matching models can be used to equivalently modulate both the relations and entities to obtain time-dependent representations. Therefore, we introduce DURA for temporal KGE in two views.}

First, as relation embeddings and timestamp embeddings are diagonal matrices, the score function becomes $s(\head_i,r_j,\tail_k,t_l)=\re((\headconjemb \timemat_l\textbf{R}_j)\tailemb_k^\top)=\re(\headconjemb (\timemat_l\textbf{R}_j\tailemb_k^\top)).$
This formulation models time-dependent relation representations. We introduce a form of DURA (DURA1) as
\begin{align}
    \sum_{(\head_i,r_j,\tail_k,t_l)\in\mathcal{S}}
    &\left[\|\overline{\heademb}_i(\textbf{R}_j  \timemat_l)\|_2^2+\|\tailemb_k\|_2^2\right.\nonumber\\
    &\left.+\|\tailemb_k(\textbf{R}_j \timemat_l)^\top\|_2^2+\|\heademb_i\|_2^2\right].\label{reg:tcomplex_1}
\end{align}

Note that when the relation and time matrices are diagonal, the right hand side of Equation \eqref{eqn:tcomplex} satisfy associative law. \zhshi{Thus, we can reformulate the score function into $s(\head_i,r_j,\tail_k,t_l)=\re((\headconjemb \timemat_l)\textbf{R}_j\tailemb_k^\top) = \re(\headconjemb \textbf{R}_j(\timemat_l\tailemb_k^\top))$, where $\heademb_i\timemat_l$ and $\tailemb_k\timemat_l$ are representations of time-dependent entity representations.} Therefore, we derive another formulation of DURA (DURA2):
\begin{align}
    \sum_{(\head_i,r_j,\tail_k,t_l)\in\mathcal{S}}
    &\left[\|\headconjemb\textbf{R}_j\|_2^2+\|\tailemb_k\timemat_l\|_2^2\right.\nonumber\\
    &\left.+\|\headconjemb\timemat_l\|_2^2+\|\tailemb_k\textbf{R}_j\|_2^2\right].\label{reg:tcomplex_2}
\end{align}
We empirically evaluate the aforementioned two versions of DURA in Section \ref{sec:experiments}.

\subsection{Semantic Meaning of Temporal DURA}\label{sec:motivation_temporal}

\zhshi{
The semantic meaning of DURA1 is similar to that of DURA for static KGE. That is, we expect tail/head entities connected to a head/tail entity through similar time-dependent relation representations to have similar embeddings. For example, as shown in Figure \ref{fig:example_tdura} if two quadruples (\tfont{Planck}, \tfont{isDoctoralAdvisorOf}, \tfont{Laue}, \tfont{1902-1903}) and (\tfont{Planck}, \tfont{isDoctoralAdvisorOf}, \tfont{Schottky}, \tfont{1909-1912}) are valid, then we expect that \tfont{Laue} and \tfont{Schottky} have similar semantics. Let $\heademb_i,\textbf{R}_j,\tailemb_k^{(1)},\tailemb_k^{(2)},\timemat_l^{(1)},$ and $\timemat_l^{(2)}$ be embeddings of \tfont{Planck}, \tfont{isDoctoralAdvisorOf}, \tfont{Laue}, \tfont{Schottky}, \tfont{1902-1903}, and \tfont{1909-1912}, respectively. In temporal KGE, close timestamps usually have similar representations \cite{tcomplex} by incorporating a timestamp smoothness regularizer. For analyses convenience and without loss of generality, we assume that $\timemat_l^{(1)}=\timemat_l^{(2)}=\timemat_l$ and hence $\overline{\heademb}_i(\textbf{R}_j  \timemat_l^{(1)})=\overline{\heademb}_i(\textbf{R}_j  \timemat_l^{(2)}) = \overline{\heademb}_i(\textbf{R}_j  \timemat_l)$. As shown in Figure \ref{fig:without_tdura1}, we can get the same score $s(\head_i,r_j,\tail_k,t_l)$ for any $\tail_k$, so long as $\tailemb_k$ lies on the same line perpendicular to $\overline{\heademb}_i(\textbf{R}_j  \timemat_l)$. When we train SM models without DURA1, the embeddings $\tailemb_k^{(1)}$ and $\tailemb_k^{(2)}$ can even be orthogonal, which means that the two embeddings are dissimilar. In contrast, Figure \ref{fig:with_tdura1} shows that DURA1 encourages $\tailemb_k^{(1)}$ and $\tailemb_k^{(2)}$ to locate in a small region, which means that two embeddings of \tfont{Laue} and \tfont{Schottky} are similar. Thus, when (\tfont{Laue}, \tfont{StudiedIn}, \tfont{HU Berlin}, \tfont{1902-1903}) is valid, DURA1 encourages (\tfont{Schottky}, \tfont{StudiedIn}, \tfont{HU Berlin}, \tfont{1905-1912}) to have a high score. The prediction is reasonable, as two persons advised by the same person in a similar period of time tend to study in the same university.

\begin{figure}[t]
    \centering 
\medskip
\begin{subfigure}{0.9\columnwidth}
  \includegraphics[width=220pt]{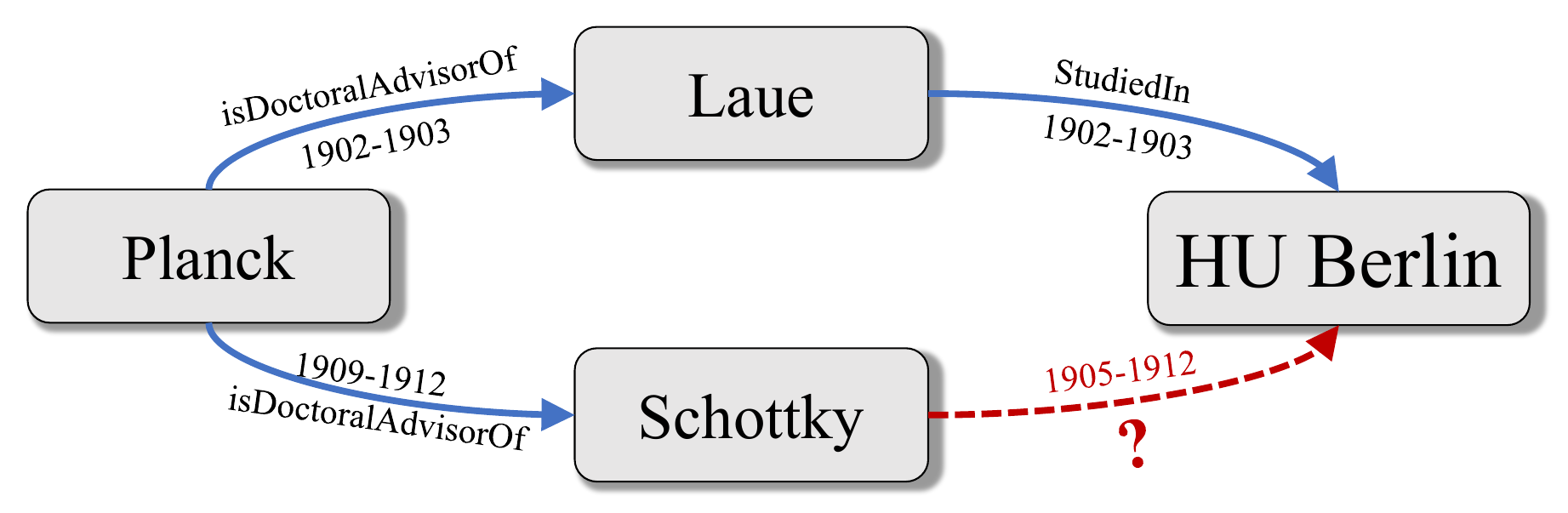}
  \caption{A temporal KGC problem.}\label{fig:example_tdura}
\end{subfigure}
\medskip
\begin{subfigure}{0.45\columnwidth}
  \includegraphics[width=100pt]{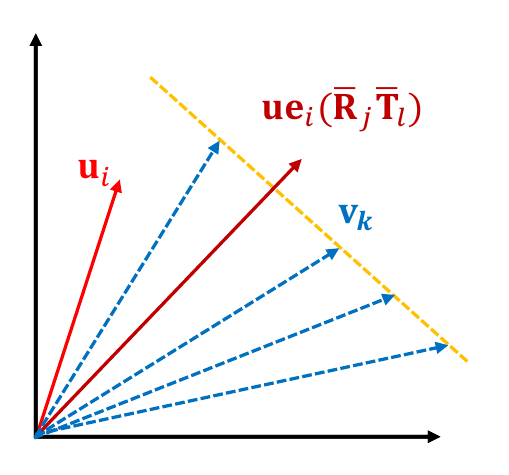}
  \caption{Without DURA1.\label{fig:without_tdura1}}
\end{subfigure}\hfil 
\medskip
\begin{subfigure}{0.45\columnwidth}
  \includegraphics[width=100pt]{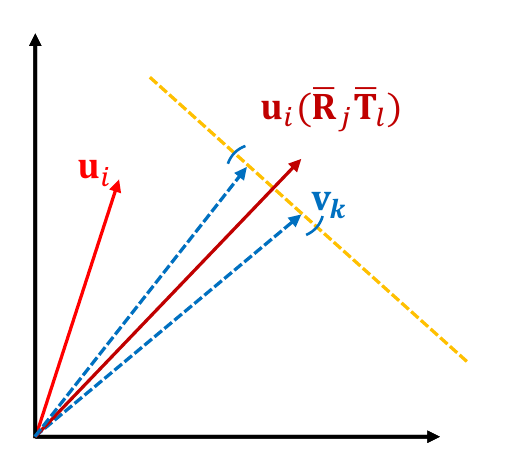}
  \caption{With DURA1.\label{fig:with_tdura1}}
\end{subfigure}\hfil 
\medskip
\begin{subfigure}{0.45\columnwidth}
  \includegraphics[width=100pt]{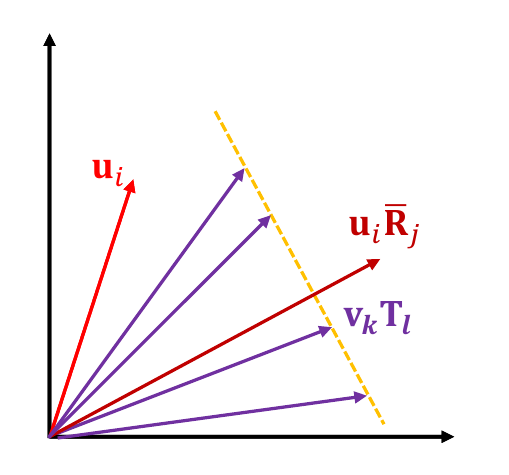}
  \caption{Without DURA2.\label{fig:without_tdura2}}
\end{subfigure}\hfil 
\medskip
\begin{subfigure}{0.45\columnwidth}
  \includegraphics[width=100pt]{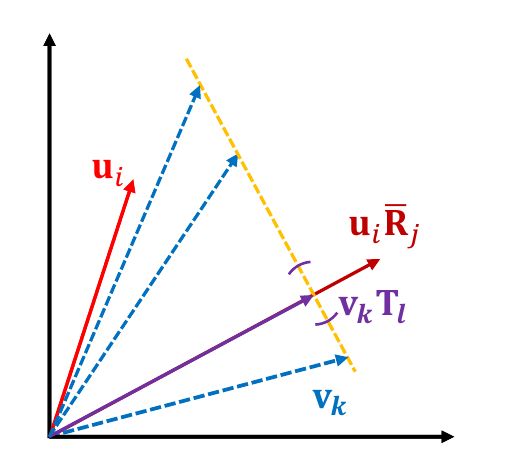}
  \caption{With DURA2.}\label{fig:with_tdura2}
\end{subfigure}

\caption{
\zhshi{An illustration of how temporal DURA influence the performance of SM models when the embedding dimensions are $2$. Suppose that triples $(\head_i,r_j,\tail_k,t_l)$ ($k=1,2,\dots,n$) are valid.
Figure \ref{fig:without_tdura1} and \ref{fig:without_tdura2} show that SM models without regularization can get the same score even though and the embeddings of original entities $\tailemb_k$ and time-dependent entities $\tailemb_k\timemat_l$ are dissimilar, respectively.
Figure \ref{fig:with_tdura1} shows that with DURA1, embeddings of $\tailemb_k$ are encouraged to locate in a small region.
Figure \ref{fig:with_tdura2} shows that with DURA2, embeddings of time-dependent entities $\tailemb_k\timemat_l$ are encouraged to locate in a small region, while the corresponding embeddings of original entities may be dissimilar.}
}
\label{fig:why_tdura}
\end{figure}

DURA2 encourages time-dependent entities with similar semantics to have similar embeddings, while similar embeddings of time-dependent entities do not imply similar embeddings of original entities.
Similar to the above discussion, DURA2 encourages the time-dependent entities $\tailemb_k^{(1)}\timemat_l$ and $\tailemb_k^{(2)}\timemat_l$ to locate in a small region. However, as shown in Figure \ref{fig:with_tdura2}, 
when embeddings of time-dependent entities $\tailemb_k\timemat_l$ are projections of $\tailemb_k$ onto the subspace spanned by $\overline{\heademb}_i\textbf{R}_j$, the embeddings of entities $\tailemb_k^{(1)}$ and $\tailemb_k^{(2)}$ can be dissimilar, even if  the corresponding time-dependent entities $\tailemb_k\timemat_l$ are the same. Thus, when (\tfont{Laue}, \tfont{StudiedIn}, \tfont{HU Berlin}, \tfont{1905-1912}) is valid, (\tfont{Schottky}, \tfont{StudiedIn}, \tfont{HU Berlin}, \tfont{1902-1903}) is not guaranteed to have a high score. However, as shown in Section \ref{sec:theo_tdura}, both DURA1 and DURA2 correspond to nuclear-2 norm of 4-order tensors. Thus, both of them can improve the performance of semantic matching models, while we expect DURA1 outperforms DURA2.
Experiments in Section \ref{sec:t_main_results} confirm our points.

}

\subsection{Theoretical Analysis for Diagonal Relation and Time Matrices}\label{sec:theo_tdura}
We give an analysis for the case that relation and time matrices are diagonal (i.e., TComplEx).

First, we review the definition of 4D tensors.
\begin{definition}[\cite{nunorm}]
The nuclear 2-norm of a 4D tensor $\mathcal{A}\in\mathbb{R}^{n_1}\otimes\mathbb{R}^{n_2}\otimes\mathbb{R}^{n_3}\otimes\mathbb{R}^{n_4}$ is 
\begin{align*}
    \|\mathcal{A}\|_{*}=\min&\left\{\sum_{i=1}^r\|\textbf{u}_{1,i}\|_2\|\textbf{u}_{2,i}\|_2\|\textbf{u}_{3,i}\|_2\|\textbf{u}_{4,i}\|_2:\right.\\
    &\left.\mathcal{A}=\sum_{i=1}^r\textbf{u}_{1,i}\otimes \textbf{u}_{2,i}\otimes \textbf{u}_{3,i}\otimes \textbf{u}_{4,i},r\in\mathbb{N}\right\},
\end{align*}
where $\textbf{u}_{k,i} \in \mathbb{R}^{n_k}$ for $k=1, ..., 4$, $i=1,...,r$, and $\otimes$ denotes the outer product.
\end{definition}
In the temporal knowledge graph completion problem, the tensor nuclear 2-norm of $\hat{\mathcal{X}}$ is
\begin{align*}
    \|\hat{\mathcal{X}}\|_{*}=\min&\left\{\sum_{d=1}^D\|\textbf{u}_{:d}\|_2\|\textbf{r}_{:d}\|_2\|\textbf{v}_{:d}\|_2\|\textbf{t}_{:d}\|_2:\right.\\
    &\left.\hat{\mathcal{X}}=\sum_{d=1}^D\textbf{u}_{:d}\otimes \textbf{r}_{:d}\otimes \textbf{v}_{:d}\otimes \textbf{t}_{:d}  \right\},
\end{align*}
where $\textbf{u}_{:d}$, $\textbf{r}_{:d}$, $\textbf{v}_{:d}$, and $\textbf{t}_{:d}$ are the $d$-th columns of the head entities matrix $\textbf{U}$, the relation matrix $\widetilde{\textbf{R}}$, the tail entities matrix $\textbf{V}$, and the time matrix $\tailmat$, respectively.

Then, we have the following theorems for DURA, which correspond to the regularizers \eqref{reg:tcomplex_1} and \eqref{reg:tcomplex_2}, respectively. The proofs are provided in the Appendixes B and C.
\begin{theorem}\label{thm:tdura_1}
Suppose that $\hat{\mathcal{X}}_{jl}=\textbf{U}(\textbf{R}_j\timemat_l)\textbf{V}^\top$ for $j=1,2,\dots,|\mathcal{R}|$, where $\textbf{U},\textbf{V},\textbf{R}_j, \timemat_l$ are real matrices and $\textbf{R}_j$ and $\timemat_l$ are diagonal. Then, the following equation holds
\begin{align*}
    \|\hat{\mathcal{X}}\|_* = \min_{\substack{\hat{\mathcal{X}}_{jl}=\textbf{U}(\textbf{R}_j\timemat_l)\textbf{V}^\top\\  }}\frac{1}{4\sqrt{|\mathcal{R}||\mathcal{T}|}}&\sum_{l=1}^{|\mathcal{T}|}\sum_{j=1}^{|\mathcal{R}|}(\|\textbf{U}\textbf{R}_j\timemat_l\|_F^2+\|\textbf{V}\|_F^2 \\
    +&\|\textbf{V}(\textbf{R}_j\timemat_l)^\top\|_F^2+\|\textbf{U}\|_F^2).
\end{align*}
The minimizers of the left side satisfy
$
  \|\textbf{u}_{:d}\|_2\|\textbf{r}_{:d}\|_2\|\textbf{t}_{:d}\|_2=\sqrt{|\mathcal{R}||\mathcal{T}|}\|\textbf{v}_{:d}\|_2
$
and
$
  \|\textbf{v}_{:d}\|_2\|\textbf{r}_{:d}\|_2\|\textbf{t}_{:d}\|_2=\sqrt{|\mathcal{R}||\mathcal{T}|}\|\textbf{u}_{:d}\|_2,
$
$\,\forall\,d\in\{1,2,\ldots, D\}$,
where $\textbf{u}_{:d}$, $\textbf{r}_{:d}$, $\textbf{v}_{:d}$, and $\textbf{t}_{:d}$ are the $d$-th columns of $\textbf{U}$, $\textbf{R}$, $\textbf{V}$ and $\tailmat$, respectively.
\end{theorem}

\begin{theorem}\label{thm:tdura_2}
Suppose that $\hat{\mathcal{X}}_{jl}=\textbf{U}(\textbf{R}_j\timemat_l)\textbf{V}^\top$ for $j=1,2,\dots,|\mathcal{R}|$, where $\textbf{U},\textbf{V},\textbf{R}_j, \timemat_l$ are real matrices and $\textbf{R}_j$ and $\timemat_l$ are diagonal. Then, the following equation holds
\begin{align*}
    \|\hat{\mathcal{X}}\|_* = \min_{\substack{\hat{\mathcal{X}}_{jl}=\textbf{U}(\textbf{R}_j\timemat_l)\textbf{V}^\top\\  }}\frac{1}{4\sqrt{|\mathcal{R}||\mathcal{T}|}}\sum_{l=1}^{|\mathcal{T}|}\sum_{j=1}^{|\mathcal{R}|}&(\|\textbf{U}\textbf{R}_j\|_F^2+\|\textbf{V}\timemat_l^\top\|_F^2\\
    +&\|\textbf{U}\timemat_l\|_F^2+\|\textbf{V}\textbf{R}_j^\top\|_F^2).
\end{align*}
The minimizers satisfy
$
  \sqrt{|\mathcal{T}|}\|\textbf{u}_{:d}\|_2\|\textbf{r}_{:d}\|_2=\sqrt{|\mathcal{R}|}\|\textbf{v}_{:d}\|_2\|\textbf{t}_{:d}\|_2
$
and
$
  \sqrt{|\mathcal{R}|}\|\textbf{u}_{:d}\|_2\|\textbf{t}_{:d}\|_2=\sqrt{\mathcal{T}}\|\textbf{v}_{:d}\|_2\|\textbf{r}_{:d}\|_2,
$
$\,\forall\,d\in\{1,2,\ldots, D\}$,
where $\heademb_{:d}$, $\textbf{r}_{:d}$, and $\textbf{t}_{:d}$ are the $d$-th columns of $\headmat$, $\widetilde{\textbf{R}}$, and $\tailmat$, respectively.
\end{theorem}
Therefore, the minimization of DURA in the temporal knowledge graph case has a form of the tensor nuclear 2-norm of 4-order tensors. Theorem \ref{thm:tdura_2} also explains why DURA2 improves the performance of temporal semantic matching KGEs, even if it cannot encourage entity with similar semantics to have similar representations.

\begin{table*}[t]
    \centering
    \caption{\modifyy{Hyper-parameters found by grid search. $k$ is the embedding size, $b$ is the batch size, $\lambda$ is the regularization coefficients, and $\lambda_1$ and $\lambda_2$ are weights for different parts of the regularizer.}}\label{table:hp_skbc}
    \label{table:hp}
    \begin{tabular}{l  c c c c c  c c c c c  c c c c c }
    \toprule
        &\multicolumn{5}{c}{\textbf{WN18RR}}&  \multicolumn{5}{c}{\textbf{FB15k-237}} & \multicolumn{5}{c}{\textbf{YAGO3-10}}\\
         \cmidrule(lr){2-6}
         \cmidrule(lr){7-11}
         \cmidrule(lr){12-16}
         & $k$ & $b$ & $\lambda$ & $\lambda_1$ &$\lambda_2$ & $k$ & $b$ & $\lambda$ & $\lambda_1$ &$\lambda_2$ & $k$ & $b$ & $\lambda$ & $\lambda_1$ &$\lambda_2$ \\
        \midrule
        CP & 2000 & 100 & 1e-1 & 0.5 & 1.5 & 2000 & 100 & 5e-2 & 0.5 & 1.5 & 1000 & 1000 & 5e-3 & 0.5 & 1.5\\
        ComplEx & 2000 & 100 & 1e-1 & 0.5  &1.5 & 2000 & 100 & 5e-2 & 0.5 & 1.5 & 1000 & 1000 & 5e-2 & 0.5 & 1.5\\
        RESCAL & 512 & 1024 &1e-1 & 1.0 & 1.0 & 512 & 512 &1e-1 & 2.0 & 1.5 & 512 & 1024 &5e-2 & 1.0 & 1.0\\
        \bottomrule
    \end{tabular}
\end{table*}
\section{Experiments} \label{sec:experiments}
The knowledge graph completion task is a standard benchmark task to evaluate KGE models. To demonstrate that DURA is an effective and widely-applicable regularizer, we conduct extensive experiments on both static and temporal knowledge graph completion datasets.

\subsection{Static Knowledge Graph Completion}
We introduce the experimental settings for static KGC in Section \ref{sec:exp_settting} and demonstrate the effectiveness of DURA on three benchmark datasets in Section \ref{sec:main_results}. Then, we compare DURA with other regularizers in Section \ref{sec:comp}.

\subsubsection{Experimental Settings for Static KGC}\label{sec:exp_settting}

We consider three public static knowledge graph datasets---WN18RR \cite{wn18rr}, FB15k-237 \cite{conve}, and YAGO3-10 \cite{yago3} for the knowledge graph completion task, which have been divided into training, validation, and testing set in previous works. The statistics of these datasets are shown in Appendix D.

\begin{table*}[ht]
    \centering
\zqzhang{
    \caption{\zqzhang{Evaluation results on WN18RR, FB15k-237 and YAGO3-10 datasets. We reimplement CP, ComplEx, and RESCAL using the ``reciprocal'' setting \cite{n3,simple}, which leads to better results than the reported results in the original paper. * indicates that the model use external textual information.}
    }
    \label{table:main_results}
    \begin{tabular}{l  c c c  c c c  c c c }
        \toprule
          &\multicolumn{3}{c}{\textbf{WN18RR}}&  \multicolumn{3}{c}{\textbf{FB15k-237}} & \multicolumn{3}{c}{\textbf{YAGO3-10}}\\
         \cmidrule(lr){2-4}
         \cmidrule(lr){5-7}
         \cmidrule(lr){8-10}
         & MRR & H@1 & H@10 & MRR & H@1 & H@10 & MRR & H@1 & H@10 \\
        \midrule
        TransE & .226 & - & .501 & .294 & - & .465 & - & - & -\\
        TransH & .186 & - & .424 & .211 & - & .366 & - & - & -\\
        TransD & .185 & - & .428 & .286 & - & .453 & - & - & -\\
        RotatE & .476  & .428 & .571  & .338 & .241 & .533 & .495 & .402 & .670\\
        MuRP   &.481 &.440 &.566 &.335 &.243 &.518 & - & - & - \\
        HAKE   &.497 &.452 &.582 &.346 &.250 &.542 &.546 &.462 &.694\\
        \midrule
        ConvE & .444 & - & .503 & .324 & - & .501 & - & - & -\\
        ConvKB & .249 & - & .524 & .243 & - & .421 & - & - & -\\
        KB-GAT & .412 & - & .554 & .157 & - & .331 & - & - & -\\
        InteractE & .463 & .430 & .528 & .354 & .263 & .535 & .541 & .462 & .687\\
        StAR & .401 & .243 & .709 & .296 & .205 &.482 & - & - & -\\
        StAR (Self-Adp)* & \textbf{.551} & \textbf{.459} & \textbf{.732} & .365 & .266 & \textbf{.562} & - & - & -\\
        \midrule
        CP      &.438 &.414 &.485 &.333 &.247 &.508 &.567 &.494 &.698\\
        RESCAL  &.455 &.419 &.493 &.353 &.264 &.528 &.566 &.490 &.701\\
        ComplEx & .460  & .428 & .522  & .346 & .256 & .525 & .573 & .500 & .703\\
        \midrule
        CP-DURA  &.478 &.441 &.552 &.367  &.272 &.555 &.579 &.506 &.709\\
        RESCAL-DURA    &.498 &.455 &.577 &.368 &\textbf{.276} &.550 &.579 &.505 &.712\\
         ComplEx-DURA &.491 &.449 &.571 &\textbf{.371} &\textbf{.276} &.560 &\textbf{.584} &\textbf{.511} &\textbf{.713}\\
        \bottomrule
    \end{tabular}
    }
\end{table*}
\begin{table*}[ht]
    \centering
    \caption{Comparison between DURA, squared Frobenius norm (FRO), and nuclear 3-norm (N3) regularizers. Results of * are taken from \cite{n3}. CP-N3 and ComplEx-N3 are re-implemented and their performances are better than the reported results in \cite{n3}. The best performance on each model are marked in bold.}
    \label{table:ablation_results}
    \begin{tabular}{l  c c c  c c c  c c c }
        \toprule
          &\multicolumn{3}{c}{\textbf{WN18RR}}&  \multicolumn{3}{c}{\textbf{FB15k-237}} & \multicolumn{3}{c}{\textbf{YAGO3-10}}\\
         \cmidrule(lr){2-4}
         \cmidrule(lr){5-7}
         \cmidrule(lr){8-10}
         & MRR & H@1 & H@10 & MRR & H@1 & H@10 & MRR & H@1 & H@10 \\
        \midrule
        CP-FRO* &.460 &- &.480 &.340 &- &.510 &.540 &- &.680\\
        CP-N3 &.470  &.430   &.544 &.354 &.261  &.544 &.577 &.505  &.705  \\
        CP-DURA  &\textbf{.478} &\textbf{.441} &\textbf{.552} &\textbf{.367}  &\textbf{.272} &\textbf{.555} &\textbf{.579} &\textbf{.506}  &\textbf{.709}\\
        \midrule
        ComplEx-FRO* &.470 &- &.540 &.350 &- &.530 &.570 &- &.710 \\
        ComplEx-N3   &.489  &.443 &\textbf{.580}  &.366  &.271   &.558  &.577  &.502  &.711  \\
        ComplEx-DURA &\textbf{.491} &\textbf{.449} &.571 &\textbf{.371} &\textbf{.276}  &\textbf{.560} &\textbf{.584} &\textbf{.511} &\textbf{.713}\\
        \midrule
        RESCAL-FRO   &.397 &.363  &.452 & .323 &.235 &.501 &.474 &.392 &.628\\
        RESCAL-DURA    &\textbf{.498} &\textbf{.455} &\textbf{.577} &\textbf{.368} &\textbf{.276}  &\textbf{.550} &\textbf{.579} &\textbf{.505}  &\textbf{.712}\\
        \bottomrule
    \end{tabular}
\end{table*}

We use the cross entropy loss function and the ``reciprocal'' setting that creates a new triplet $(\tail_k,r_j^{-1},\head_i)$ for each triplet  $(\head_i,r_j,\tail_k)$ \cite{n3,simple}. We use Adagrad \cite{adagrad} as the optimizer, and use grid search to find the best hyperparameters based on the performance on the validation datasets. 
\modifyy{
As we regard the link prediction as a multi-class classification problem and use the cross entropy loss, we can assign different weights for different classes (i.e., tail entities) based on their frequency of occurrence in the training dataset. Specifically, suppose that the loss of a given query $(\head,r,?)$ is $\ell((\head,r,?),\tail)$, where $\tail$ is the true tail entity, then the weighted loss is $w(t)\ell((\head,r,?),\tail),$
where 
\begin{align*}
    w(\tail)=w_0\frac{\#\tail}{\max\{\#\tail_i:\tail_i \in \text{training set}\}} + (1-w_0),
\end{align*}
$w_0$ is a fixed number, $\#\tail$ denotes the frequency of occurrence in the training set of the entity $\tail$. For all models on WN18RR and RESCAL on YAGO3-10, we choose $w_0=0.1$ and for all the other cases we choose $w_0=0$.

}

Following \cite{transe}, we use entity ranking as the evaluation task. For each triplet $(\head_i,r_j,\tail_k)$ in the test dataset, the model is asked to answer $(\head_i, r_j, ?)$ and $(\tail_k, r_j^{-1}, ?)$. To do this, we fill the positions of missing entities with candidate entities to create a set of candidate triplets, and then rank the triplets in descending order by their scores. Following the ``Filtered'' setting in \cite{transe}, we then filter out all existing triplets known to be true at ranking. We choose Mean Reciprocal Rank (MRR) and Hits at N (H@N) as the evaluation metrics. Higher MRR or H@N indicates better performance.

\modifyy{
We find the best hyper-parameters by grid search.  Specifically, we search learning rates in $\{0.1, 0.01\}$, regularization coefficients in $\{0, 1\times 10^{-3}, 5\times 10^{-3}, 1\times 10^{-2}, 5\times 10^{-2}, 1\times 10^{-1}, 5\times 10^{-1}\}$. On WN18RR and FB15k-237, we search batch sizes in $\{100,500,1000\}$ and embedding sizes in $\{500,1000,2000\}$. On YAGO3-10, we search batch sizes in $\{256,512,1024\}$ and embedding sizes in $\{500,1000\}$.
Table \ref{table:hp} shows the best hyperparameters for DURA found by grid search.
}

\subsubsection{Main Results}\label{sec:main_results}
We demonstrate the performance of DURA on several semantic matching KGE models, including CP \cite{cp}, RESCAL \cite{rescal}, and ComplEx \cite{complex}. 
Note that we reimplement CP, ComplEx, and RESCAL under the ``reciprocal'' setting \cite{simple,n3}, and obtain better results than the reported performance in the original papers. Sun et al. \cite{re-evaluation} demonstrate that many existing learning-based models use inappropriate evaluation protocols and suffer from inflating performance. Therefore, for ConvE, ConvKB, and KB-GAT, we take their results from \cite{re-evaluation} for a fair comparison.

The baseline models include distance-based models (TransE \cite{transe}, TransH \cite{transh}, TransD \cite{transd}, RotatE \cite{rotate}, MuRP \cite{murp}, HAKE \cite{hake}), learning-based models (ConvE \cite{conve}, ConvKB \cite{convkb}, KB-GAT \cite{kbgat}, InteractE \cite{interacte}, and StAR \cite{star}), and semantic matching models (CP \cite{cp}, RESCAL \cite{rescal}, ComplEx \cite{complex}). 
\modifyy{
TransE is the most representative translational distance model, while failing to deal with 1-to-N, N-to-1, and N-to-N relations. TransH  introduces relation-specific hyperplanes to overcome these disadvantages. TransD shares a similar idea with TransH. It introduces relation-specific spaces rather than hyperplanes. TransD also decomposes the projection matrix into a product of two vectors for further simplification. RotatE  defines each relation as a rotation from the source entity to the target entity in the complex vector space and naturally, it is able to model and infer various relation patterns. MuRP embeds multi-relational graph in the Poincar{\'e} ball model of hyperbolic space to capture multiple simultaneous hierarchies. HAKE  maps entities into the polar coordinate system to model semantic hierarchies, which are common in real-world applications. ConvE introduces a multi-layer convolutional network model to learn more expressive features with less parameters. ConvKB  represents triples as 3-column matrixs, and employs a convolutional neural network to capture global relationships and transitional characteristics between entities and relations. KB-GAT proposes an attention-based feature embedding that captures both entity and relation features to cover the complex and hidden information in the local neighborhood surrounding a triple. InteractE improves ConvE by increasing feature interactions. StAR follows the textual encoding paradigm and augments it with graph embedding techniques. 
}

Table \ref{table:main_results} shows the evaluation results.
Overall, DURA significantly improves the performance of all considered SM models. The results demonstrate that without regularization, semantic matching models perform worse than state-of-the-art distance-based models on WN18RR, e.g., RotatE, MuRP, and HAKE. When incorporating with DURA, these semantic matching models achieve competitive performance with distance-based models. Moreover, DURA further improve the performance gain of semantic matching models on FB15k-237 and YAGO3-10. Compared with learning-based models, semantic matching models with DURA outperforms all of them without using external textual information. Notably,
semantic matching models with DURA even outperform StAR (Self-Adp) on FB15k-237, which use external textual information to enhance the expressiveness of KGE.
Generally, models with more parameters and datasets with smaller sizes imply a larger risk of overfitting. Among the three datasets, WN18RR has the smallest size of only $11$ kinds of relations and around $80k$ training samples. Therefore, the improvements brought by DURA on WN18RR are expected to be larger compared with other datasets, which is consistent with the experiments.
As stated in  \cite{kge_survey}, RESCAL is a more expressive model, but it is prone to overfit on small- and medium-sized datasets because it represents relations with much more parameters. For example, on WN18RR dataset, RESCAL gets an H@10 score of 0.493, which is lower than ComplEx (0.522). The advantage of its expressiveness does not show up at all. Incorporated with DURA, RESCAL gets an 8.4\% improvement on H@10 and finally attains 0.577, outperforming all compared models. On larger datasets such as YAGO3-10, overfitting also exists but will be insignificant. Nonetheless, DURA still leads to consistent improvement, showing the ability of DURA to prevent models from overfitting.

\subsubsection{Comparison with Other Regularizers}\label{sec:comp}
We compare DURA to the squared Frobenius norm regularizer and the tensor nuclear 3-norm (N3) regularizer \cite{n3}.  
The squared Frobenius norm regularizer  $g(\hat{\mathcal{X}})= \|\headmat\|_F^2+\|\tailmat\|_F^2+\sum_{j=1}^{|\mathcal{R}|}\|\textbf{R}_j\|_F^2$.
N3 regularizer is given by $g(\hat{\mathcal{X}})= \sum_{d=1}^D(\|\heademb_{:d}\|_3^3+\|\textbf{r}_{:d}\|_3^3+\|\tailemb_{:d}\|_3^3)$.
where $\|\cdot\|_3$ denotes $L_3$ norm of vectors.
N3 regularizer is only suitable for models in which $\textbf{R}_j$ is diagonal, such as CP or ComplEx. 

\begin{table*}[t]
    \centering
    \caption{\modifyy{Hyper-parameters found by grid search. $\lambda$ is the regularization coefficients, $\lambda_1$ and $\lambda_2$ are weights for different parts of the regularizer for DURA2, and $\lambda_3$ and $\lambda_4$ are weights for different parts of the regularizer for DURA1.}}\label{table:hp_tkbc}
    \label{table:temperal_hp}
    \begin{tabular}{l  c c c c c  c c c c c  c c c c c }
    \toprule
        &\multicolumn{5}{c}{\textbf{ICEWS14}}&  \multicolumn{5}{c}{\textbf{ICEWS05-15}} & \multicolumn{5}{c}{\textbf{YAGO15k}}\\
         \cmidrule(lr){2-6}
         \cmidrule(lr){7-11}
         \cmidrule(lr){12-16}
         & $\lambda$ & $\lambda_1$ & $\lambda_2$ & $\lambda_3$ &$\lambda_4$ & $\lambda$ & $\lambda_1$ & $\lambda_2$ & $\lambda_3$ &$\lambda_4$ & $\lambda$ & $\lambda_1$ & $\lambda_2$ & $\lambda_3$ &$\lambda_4$ \\
        \midrule
        TComplEx-DURA1 & 1e-2 & - & - & 1e-3  & 1e2 & 1e-3 & - & - & 3e-2 & 30 & 1e-2 & - & - & 1e-2 & 1e2\\
        TComplEx-DURA2 & 1e-1 & 1e-1 & 3e-2 & - & - & 1e-1  & 1e-2 & 3e-2 & - & - & 1e-2 & 3e-2 & 3e-2 & - & -\\
        TRESCAL-DURA1  & 1e-2 & - & - & 1e1 & 1e1 & 1e-2 & - & - & 1e1 & 1 & 1e-2 & - & - & 1e-1 & 1e1\\
        \bottomrule
    \end{tabular}
\end{table*}

We implement both the squared Frobenius norm (FRO) and N3 regularizers in the weighted way as stated in \cite{n3}.
Table \ref{table:ablation_results} shows the performance of the three regularizers on three popular models: CP, ComplEx, and RESCAL. Note that when the considered model is RESCAL, we only compare DURA to the squared Frobenius norm regularization as N3 does not apply to it.

\begin{table*}[ht]
    \centering
    \zqzhang{
    \caption{
    \zqzhang{Evaluation results on ICEWS14, ICEWS05-15, and YAGO15k datasets. }
    }\label{table:t_main_results}
    \begin{tabular}{l  c c c  c c c  c c c }
        \toprule
          &\multicolumn{3}{c}{\textbf{ICEWS14}}&  \multicolumn{3}{c}{\textbf{ICEWS05-15}} & \multicolumn{3}{c}{\textbf{YAGO15k}}\\
         \cmidrule(lr){2-4}
         \cmidrule(lr){5-7}
         \cmidrule(lr){8-10}
         & MRR & H@1 & H@10 & MRR & H@1 & H@10 & MRR & H@1 & H@10 \\
        \midrule
        TTransE      & .26 & .07 & .60 & .27 & .08 & .62 & .32 & .23 & .51\\
        TA-TransE    & .28 & .10 & .63 & .30 & .10 & .67 & .32 & .23 & .51\\
        RFTE-HAKE    &.50 &.40 &.70 &.47 &.37 &.65 & - & - & - \\
        BoxTE (k=2)  &.62 &.53 &.77 &.66 &.58 &\textbf{.82} & - & - & -\\
        \midrule
        TComplEx & .60 & .51 & .75 & .65 & .57 & .79 & .35 & .27 & .52 \\
        TComplEx-FRO & .60 & .53 & .74 & .65 & .57 & .80 & .34 & .27 & .49 \\
        TComplEx-N3   &.61  &.53 &.77  &.66  &.59   &.80  &.36  &.28  &.54  \\
        TComplEx-DURA1 & \textbf{.64} & \textbf{.56} & \textbf{.79} & \textbf{.67} & .59 & \textbf{.81} & \textbf{.38} & \textbf{.30} & \textbf{.55}\\
        TComplEx-DURA2 & .62 & .54 & .77 & \textbf{.67} & .59 & \textbf{.81} & .35 & .27 & .52\\
        \midrule
        TRESCAL   & .56 & .47 & .73 & .65 & .57 & .79 & .28 & .21 & .45 \\
        TRESCAL-FRO   & .55 & .48 & .69 & .60 & .52 & .75 & .32 & .25 & .46 \\
        TRESCAL-DURA1    &.60 &.51 &.76 &\textbf{.67} & \textbf{.60} &.80 &.30 &.22 &.48\\
        \bottomrule
    \end{tabular}
    }
\end{table*}

For CP and ComplEx, DURA brings consistent improvements compared to FRO and N3 on all datasets. Specifically, on FB15k-237, compared to CP-N3, CP-DURA gets an improvement of 0.013 in terms of MRR. Even for the previous state-of-the-art SM model ComplEx, DURA brings further improvements against the N3 regularizer. Incorporated with FRO, RESCAL performs worse than the vanilla model, which is consistent with the results in \cite{old_dog}.  However, RESCAL-DURA brings significant improvements against RESCAL. All the results demonstrate that DURA is more widely applicable than N3 and more effective than the squared Frobenius norm regularizer.

\subsection{Temporal Knowledge Graph Completion}
We also conduct extensive experiments on temporal knowledge graph completion datasets. Specifically, we introduce the experimental settings for temporal KGC in Section \ref{sec:t_exp_setting} and show the effectiveness of DURA on three benchmark datasets in Section \ref{sec:t_main_results}.

\subsubsection{Experimental Settings for Temporal KGC}\label{sec:t_exp_setting}

We consider three public temporal knowledge graph datasets---ICEWS14 \cite{icew}, ICEWS05-15 \cite{icew}, and YAGO15k \cite{yago15k} for the knowledge graph completion task, which have been divided into training, validation, and testing set in previous works. The statistics of these datasets are shown in Appendix D.

Following \cite{tcomplex}, we use the following temporal regularizer for TComplEx to smooth timestamp embeddings
\begin{align*}
    \Lambda_3(\mathcal{T}) = \frac{1}{|\mathcal{T}|-1}\sum_{l=1}^{|\mathcal{T}|-1}\|\diag^{-1}(\timemat_{l+1}-\timemat_l)\|_3^3,
\end{align*}
where $\diag^{-1}(\cdot)$ gives the vector made of the diagonal elements from the a matrix. For TRESCAL, we propose to use the temporal regularizer
\begin{align*}
    \Lambda_2(\mathcal{T}) = \frac{1}{|\mathcal{T}|-1}\sum_{l=1}^{|\mathcal{T}|-1}\|\timemat_{l+1}-\timemat_l\|_F^2
\end{align*}
to smooth timestamp embeddings.

Moreover, we find it better to assign different weights for the parts involved with relations. That is, the optimization problem has the form of
\begin{align*}
    \min \sum_{(\head_i,r_j,\tail_k,t_l)\in\mathcal{S}}[&\ell_{ijkl}(\textbf{U},\textbf{R}_1,\dots,\textbf{R}_J,\textbf{V},\timemat_1,\dots,\timemat_L)\\
    +&\lambda (\lambda_1 (\|\overline{\textbf{u}}_i\textbf{R}_j\|_2^2+\|\textbf{v}_k\timemat_l\|_2^2)\\
    +&\lambda_2 (\|\overline{\textbf{u}}_i\timemat_l\|_2^2+\|\textbf{v}_k\textbf{R}_j\|_2^2)\\
    +&\lambda_3 (\|\textbf{u}_i\|_2^2+\|\textbf{v}_k\|_2^2)\\
    +&\lambda_4 (\|\overline{\textbf{u}}_i(\textbf{R}_j\odot  \timemat_l)\|_2^2+\|\textbf{v}_k(\textbf{R}_j \odot \timemat_l)\|_2^2))]
\end{align*}
where $\lambda,\lambda_1,\lambda_2,\lambda_3,\lambda_4 \geq 0$ are fixed hyperparameters.
We denote the regularizer \eqref{reg:tcomplex_1} by DURA2 (i.e.,  $\lambda_3=\lambda_4=0$) and the regularizer \eqref{reg:tcomplex_2}  by DURA1 (i.e., $\lambda_1=\lambda_2=0$).

We use grid search to find the best hyperparameters based on the performance on the validation datasets. 
\modifyy{
We search $\lambda \in$ \{1, 3e-1, 1e-1, 3e-2, 1e-2, 3e-3, 1e-3, 3e-4, 1e-4\} for all experiments. We search $\lambda_1,\lambda_2\in$ \{1, 3e-1, 1e-1, 3e-2, 1e-2, 3e-3, 1e-3\} for DURA2, and $\lambda_3\in$ \{1e2, 3e1, 1e1, 3, 1, 3e-1, 1e-1, 3e-2, 1e-2, 3e-3, 1e-3, 3e-4, 1e-4\}, $\lambda_4 \in$ \{1e3, 3e2, 1e2, 3e1, 1e1, 3, 1, 3e-1, 1e-1, 3e-2\} for DURA1.
Table \ref{table:temperal_hp} shows the best hyperparameters found by grid search.
}

\subsubsection{Main Results on Temporal KGC}\label{sec:t_main_results}
We compare DURA with the popular squared Frobenius norm and the tensor nuclear 3-norm (N3) regularizers \cite{tcomplex} on the  state-of-the-art temporal semantic matching KGE model TComplEx and our proposed TRESCAL method.  \zqzhang{We also include four recent distance-based models, including TTransE \cite{ttranse}, TA-TransE \cite{ta-transe}, RFTE-HAKE \cite{rfte} and BoxTE \cite{boxte}.}
\modifyy{
TTransE \cite{ttranse} proposes a time-aware knowledge graph completion model by using temporal order information among facts. TA-TransE \cite{ta-transe} utilizes recurrent neural networks to learn time-aware representations. RFTE \cite{rfte} is a framework to transplant static KGE models to temporal KGE models. It treats the sequence of graphs as a Markov chain. BoxTE \cite{boxte} assumes that time is a filter that helps pick out answers to be correct during certain periods. Therefore, it introduces boxes to represent a set of answer entities to a time-agnostic query.
}

Table \ref{table:t_main_results} shows the effectiveness of DURA. Note that DURA2 as shown in the formulation \eqref{reg:tcomplex_2} does not applicable to TRESCAL since the matrix multiplication and element-wise multiplication are not commutative. Moreover, the N3 regularizer is not applicable to TRESCAL since this model is not CP tensor factorization. Table \ref{table:t_main_results} shows that DURA effectively improve the performance of TComplEx and TRESCAL. \zqzhang{With DURA, these semantic marching models outperform the state-of-the-art distance-based models in terms of MRR.}

\begin{figure}[ht]
\begin{subfigure}{0.5\columnwidth}
  \includegraphics[width=150pt]{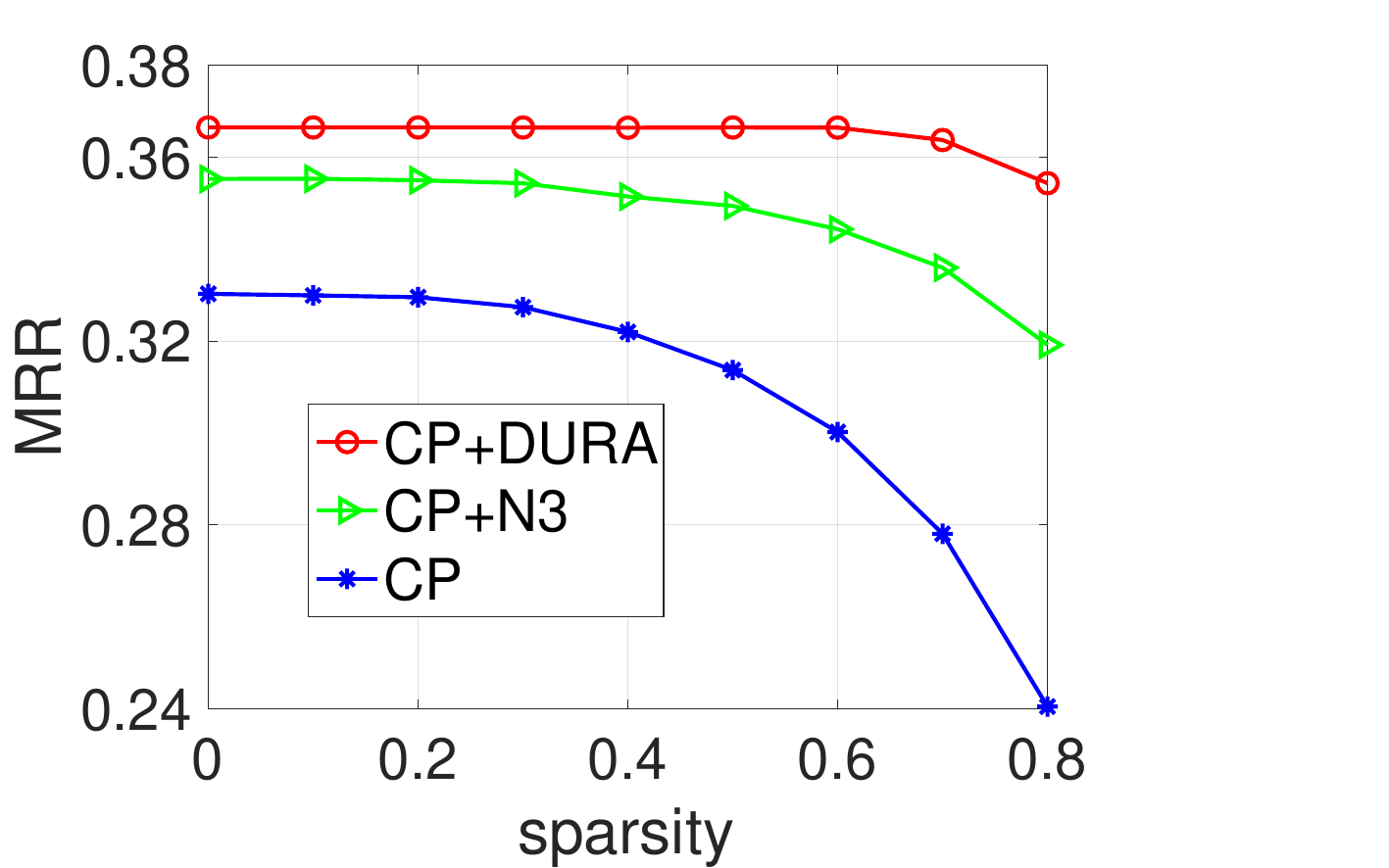}
  
  \caption{CP}\label{fig:cp_sparsity_mrr}
\end{subfigure}
\hspace{-3mm}
\medskip
\begin{subfigure}{0.5\columnwidth}
  \includegraphics[width=161pt]{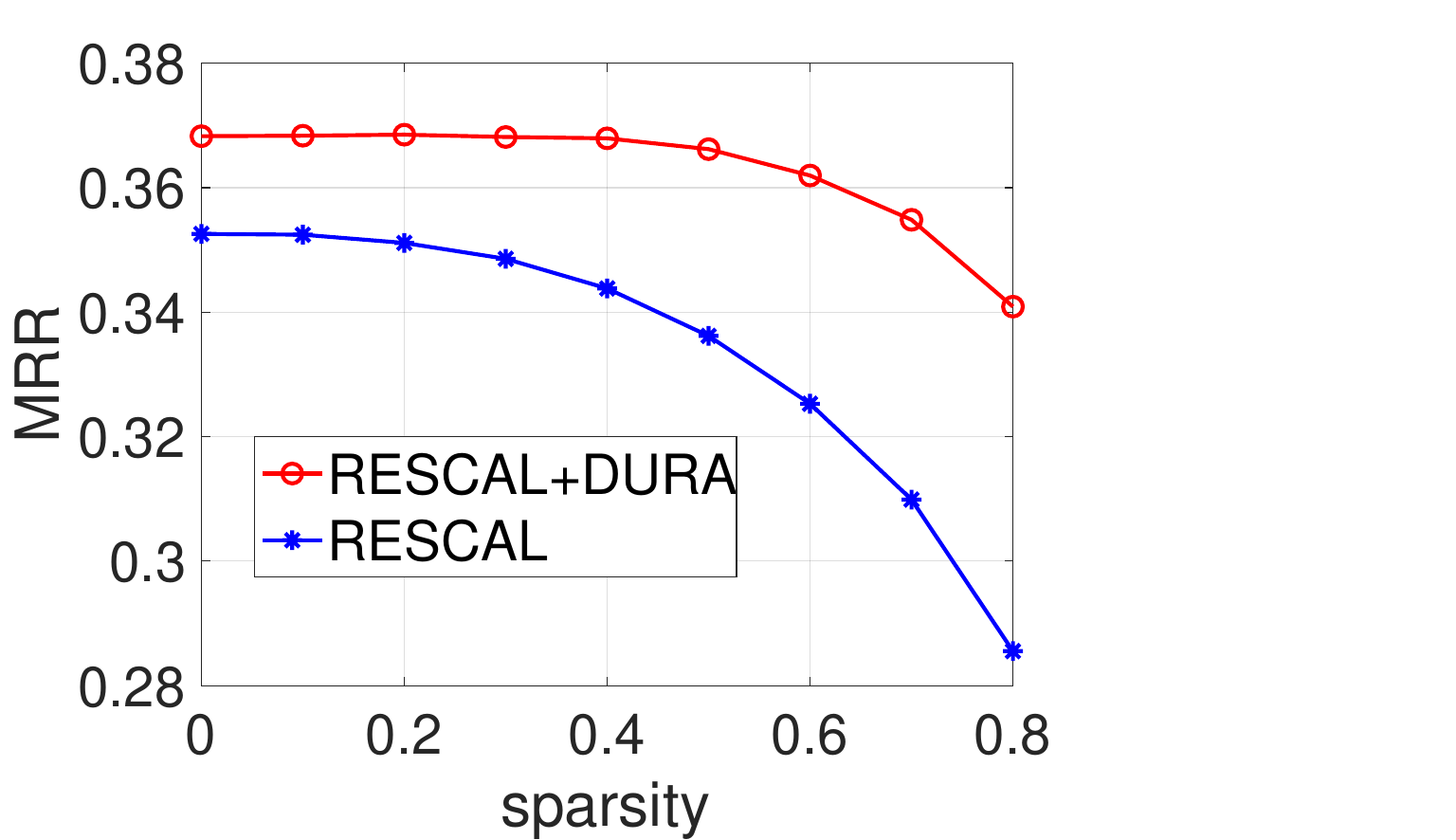}
  \caption{RESCAL}\label{fig:rescal_sparsity_mrr}
\end{subfigure}
\caption{The effect of entity embeddings' $\lambda$-sparsity on MRR. The experiments are conducted on FB15k-237.}
\label{fig:sparsity_mrr}
\end{figure}
\subsection{Sparsity Analysis}\label{sec:sparsity}
As knowledge graphs usually contain billions of entities, the storage of entity embeddings faces severe challenges. Intuitively, if embeddings are sparse, that is, most of the entries are zero, we can store them with less storage. Therefore, the sparsity of the generated entity embeddings becomes crucial for real-world applications. 
Generally, there are few entries of entity embeddings that are exactly equal to 0 after training, which means that it is hard to obtain sparse entity embeddings directly. However, when we score triples using the trained model, the embedding entries with values close to 0 will have minor contributions to the score of a triplet. If we set the embedding entries close to 0 to be exactly 0, we could transform the embeddings into sparse ones. Therefore, there is a trade-off between the sparsity and performance decrement. 
We define the following $\lambda$-sparsity to indicate the proportion of entries that are close to zero:
\begin{align}\label{eqn:sparsity}
    s_{\lambda} = \frac{\sum_{i=1}^{I} \sum_{d=1}^{D}\mathbbm{1}_{\{|x| < \lambda\}}(\textbf{E}_{id})}{I \times D},
\end{align}
where $\textbf{E} \in \mathbb{R}^{I \times D}$ is the entity embedding matrix, $\textbf{E}_{id}$ is the entry in the $i$-th row and $d$-th column of $\textbf{E}$, $I$ is the number of entities, $D$ is the embedding dimension, and $\mathbbm{1}_\mathcal{C}(x)$ is the indicator function taking value of 1 if $x\in\mathcal{C}$ or otherwise $0$. 

\begin{table}[ht]
    \centering
    \caption{Ablation results on the combined and uncombined regularizers. ``-I'' and ``-II'' represents the regularizer defined in formula (\ref{reg:f}) and (\ref{reg:b}), respectively.}
    \label{table:reg_direction_ablation}
        \begin{tabular}{l c c c c}
            \toprule
              &\multicolumn{2}{c}{\textbf{WN18RR}}&  \multicolumn{2}{c}{\textbf{FB15k-237}}\\
             \cmidrule(lr){2-3}
             \cmidrule(lr){4-5}
             & MRR & H@10 & MRR & H@10 \\
            \midrule
            CP         & .438 & .485 & .333 & .508 \\
            CP-DURA-I & .463 & .535 & .331 & .522 \\
            CP-DURA-II & .439 & .492 & .332 & .507 \\
            CP-DURA  & \textbf{.478} & \textbf{.552} & \textbf{.367} & \textbf{.555} \\
            \midrule
            RESCAL         & .455 & .493 & .353 & .528 \\
            RESCAL-DURA-I & .475 & .541 & .341 & .523 \\
            RESCAL-DURA-II & .462 & .524 & .342 & .510\\
            RESCAL-DURA  & \textbf{.498} & \textbf{.577} & \textbf{.368} & \textbf{.550} \\
            \bottomrule
        \end{tabular}
\end{table}

Figure \ref{fig:sparsity_mrr} shows the effect of entity embeddings' $\lambda$-sparsity on MRR.
Following Equation \eqref{eqn:sparsity}, we select entries of which the absolute value are less than a threshold and set them to be 0. Note that for any given $s_{\lambda}$, we can always find a proper threshold $\lambda$ to approximate it, as the formula is increasing with respect to $\lambda$.
Results in the figure show that DURA causes much gentler performance decrement as the embedding sparsity increases. In Figure \ref{fig:cp_sparsity_mrr}, incorporated with DURA, CP maintains MRR of 0.366 unchanged even when 60\% entries are set to 0. More surprisingly, when the sparsity reaches 70\%, CP-DURA can still outperform CP-N3 with zero sparsity. For RESCAL, when set 80\% entries to be 0, RESCAL+DURA still has the MRR of 0.341, which significantly outperforms vanilla RESCAL, whose MRR has decreased from 0.352 to 0.286. In a word, incorporating with DURA regularizer, the performance of CP and RESCAL remains comparable to the state-of-the-art models, even when 70\% of entity embeddings' entries are set to 0. 

Following \cite{sparse}, we store the sparse embedding matrices using compressed sparse row (CSR) or compressed sparse column (CSC) format. Experiments show that DURA brings about 65\% fewer storage costs for entity embeddings when 70\% of the entries are set to 0.

\begin{figure*}[t]
    \centering 
\begin{subfigure}{0.55\columnwidth}
  \includegraphics[width=130pt]{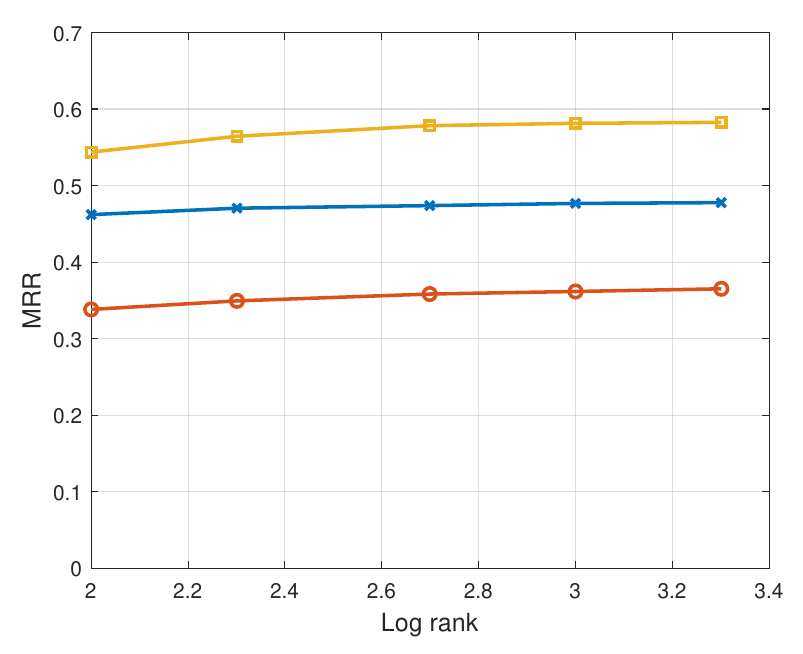}
  \caption{Rank for CP.}
\end{subfigure}\hfil 
\medskip
\begin{subfigure}{0.55\columnwidth}
  \includegraphics[width=130pt]{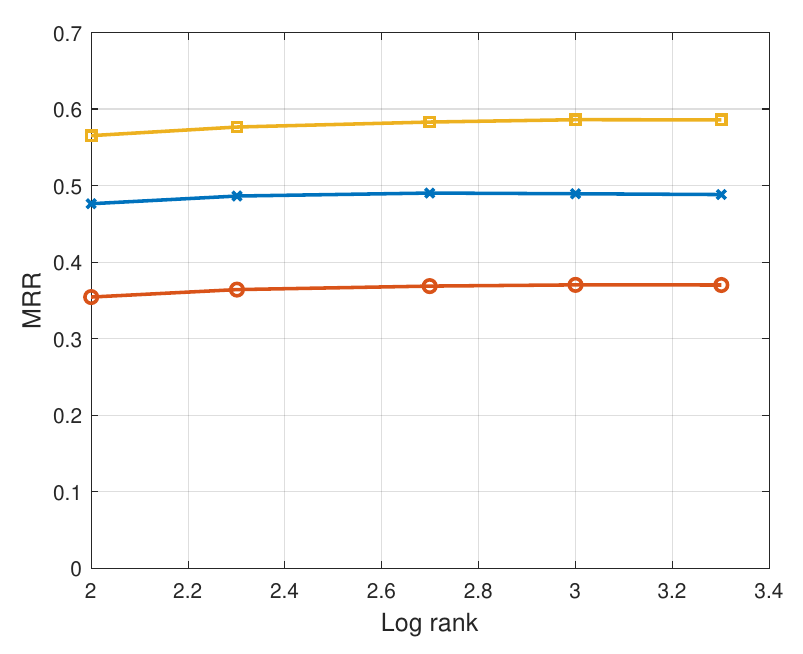}
  \caption{Rank for ComplEx.}
\end{subfigure}\hfil 
\medskip
\begin{subfigure}{0.55\columnwidth}
  \includegraphics[width=130pt]{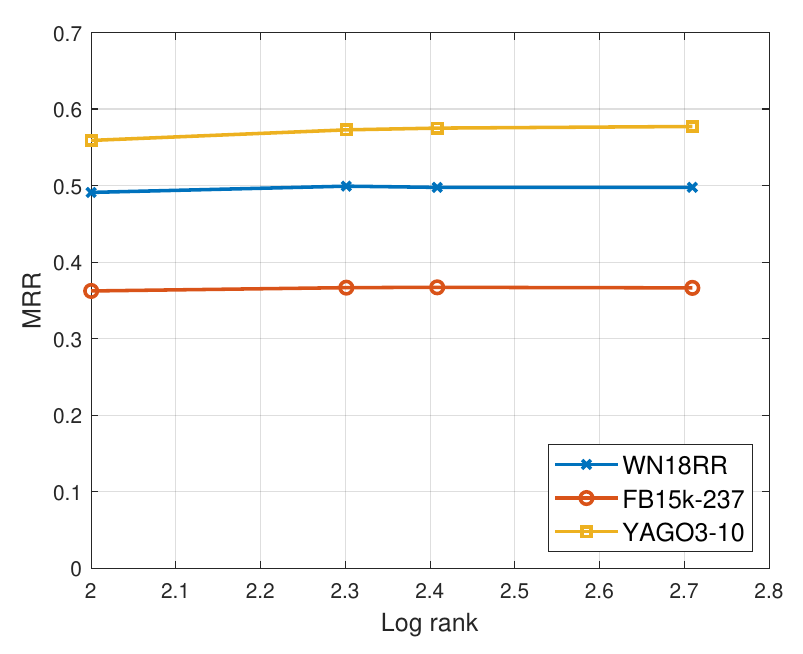}
  \caption{Rank for RESCAL.}
\end{subfigure}\hfil 

\begin{subfigure}{0.55\columnwidth}
  \includegraphics[width=130pt]{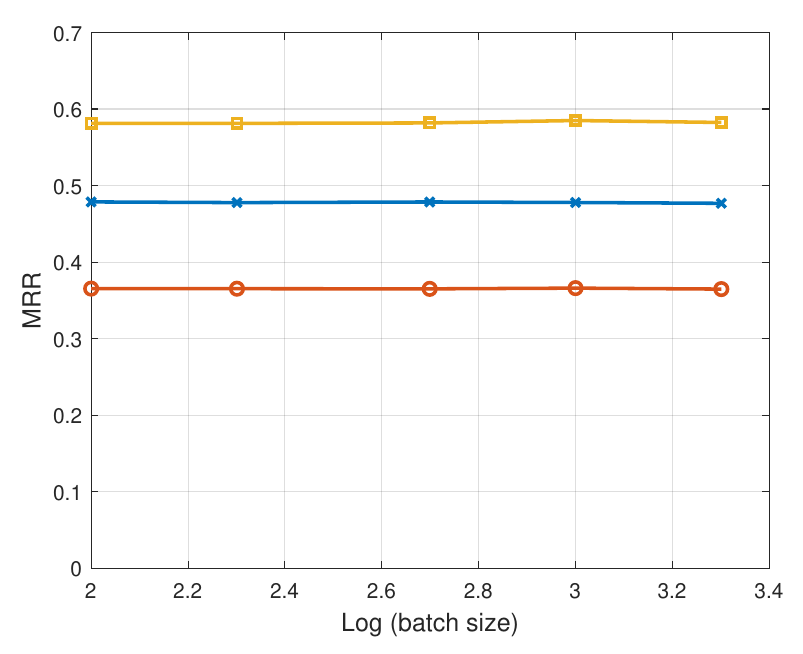}
  \caption{Batch size  for CP.}
\end{subfigure}\hfil 
\medskip
\begin{subfigure}{0.55\columnwidth}
  \includegraphics[width=130pt]{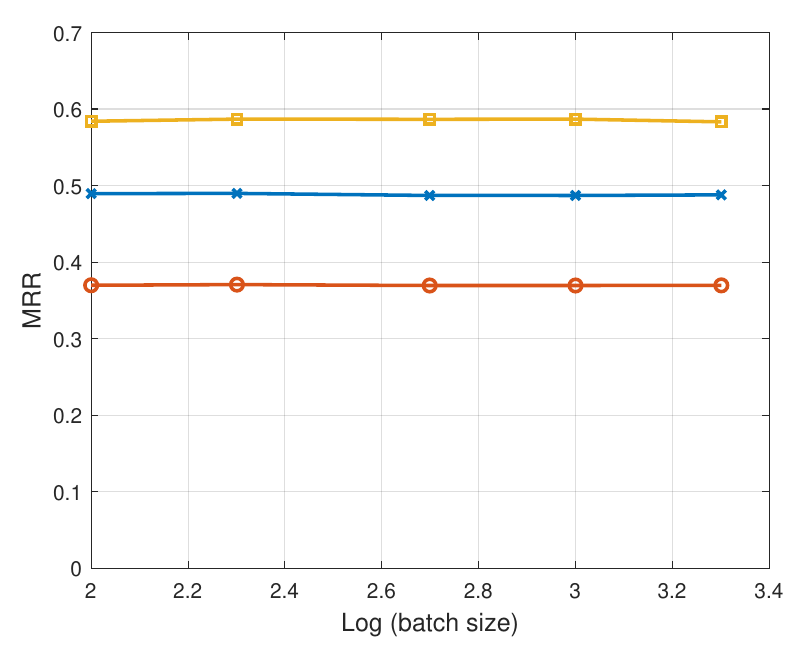}
  \caption{Batch size  for ComplEx.}
\end{subfigure}\hfil 
\medskip
\begin{subfigure}{0.55\columnwidth}
  \includegraphics[width=130pt]{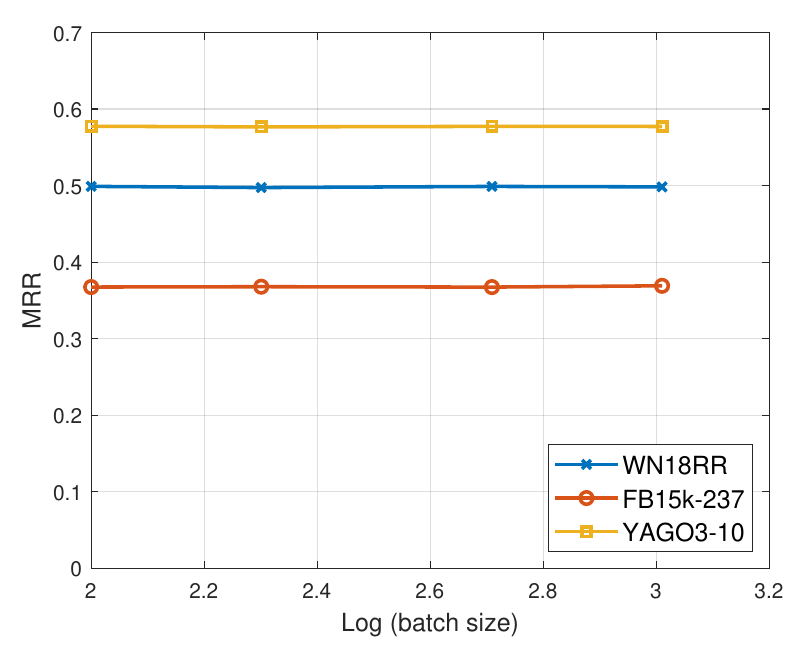}
  \caption{Batch size for RESCAL.}
\end{subfigure}\hfil 
\caption{
\modifyy{
Sensitivity to hyper-parameters of DURA. The x-axis is the logarithm of each hyper-parameter based on 10 and The y-axis is mean reciprocal rank (MRR) on test data.}
}
\label{fig:sens_dura_bsize}
\end{figure*}

\begin{figure}[!ht]
    \centering 
\begin{subfigure}{0.24\textwidth}
  \includegraphics[width=135pt]{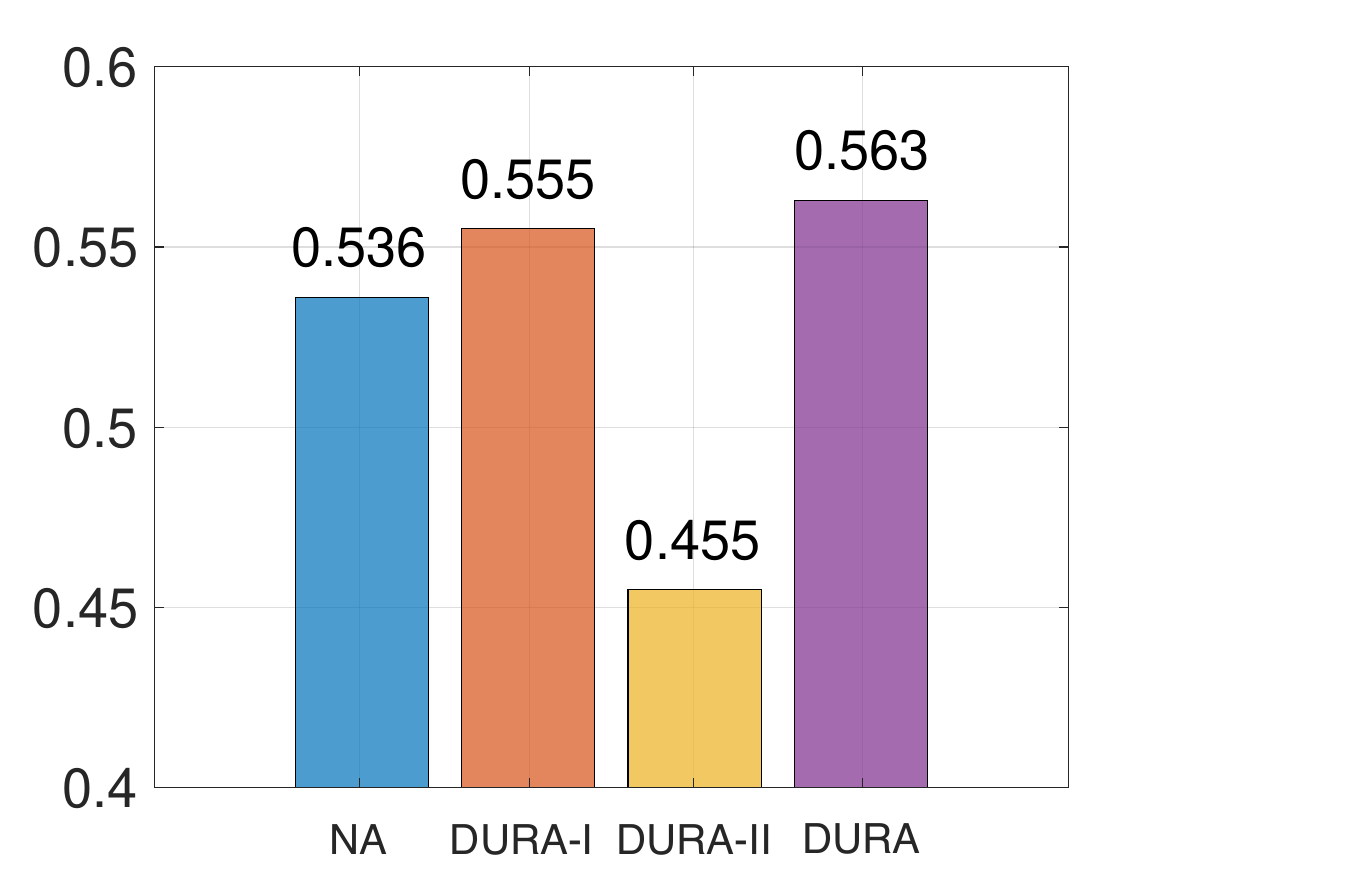}
  \label{fig:1_to_1_mrr}
  \vskip -0.1in
  \caption{1-1 relations}
\end{subfigure}\hfil 
\hspace{-2mm}
\medskip
\begin{subfigure}{0.24\textwidth}
  \includegraphics[width=135pt]{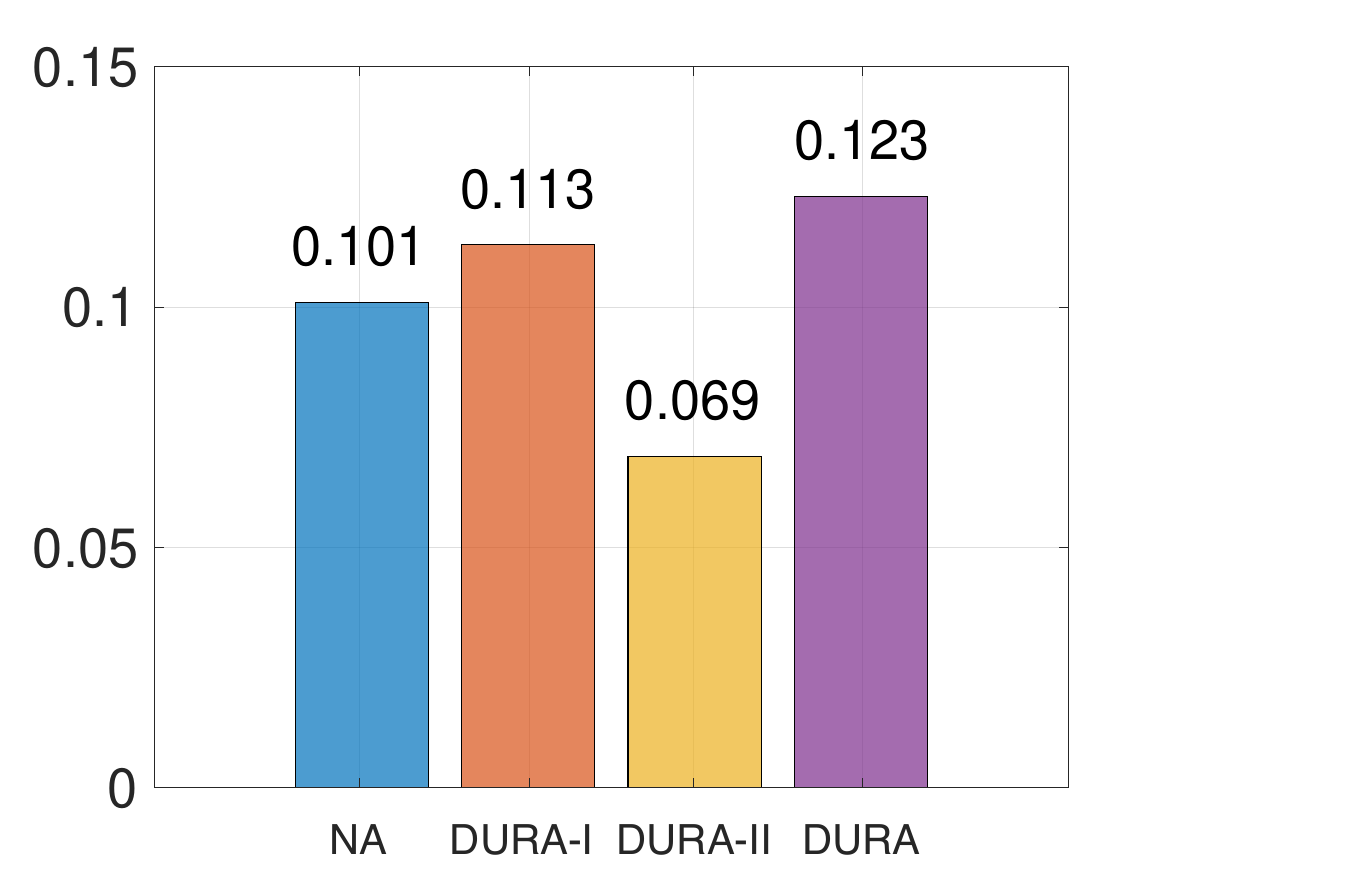}
  \label{fig:1_to_n_mrr}
  \vskip -0.1in
  \caption{1-N relations}
\end{subfigure}
\medskip
\begin{subfigure}{0.24\textwidth}
  \includegraphics[width=135pt]{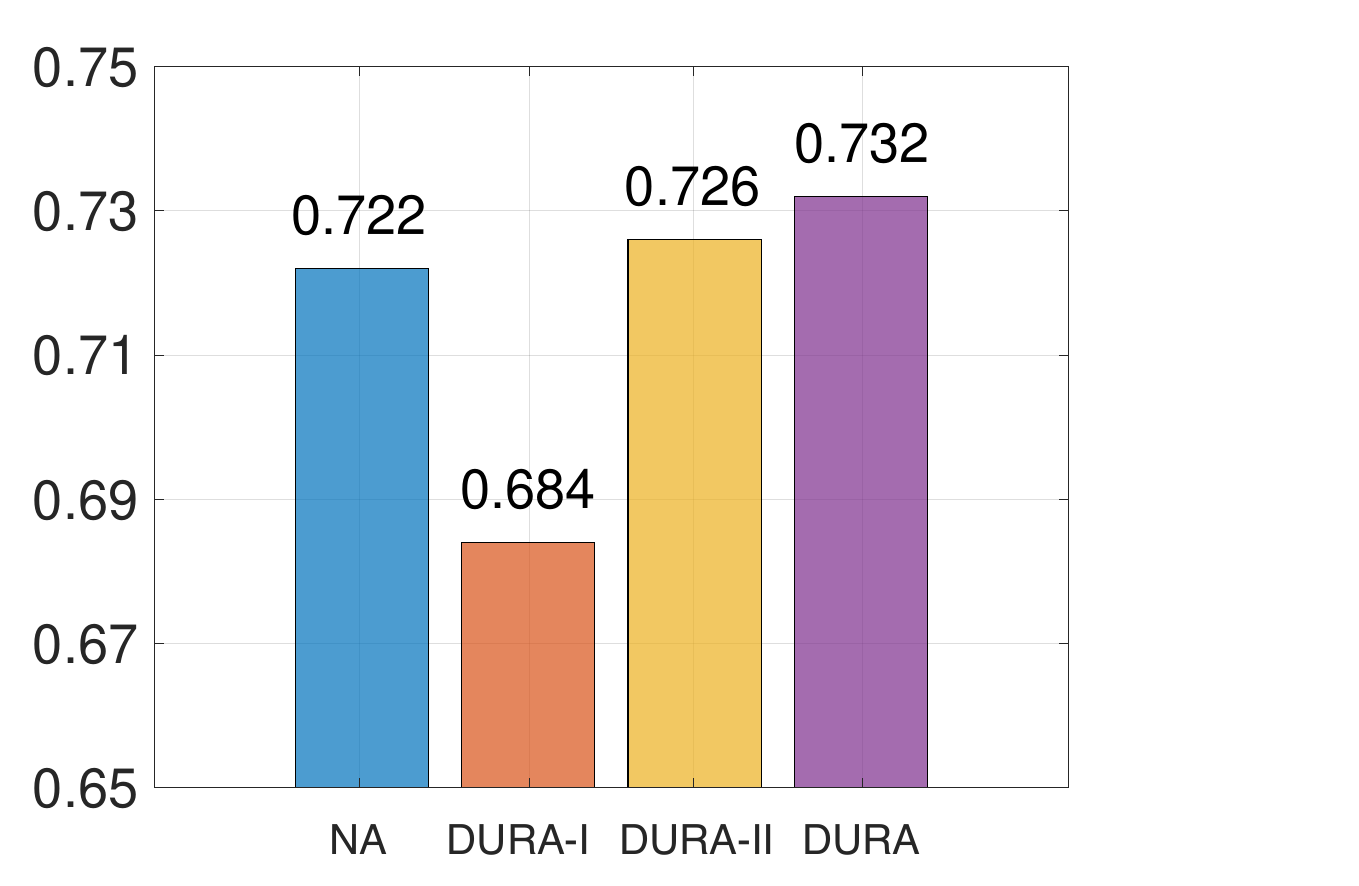}
  \label{fig:n_to_1_mrr}
  \vskip -0.1in
  \caption{N-1 relations}
\end{subfigure}\hfil 
\hspace{-2mm}
\medskip
\begin{subfigure}{0.24\textwidth}
  \includegraphics[width=135pt]{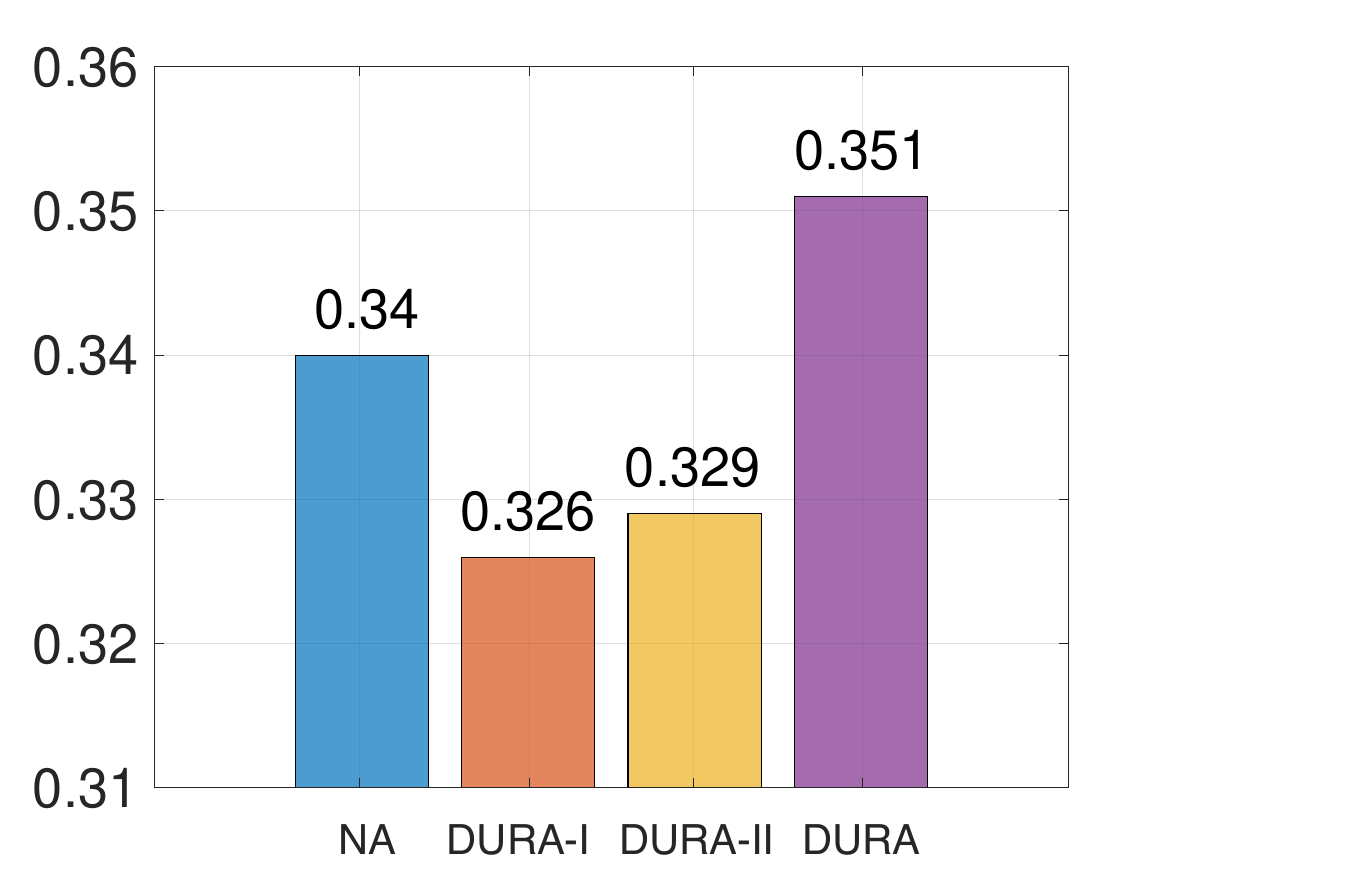}
  \label{fig:n_to_n_mrr}
  \vskip -0.1in
  \caption{N-N relations}
\end{subfigure}
\vskip -0.1in
\caption{Average MRR of RESCAL with different regularizers on different types of relations. The dataset is FB15k-237. 
}
\label{fig:n_to_n_mrrs}
\end{figure}

\subsection{Ablation Study}
We conduct ablation studies to show the effectiveness of the combination of (\ref{reg:f}) and (\ref{reg:b}). The results are shown in Table \ref{table:reg_direction_ablation}, where ``-I'' and ``-II'' represent the regularizer defined in formula (\ref{reg:f}) and (\ref{reg:b}), respectively.
On WN18RR, the uncombined regularizer DURA-I and DURA-II can also improve the performance for the semantic matching models, but the combination brings further improvements.
On FB15k-237, DURA-I and DURA-II even lead to a decrement in terms of MRR. One possible reason is that FB15k-237 has a large number of 1-N and N-1 relations (refer to \cite{transh} for the definitions relation types) and they have different preferences for the uncombined regularizers. 
Figure \ref{fig:n_to_n_mrrs} shows the performance of RESCAL with different regularizers on different relation types. Compared with vanilla RESCAL, DURA-I only improves the MRR for 1-1 and 1-N relations, while DURA-II only improves the MRR for N-1 relations, which validates our analysis. DURA significantly outperforms both DURA-I and DURA-II on all relation types.

\begin{figure}[ht]
\begin{subfigure}{0.24\textwidth}
\centering
  \includegraphics[width=100pt]{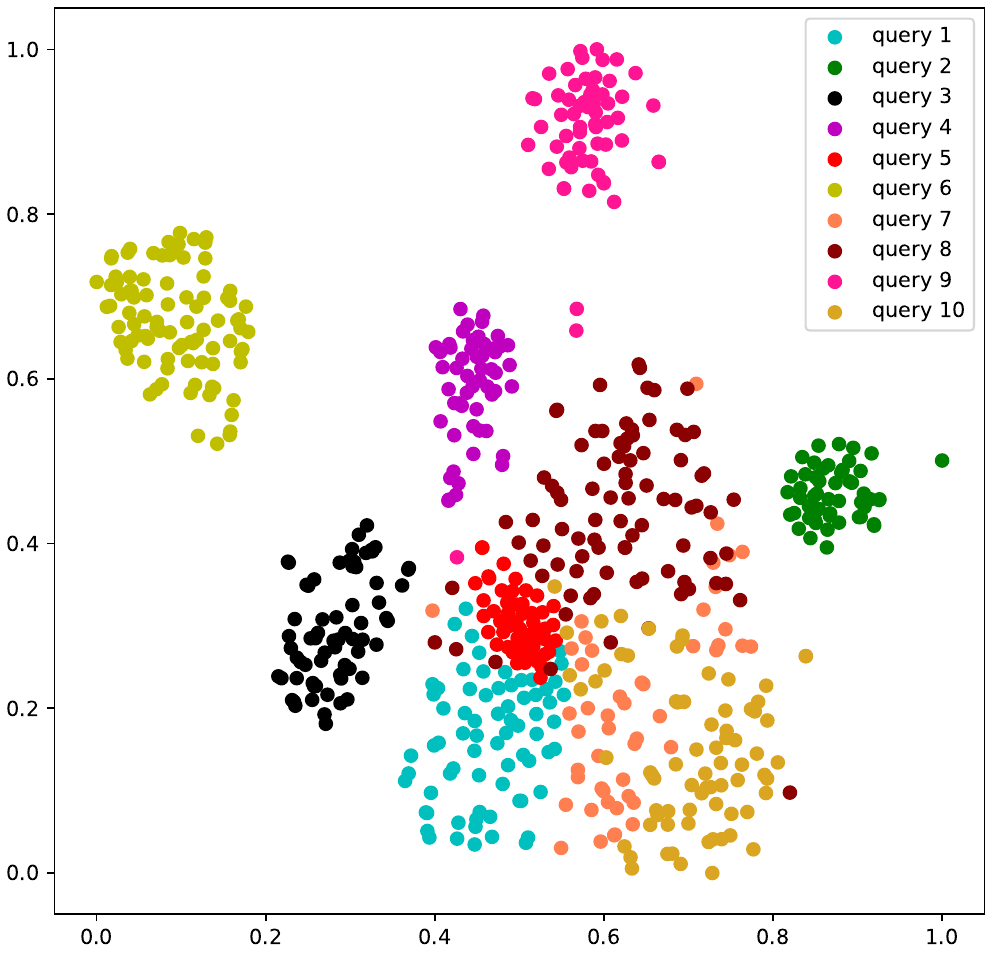}
  \caption{CP}\label{fig:na_tsne}
\end{subfigure}
\begin{subfigure}{0.24\textwidth}
\centering
  \includegraphics[width=100pt]{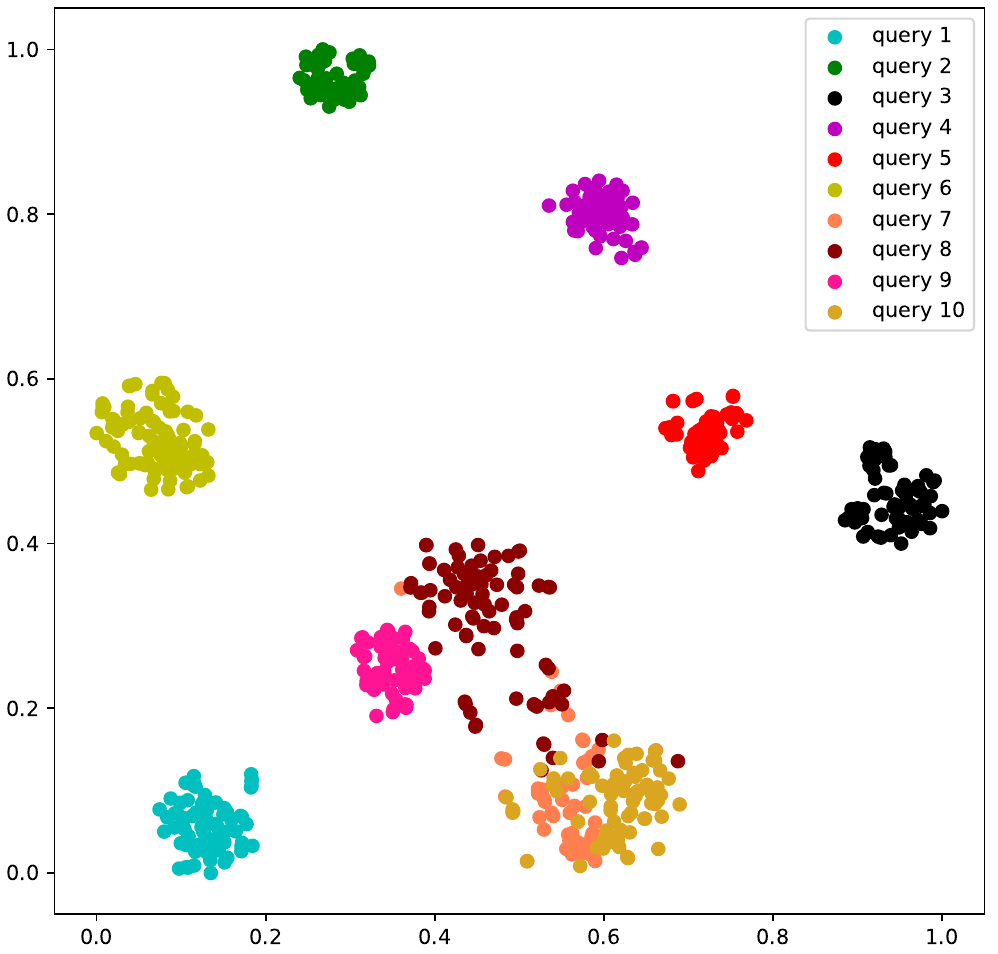}
  \caption{CP-DURA}\label{fig:dura_tsne}
\end{subfigure}
\caption{Visualization of tail entity embeddings using T-SNE. A point represents a tail entity. Points in the same color represent tail entities that have the same $(h_r,r_j)$ context.}
\label{fig:tsne}
\end{figure}

\begin{figure*}[!h]
    \centering 
\begin{subfigure}{0.55\columnwidth}
  \includegraphics[width=130pt]{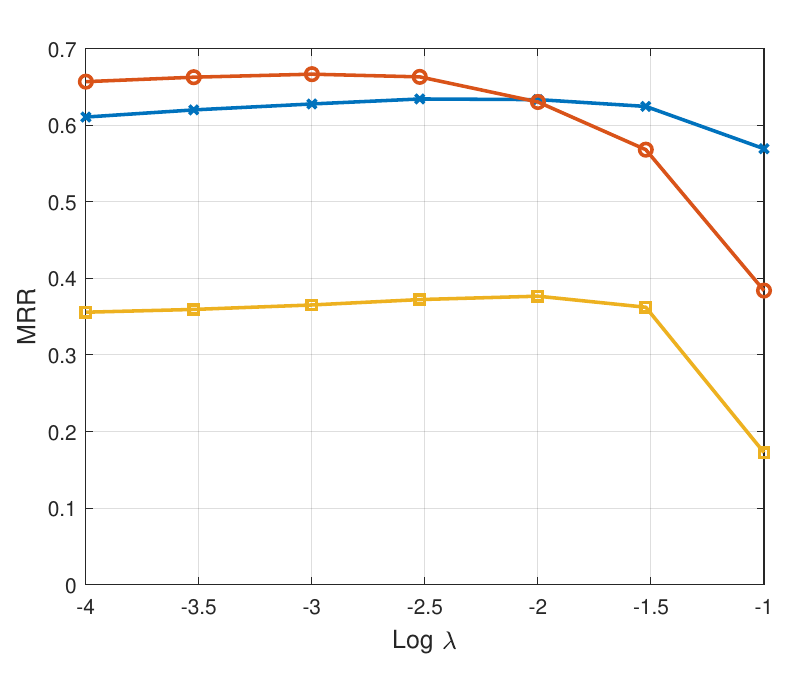}
  \caption{$\lambda$ of DURA1 for TComplEx.}
\end{subfigure}\hfil 
\medskip
\begin{subfigure}{0.55\columnwidth}
  \includegraphics[width=130pt]{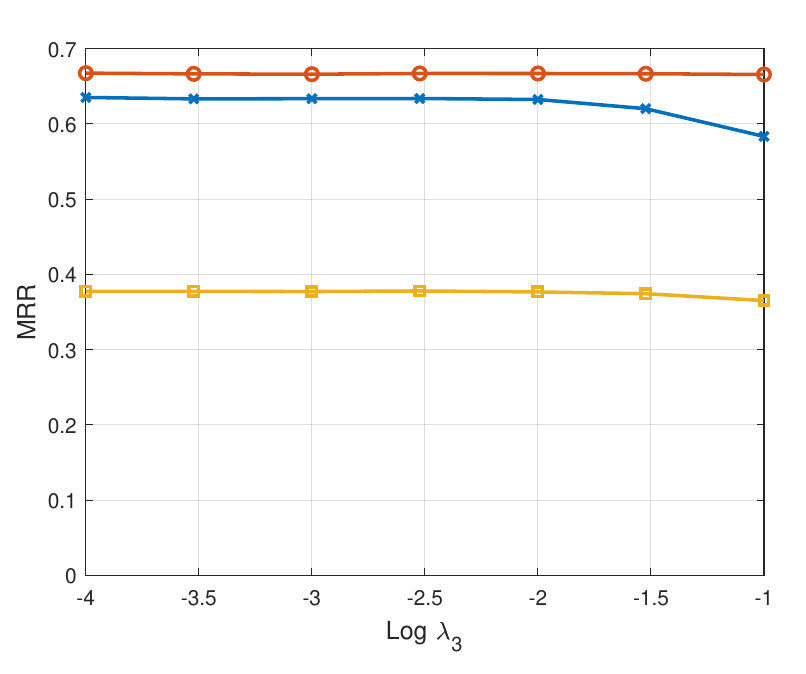}
  \caption{$\lambda_3$ of DURA1 for TComplEx.}
\end{subfigure}\hfil 
\medskip
\begin{subfigure}{0.55\columnwidth}
  \includegraphics[width=130pt]{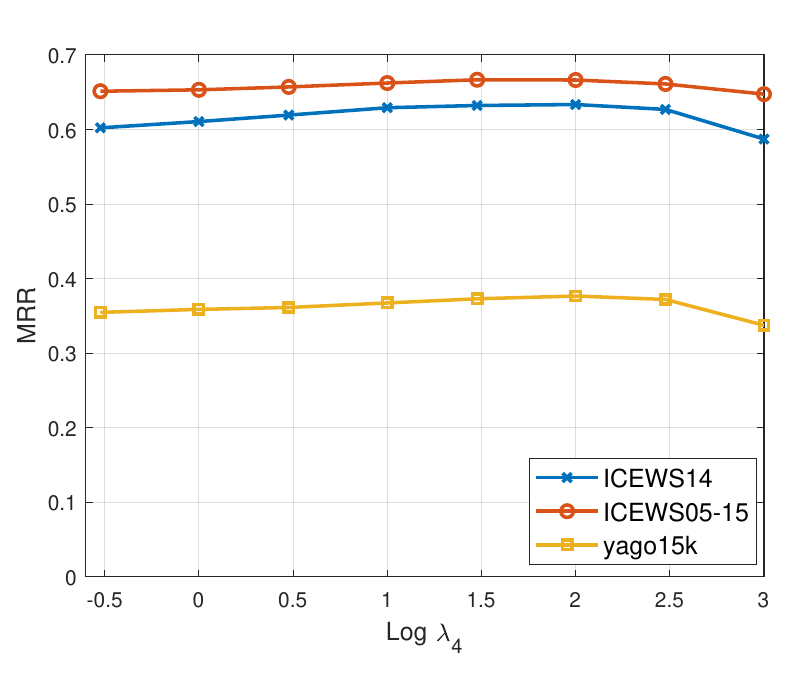}
  \caption{$\lambda_4$ of DURA1 for TComplEx.}
\end{subfigure}\hfil

\begin{subfigure}{0.55\columnwidth}
  \includegraphics[width=130pt]{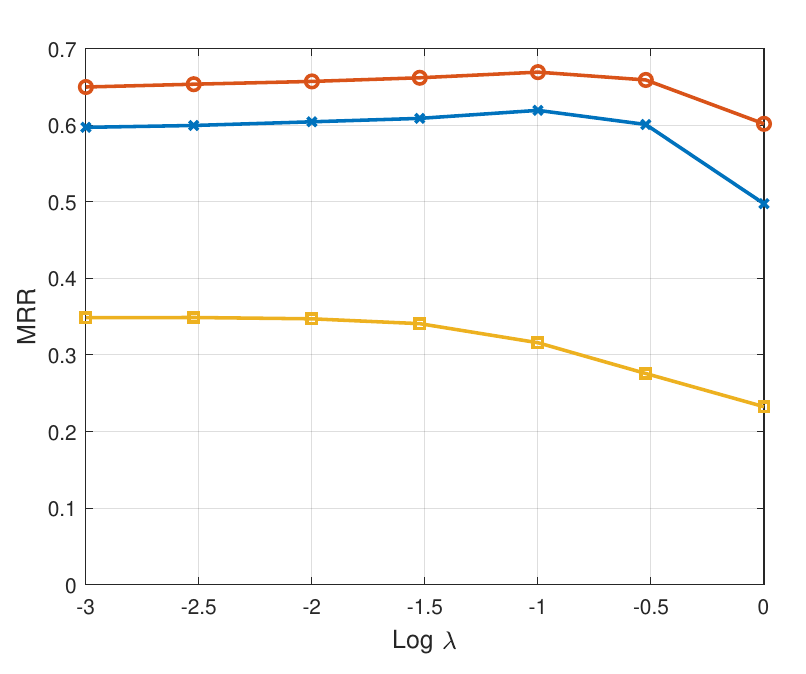}
  \caption{$\lambda$ of DURA2 for TComplEx.}
\end{subfigure}\hfil 
\medskip
\begin{subfigure}{0.55\columnwidth}
  \includegraphics[width=130pt]{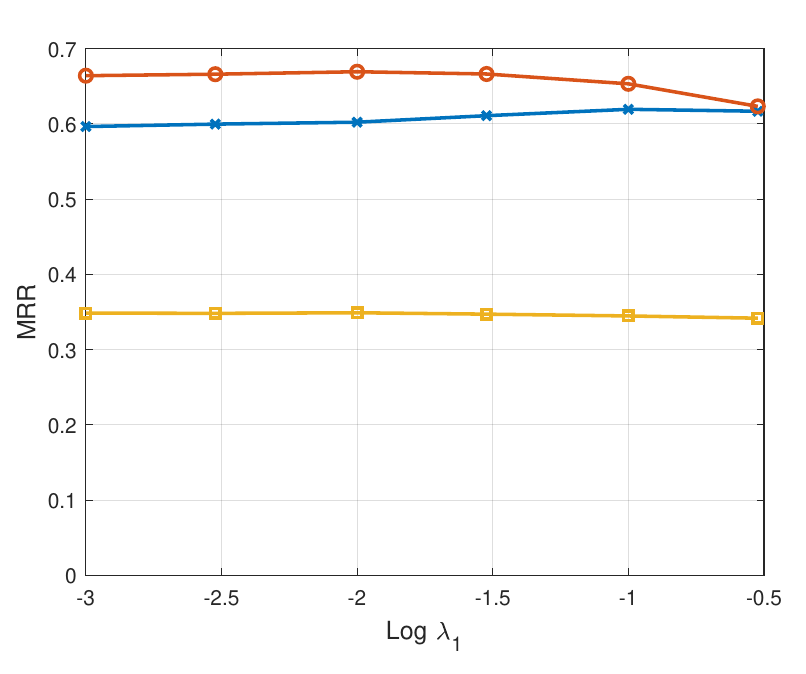}
  \caption{$\lambda_1$ of DURA2 for TComplEx.}
\end{subfigure}\hfil 
\medskip
\begin{subfigure}{0.55\columnwidth}
  \includegraphics[width=130pt]{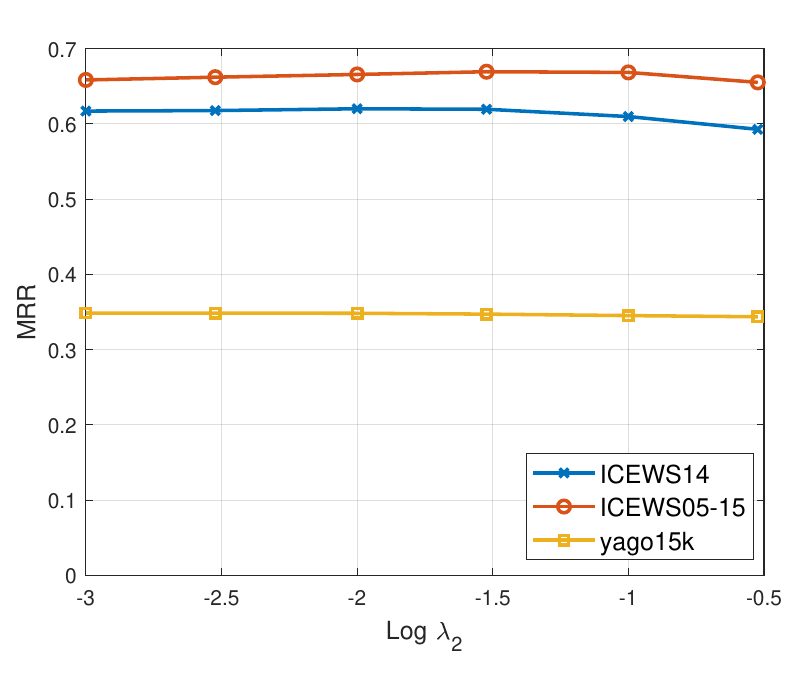}
  \caption{$\lambda_2$ of DURA2 for TComplEx.}
\end{subfigure}\hfil

\begin{subfigure}{0.55\columnwidth}
  \includegraphics[width=130pt]{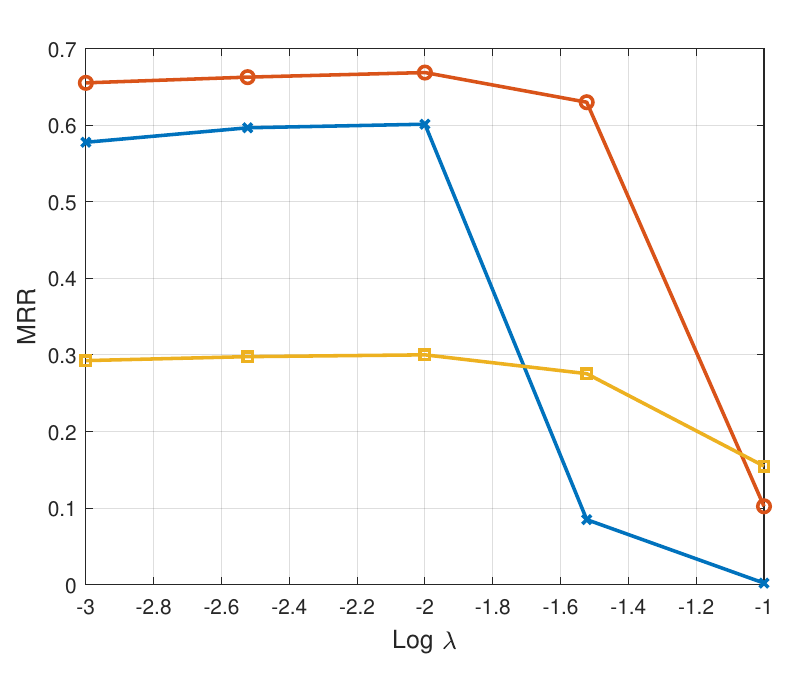}
  \caption{$\lambda$ of DURA1 for TRESCAL}
\end{subfigure}\hfil 
\medskip
\begin{subfigure}{0.55\columnwidth}
  \includegraphics[width=130pt]{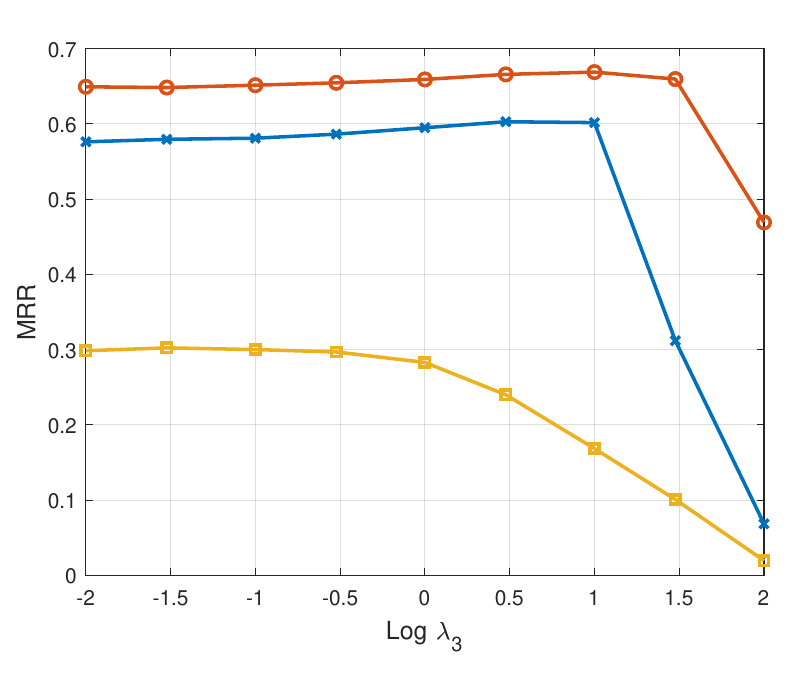}
  \caption{$\lambda_3$ of DURA1 for TRESCAL}
\end{subfigure}\hfil 
\medskip
\begin{subfigure}{0.55\columnwidth}
  \includegraphics[width=130pt]{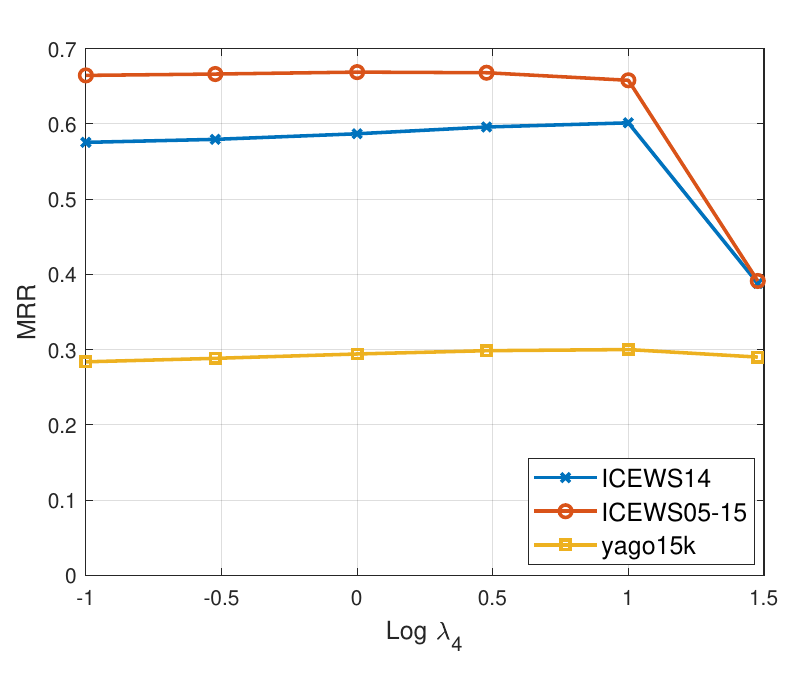}
  \caption{$\lambda_4$ of DURA1 for TRESCAL}
\end{subfigure}\hfil
\caption{
\modifyy{Sensitivity to hyper-parameters of DURA1 and DURA2. The x-axis is the logarithm of each hyper-parameter based on 10 and the y-axis is mean reciprocal rank (MRR) on test data.}
}
\label{fig:sens_tdura}
\end{figure*}

\subsection{Visualization}\label{sec:vis}
We visualize the entity embeddings using T-SNE \cite{tsne} to show that DURA encourages entities with similar semantics to have similar embeddings. Without loss of generality, we use tail entities for visualization.

Suppose that $(\head_i, r_j)$ is a \textit{query}, where $\head_i$ and $r_j$ are head entities and relations, respectively. An entity $\tail_k$ is an \textit{answer} to a query $(\head_i, r_j)$ if $(\head_i, r_j,\tail_k)$ is valid. We randomly selected 10 queries in FB15k-237, each of which has more than 50 answers. Then, we use T-SNE to visualize the answers' embeddings generated by CP and CP-DURA. Figure \ref{fig:tsne} shows the visualization results. Each entity is represented by a 2D point and points in the same color represent tail entities with the same $(\head_i,r_j)$ context (i.e. query). Figure \ref{fig:tsne} shows that, with DURA, entities with the same $(\head_i, r_j)$ contexts are indeed being assigned more similar representations, which verifies the claims in Section \ref{sec:motivation}.



\modifyy{
\subsection{Sensitivity to  Hyper-Parameters}
For static KGE, the hyper-parameters mainly include embedding sizes $k$, batch sizes $b$, and regularization coefficients $\lambda$, $\lambda_i$ ($i=1,2$); for temporal KGE, the hyper-parameters mainly include regularization coefficients $\lambda$, $\lambda_i$ ($i=1,2,3,4)$. We provide the performance curves on different datasets for these hyper-parameters. Note that when exploring the effect of a specific hyper-parameter, we fix the other hyper-parameters as their best values.

As shown in Figure \ref{fig:sens_dura_bsize}, the performance is stable when the embedding and batch sizes vary. Figure \ref{fig:sens_dura} demonstrates that the performance of static KGE is relatively insensitive to $\lambda_i$. However, when the regularization coefficients $\lambda$ are too large, the performance drops quickly. It is expectable since too strong regularization will weaken models' ability to model complex relations. 

As shown in Figure \ref{fig:sens_tdura}, the performance of temporal KGE is relatively insensitive to $\lambda$ and $\lambda_i$ when the KGE model is TComplEx.
However, the performance of TRESCAL drops quickly when these regularization coefficients are large. It suggests that we should choose hyperparameters for TRESCAL carefully.
}

\section{Conclusion}
We propose a widely applicable and effective regularizer---namely, DURA---for semantic matching knowledge graph embedding models. The formulation of DURA is based on the observation that, for an existing semantic matching KGE model (primal), there is often another distance based KGE model (dual) closely associated with it. Extensive experiments show that DURA yields consistent and significant improvements on both static and temporal knowledge graph embedding benchmark datasets. 

\section*{Acknowledgment}
This work was supported in part by National Science Foundations of China grants 61822604, U19B2026, 61836006, and 62021001, and the Fundamental Research Funds for the Central Universities grant WK3490000004.

\begin{figure*}[!ht]
    \centering 
\begin{subfigure}{0.55\columnwidth}
  \includegraphics[width=130pt]{imgs//exp_DURA/CP_rank.pdf}
  \caption{Rank for CP.}
\end{subfigure}\hfil 
\medskip
\begin{subfigure}{0.55\columnwidth}
  \includegraphics[width=130pt]{imgs//exp_DURA/ComplEx_rank.pdf}
  \caption{Rank  for ComplEx.}
\end{subfigure}\hfil 
\medskip
\begin{subfigure}{0.55\columnwidth}
  \includegraphics[width=130pt]{imgs//exp_DURA/RESCAL_rank.pdf}
  \caption{Rank for RESCAL.}
\end{subfigure}\hfil 

\begin{subfigure}{0.55\columnwidth}
  \includegraphics[width=130pt]{imgs//exp_DURA/CP_batchsize.pdf}
  \caption{Batch size for CP.}
\end{subfigure}\hfil 
\medskip
\begin{subfigure}{0.55\columnwidth}
  \includegraphics[width=130pt]{imgs//exp_DURA/ComplEx_batchsize.pdf}
  \caption{Batch size for ComplEx.}
\end{subfigure}\hfil 
\medskip
\begin{subfigure}{0.55\columnwidth}
  \includegraphics[width=130pt]{imgs//exp_DURA/RESCAL_batchsize.pdf}
  \caption{Batch size for RESCAL.}
\end{subfigure}\hfil 

\begin{subfigure}{0.55\columnwidth}
  \includegraphics[width=130pt]{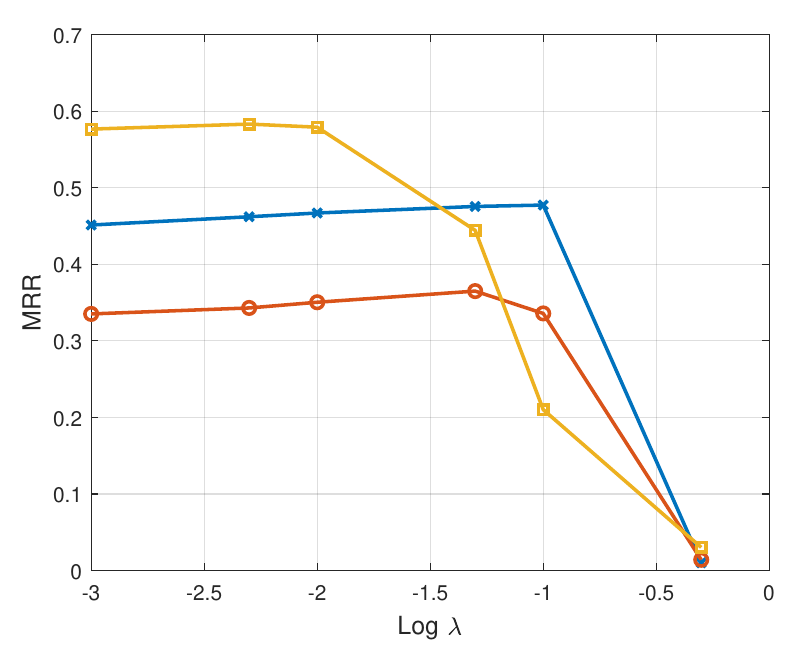}
  \caption{$\lambda$ for CP.}
\end{subfigure}\hfil 
\medskip
\begin{subfigure}{0.55\columnwidth}
  \includegraphics[width=130pt]{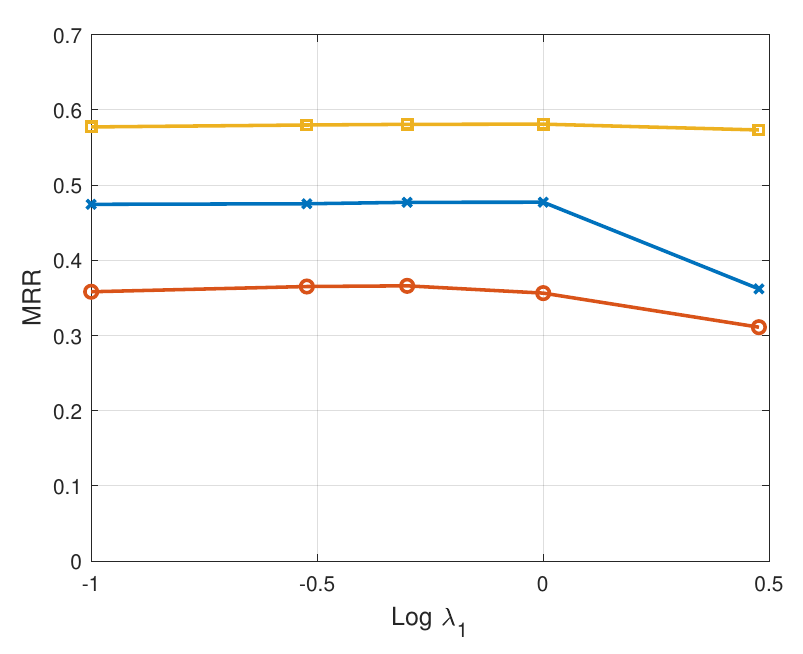}
  \caption{$\lambda_1$ for CP.}
\end{subfigure}\hfil 
\medskip
\begin{subfigure}{0.55\columnwidth}
  \includegraphics[width=130pt]{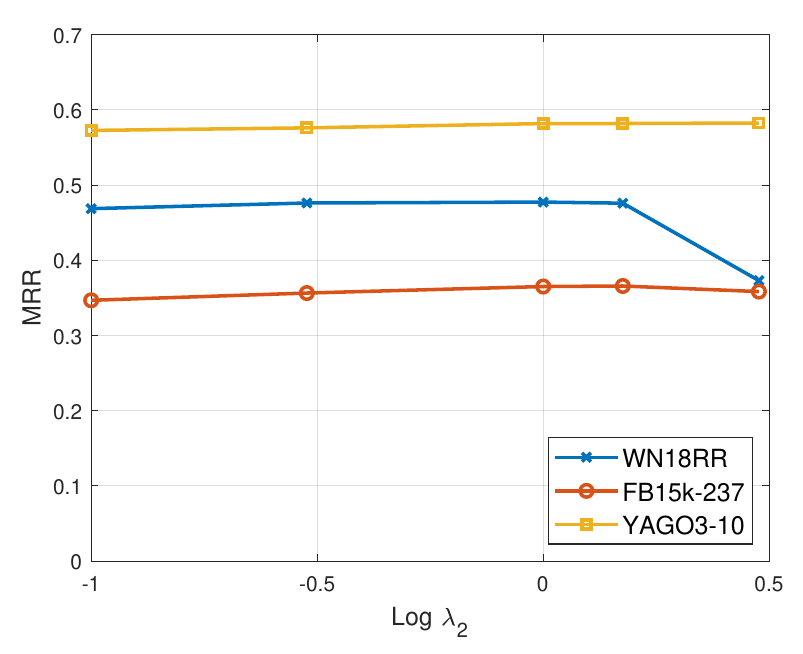}
  \caption{$\lambda_2$ for CP.}
\end{subfigure}\hfil 

\begin{subfigure}{0.55\columnwidth}
  \includegraphics[width=130pt]{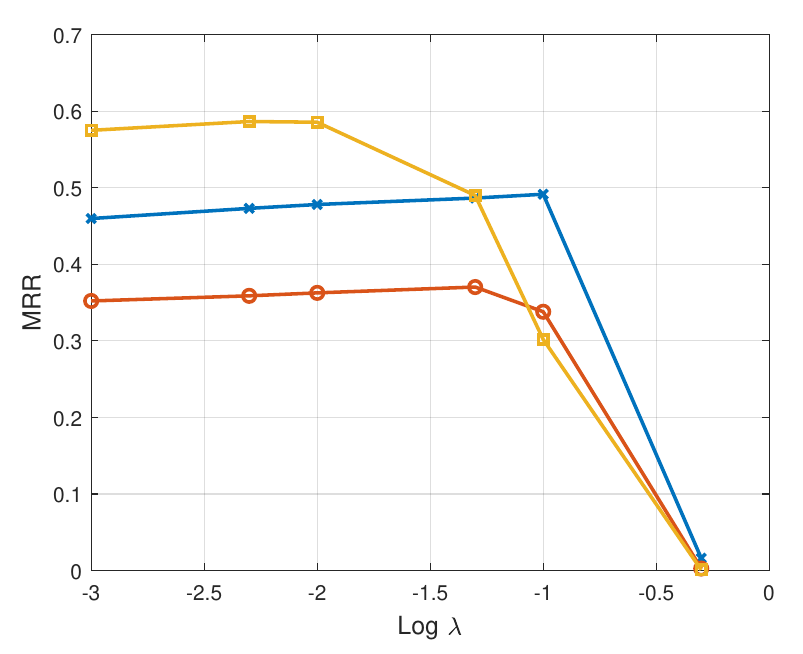}
  \caption{$\lambda$ for ComplEx.}
\end{subfigure}\hfil 
\medskip
\begin{subfigure}{0.55\columnwidth}
  \includegraphics[width=130pt]{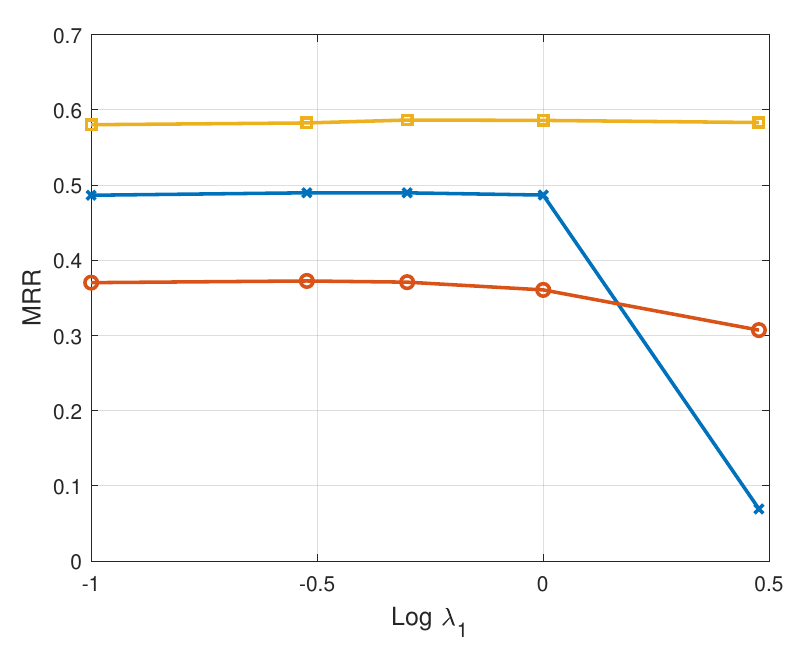}
  \caption{$\lambda_1$ for ComplEx.}
\end{subfigure}\hfil 
\medskip
\begin{subfigure}{0.55\columnwidth}
  \includegraphics[width=130pt]{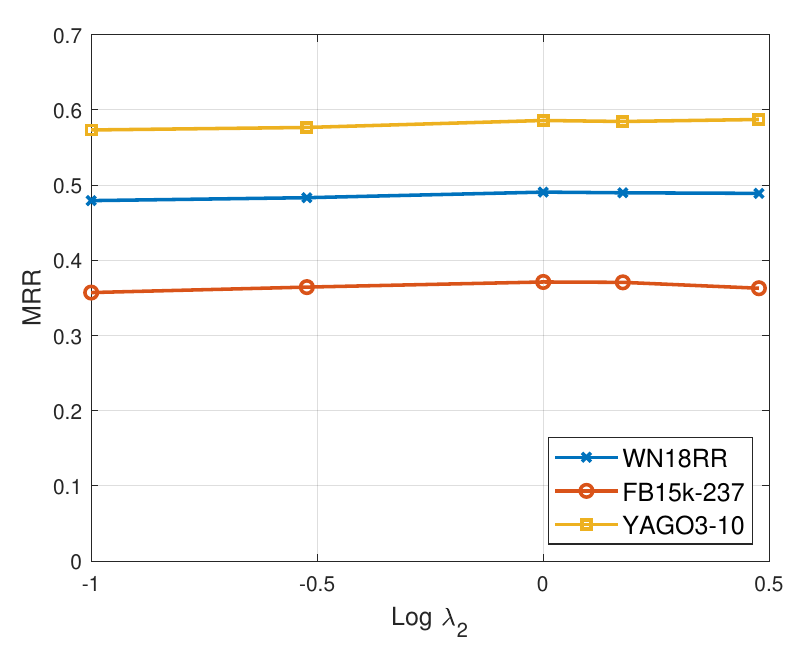}
  \caption{$\lambda_2$ for ComplEx.}
\end{subfigure}\hfil 

\begin{subfigure}{0.55\columnwidth}
  \includegraphics[width=130pt]{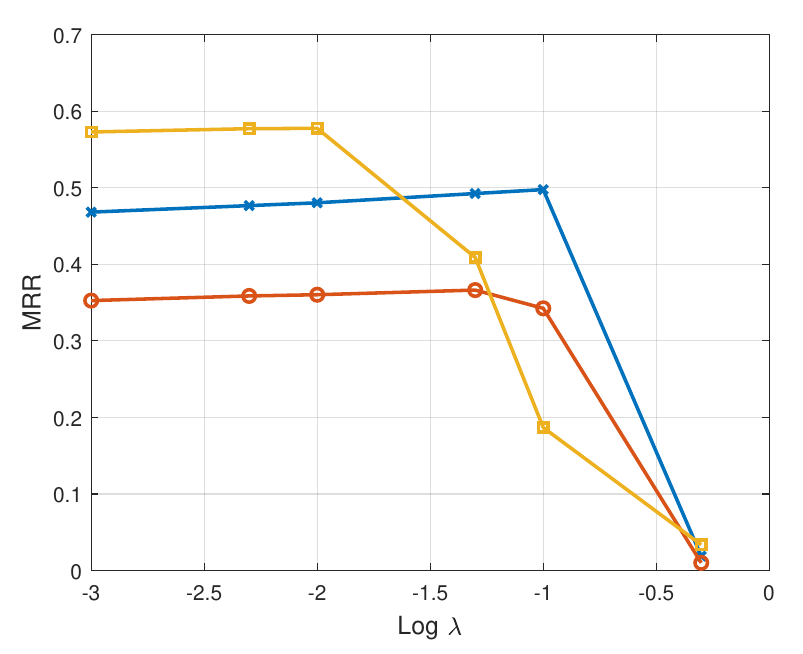}
  \caption{$\lambda$ for RESCAL.}
\end{subfigure}\hfil 
\medskip
\begin{subfigure}{0.55\columnwidth}
  \includegraphics[width=130pt]{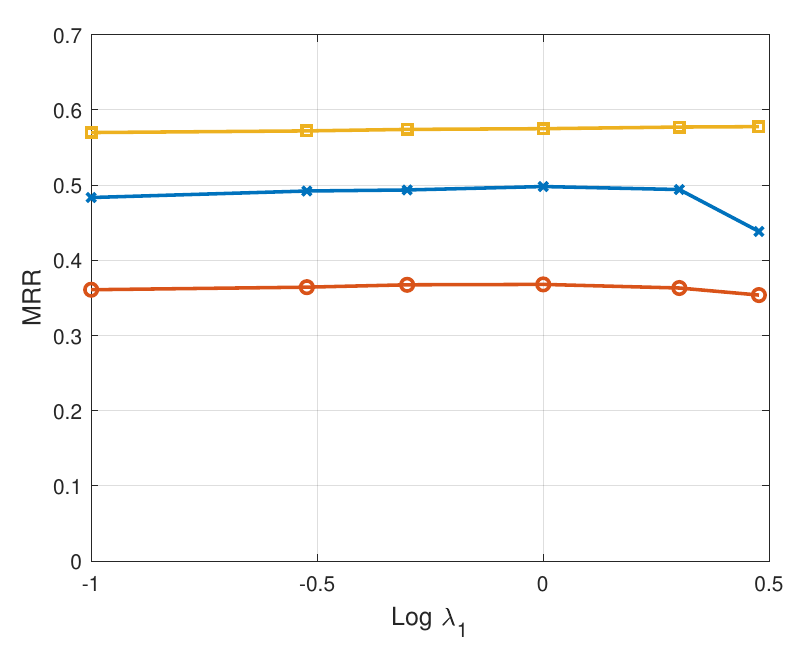}
  \caption{$\lambda_1$ for RESCAL.}
\end{subfigure}\hfil 
\medskip
\begin{subfigure}{0.55\columnwidth}
  \includegraphics[width=130pt]{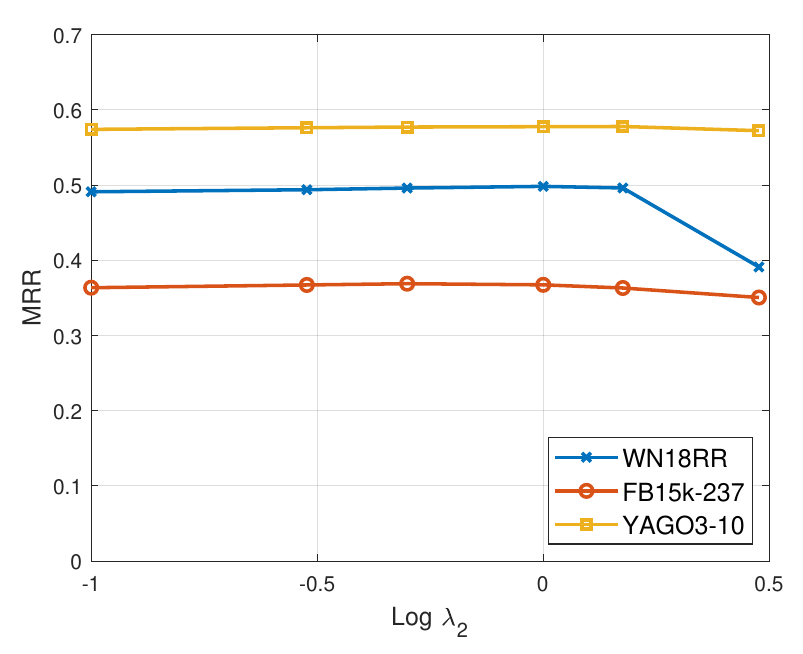}
  \caption{$\lambda_2$ for RESCAL.}
\end{subfigure}\hfil 
\caption{
\modifyy{Sensitivity to hyper-parameters of DURA. The x-axis is the logarithm of each hyper-parameter based on 10 and the y-axis is mean reciprocal rank (MRR) on test data.}}
\label{fig:sens_dura}
\end{figure*}

\bibliographystyle{IEEEtran}
\bibliography{kge}

\begin{IEEEbiography}[{\includegraphics[width=1in,height=1.25in,clip,keepaspectratio]{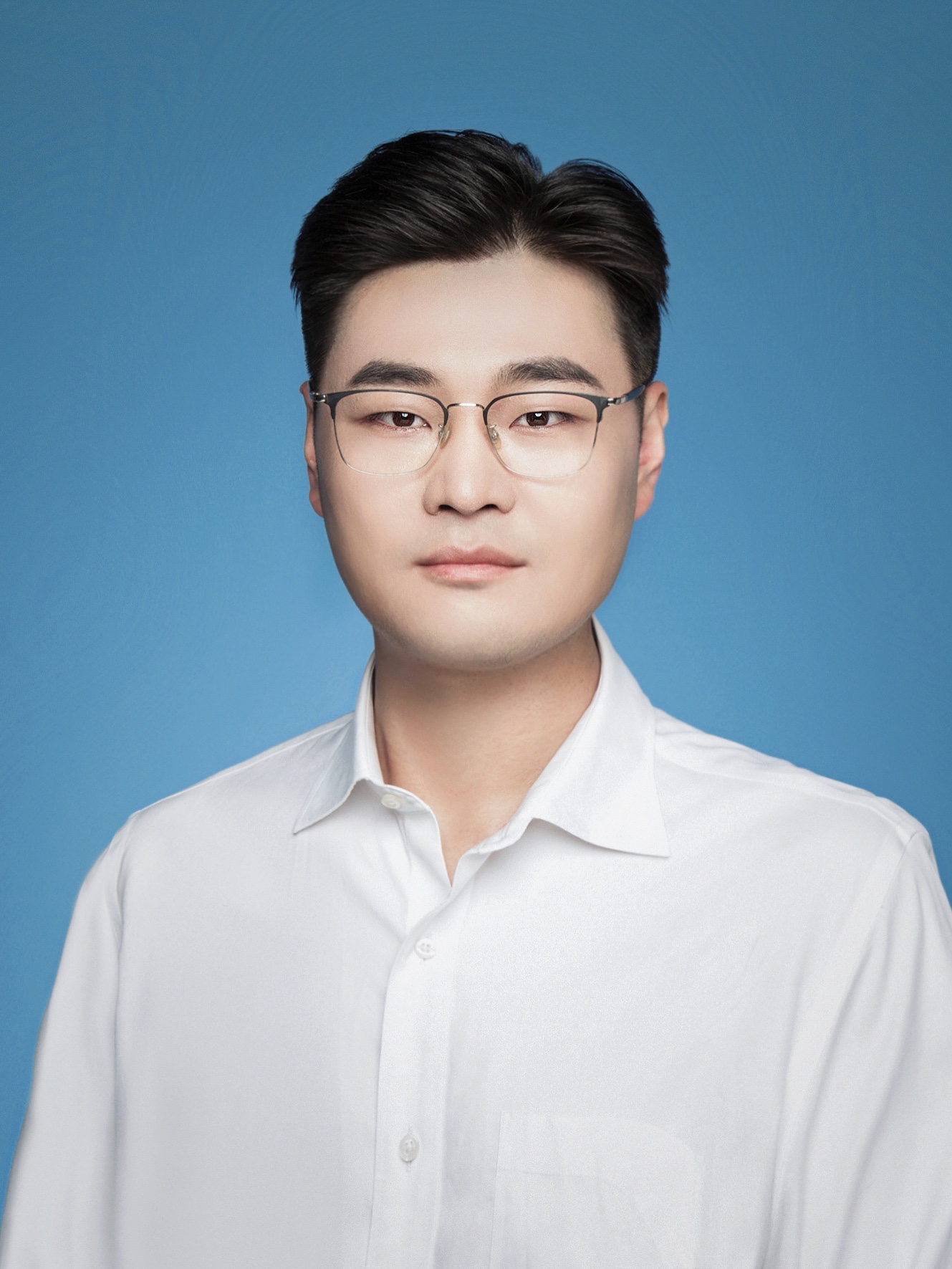}}]{Jie Wang}
  received the B.Sc. degree in electronic information science and technology from University of Science and Technology of China, Hefei, China, in 2005, and the Ph.D. degree in computational science from the Florida \mbox{State} University, Tallahassee, FL, in 2011. He is currently a professor in the Department of Electronic Engineering and Information Science at University of Science and Technology of China, Hefei, China. His research interests include reinforcement learning, knowledge graph, large-scale optimization, deep learning, etc.  He is a senior member of IEEE.
\end{IEEEbiography}

\begin{IEEEbiography}[{\includegraphics[width=1in,height=1.25in,clip,keepaspectratio]{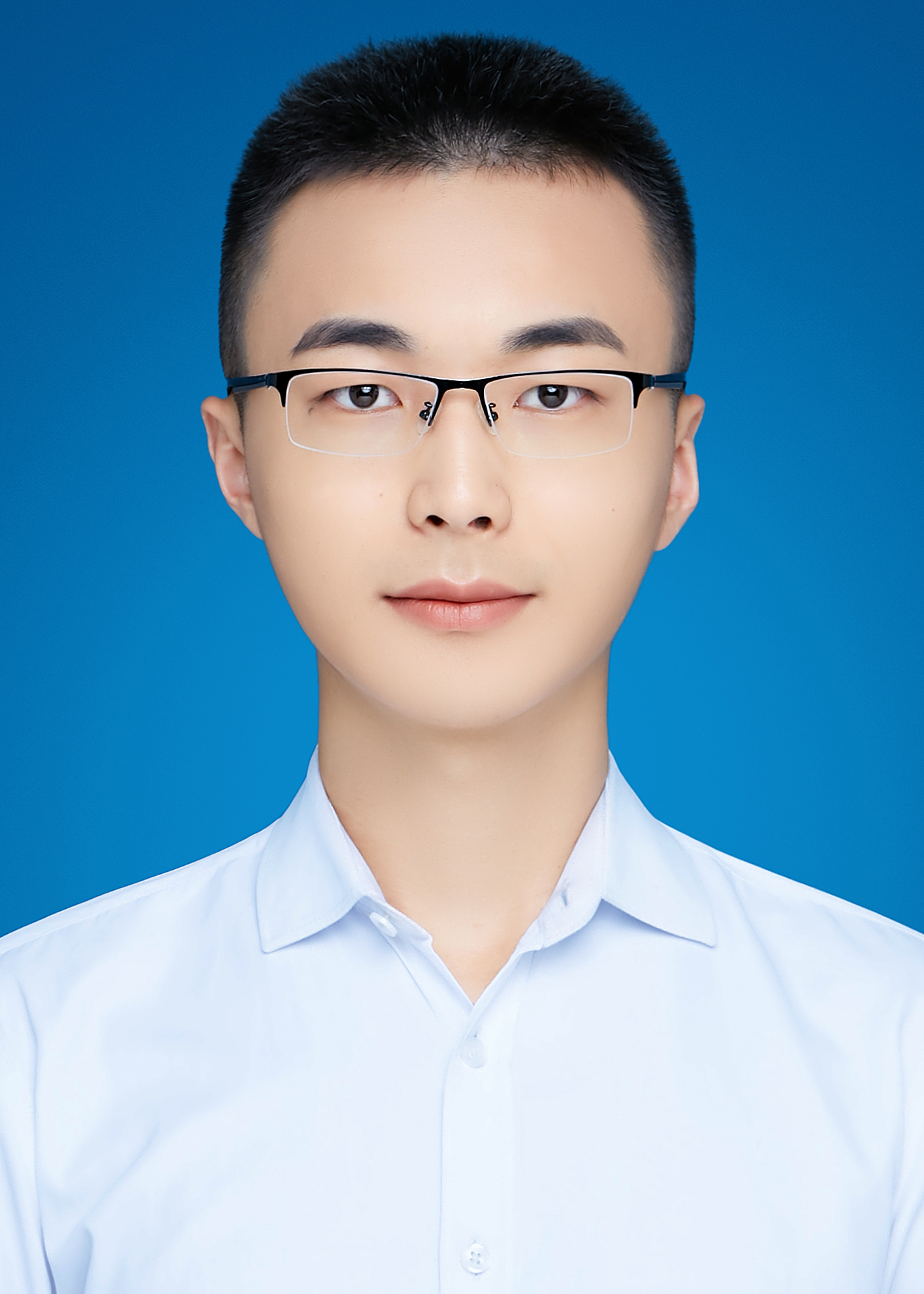}}]{Zhanqiu Zhang}
  received the B.Sc. degree in information and computing science from University of Science and Technology of China, Hefei, China, in 2018. He is currently a Ph.D. candidate in the Department of Electronic Engineering and Information Science at University of Science and Technology of China, Hefei, China. His research interests include graph representation learning and natural language processing.
\end{IEEEbiography}

\begin{IEEEbiography}[{\includegraphics[width=1in,height=1.25in,clip,keepaspectratio]{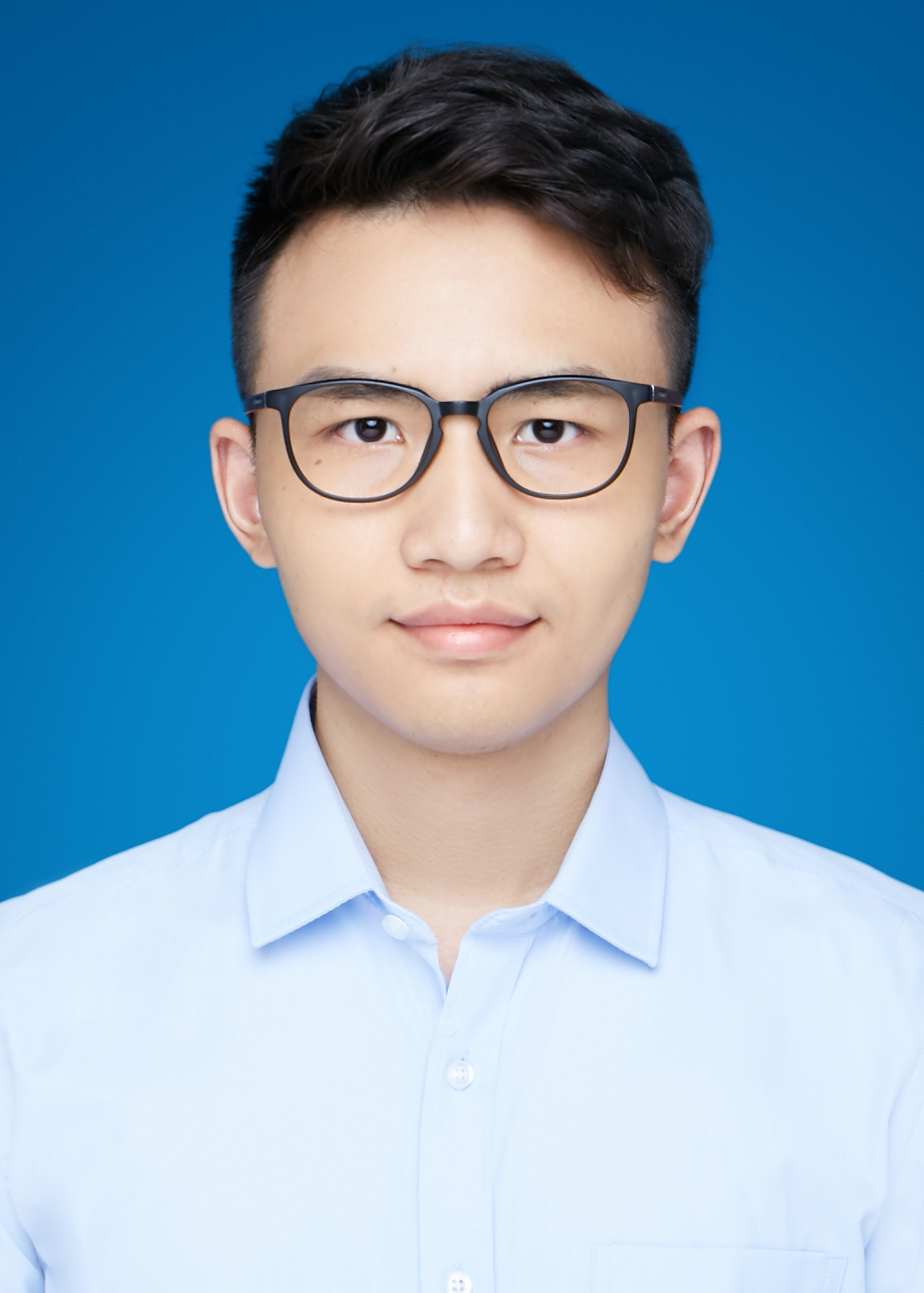}}]{Zhihao Shi}
  received the B.Sc. degree in Department of Electronic Engineering and Information Science, Hefei, China, in 2020. He is currently a graduate student in the Department of Electronic Engineering and Information Science at University of Science and Technology of China, Hefei, China. His research interests include graph representation learning and reinforcement learning.
\end{IEEEbiography}

\begin{IEEEbiography}[{\includegraphics[width=1in,height=1.25in,clip,keepaspectratio]{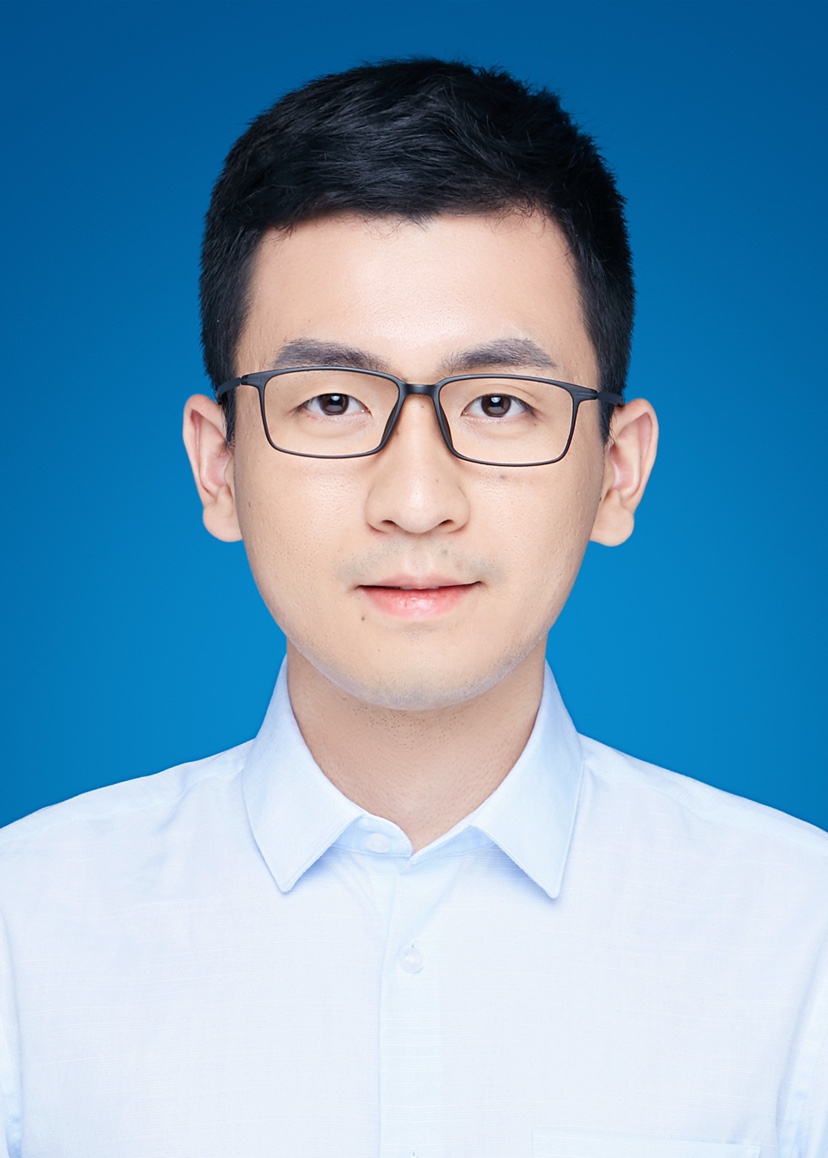}}]{Jianyu Cai}
  received the B.Sc. degree in software engineering from Southeast University, Nanjing, China, in 2019. He is currently a graduate student in the Department of Electronic Engineering and Information Science at University of Science and Technology of China, Hefei, China. His research interests include knowledge graph and natural language processing.
\end{IEEEbiography}

\begin{IEEEbiography}[{\includegraphics[width=1in,height=1.25in,clip,keepaspectratio]{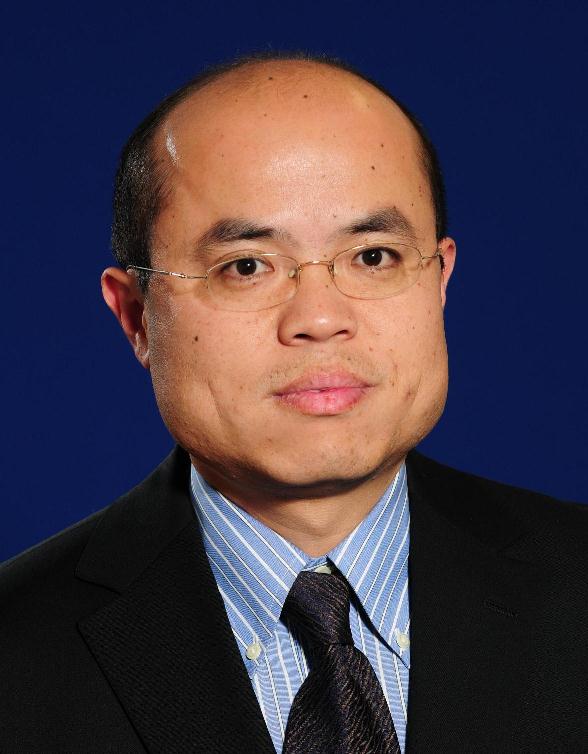}}]{Shuiwang Ji}
  received the PhD degree in computer science from Arizona State University, Tempe, Arizona, in 2010. Currently, he is an Associate Professor in the Department of Computer Science and Engineering, Texas A\&M University, College Station, Texas. His research interests include machine learning, deep learning, data mining, and computational biology. He received the National Science Foundation CAREER Award in 2014. He is currently an Associate Editor for IEEE Transactions on Pattern Analysis and Machine Intelligence, ACM Transactions on Knowledge Discovery from Data, and ACM Computing Surveys. He regularly serves as an Area Chair or equivalent roles for data mining and machine learning conferences, including AAAI, ICLR, ICML, IJCAI, KDD, and NeurIPS. He is a senior member of IEEE.
\end{IEEEbiography}

\begin{IEEEbiography}[{\includegraphics[width=1in,height=1.25in,clip,keepaspectratio]{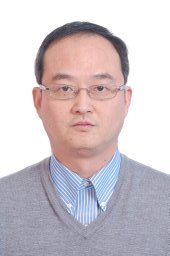}}]{Feng Wu}
received the B.S. degree in electrical engineering from Xidian University in 1992, and the M.S. and Ph.D. degrees in computer science from the Harbin Institute of Technology in 1996 and 1999, respectively. He is currently a Professor with the University of Science and Technology of China, where he is also the Dean of the School of Information Science and Technology. Before that, he was a Principal Researcher and the Research Manager with Microsoft Research Asia. His research interests include image and video compression, media communication, and media analysis and synthesis. He has authored or coauthored over 200 high quality articles (including several dozens of IEEE Transaction papers) and top conference papers on MOBICOM, SIGIR, CVPR, and ACM MM. He has 77 granted U.S. patents. His 15 techniques have been adopted into international video coding standards. As a coauthor, he received the Best Paper Award at 2009 IEEE Transactions on Circuits and Systems for Video Technology, PCM 2008, and SPIE VCIP 2007. He also received the Best Associate Editor Award from IEEE Circuits and Systems Society in 2012. He also serves as the TPC Chair for MMSP 2011, VCIP 2010, and PCM 2009, and the Special Sessions Chair for ICME 2010 and ISCAS 2013. He serves as an Associate Editor for IEEE Transactions on Circuits and Systems for Video Technology, IEEE Transactions ON Multimedia, and several other international journals.
\end{IEEEbiography}

	\appendices
\section{Proof of Theorem 1}
\begin{proof}
	We have that
	\begin{align*}
		&\sum_{j=1}^{|\mathcal{R}|}\left(\|\headmat\textbf{R}_j\|_F^2+\|\tailmat\|_F^2\right)\\
		=&\sum_{j=1}^{|\mathcal{R}|}\left(\sum_{i=1}^I\|\heademb_{i}\circ \textbf{r}_{j}\|_F^2+\sum_{d=1}^D\|\tailemb_{:d}\|_F^2\right)\\
		=&\sum_{j=1}^{|\mathcal{R}|}\left(\sum_{d=1}^D\|\tailemb_{:d}\|_F^2+\sum_{i=1}^I\sum_{d=1}^D\heademb_{id}^2\textbf{r}_{jd}^2\right)\\
		=&\sum_{j=1}^{|\mathcal{R}|}\sum_{d=1}^D\|\tailemb_{:d}\|_2^2+\sum_{d=1}^D\|\heademb_{:d}\|_2^2\|\textbf{r}_{:d}\|_2^2\\
		=&\sum_{d=1}^D(\|\heademb_{:d}\|_2^2\|\textbf{r}_{:d}\|_2^2+|\mathcal{R}|\|\tailemb_{:d}\|_2^2)\\
		\ge& \sum_{d=1}^D2\sqrt{|\mathcal{R}|}\|\heademb_{:d}\|_2\|\textbf{r}_{:d}\|_2\|\tailemb_{:d}\|_2\\
		=&2\sqrt{|\mathcal{R}|}\sum_{d=1}^D\|\heademb_{:d}\|_2\|\textbf{r}_{:d}\|_2\|\tailemb_{:d}\|_2.
	\end{align*}
	The equality holds if and only if $\|\heademb_{:d}\|_2^2\|\textbf{r}_{:d}\|_2^2=|\mathcal{R}|\|\tailemb_{:d}\|_2^2$, i.e., $\|\heademb_{:d}\|_2\|\textbf{r}_{:d}\|_2=\sqrt{|\mathcal{R}|}\|\tailemb_{:d}\|_2$.
	
	For all CP decomposition $\hat{\mathcal{X}}=\sum_{d=1}^D \heademb_{:d} \otimes \textbf{r}_{:d} \otimes \tailemb_{:d}$,  we can always let $\heademb'_{:d}=\heademb_{:d}$, $\textbf{r}'_{:d}=\sqrt{\frac{\|\tailemb_d\|_2\sqrt{|\mathcal{R}|}}{\|\heademb_{:d}\|_2\|\textbf{r}_{:d}\|_2}}\textbf{r}_{:d}$ and $\tailemb'_{:d}= \sqrt{\frac{\|\heademb_{:d}\|_2\|\textbf{r}_{:d}\|_2}{\|\tailemb_{:d}\|_2\sqrt{|\mathcal{R}|}}}\tailemb_{:d}$ such that
	$$\|\heademb'_{:d}\|_2\|\textbf{r}'_{:d}\|_2=\sqrt{|\mathcal{R}|}\|\tailemb'_{:d}\|_2,$$
	and meanwhile ensure that $\hat{\mathcal{X}}=\sum_{d=1}^D \heademb'_{:d} \otimes \textbf{r}'_{:d} \otimes \tailemb'_{:d}$. Therefore, we know that
	\begin{align*}
		\frac{1}{\sqrt{|\mathcal{R}|}}\sum_{j=1}^{|\mathcal{R}|}\|\hat{\mathcal{X}}_j\|_*&=\frac{1}{2\sqrt{|\mathcal{R}|}}\sum_{j=1}^{|\mathcal{R}|}\min_{\hat{\mathcal{X}}_j=\headmat\textbf{R}_j\tailmat^\top}(\|\headmat\textbf{R}_j\|_F^2+\|\tailmat\|_F^2)\\
		&\le\frac{1}{2\sqrt{|\mathcal{R}|}}\min_{\hat{\mathcal{X}}_j=\headmat\textbf{R}_j\tailmat^\top}\sum_{j=1}^{|\mathcal{R}|}(\|\headmat\textbf{R}_j\|_F^2+\|\tailmat\|_F^2)\\
		&= \min_{\hat{\mathcal{X}}=\sum_{d=1}^D \heademb_{:d}\otimes \textbf{r}_{:d}\otimes \tailemb_{:d}}\sum_{d=1}^D\|\heademb_{:d}\|_2\|\textbf{r}_{:d}\|_2\|\tailemb_{:d}\|_2\\
		&=\|\hat{\mathcal{X}}\|_*.
	\end{align*}
	In the same manner, we know that
	\begin{align*}
		\frac{1}{2\sqrt{|\mathcal{R}|}}\min_{\hat{\mathcal{X}}_j=\headmat\textbf{R}_j\tailmat^\top}\sum_{j=1}^{|\mathcal{R}|}(\|\tailmat\textbf{R}_j^\top\|_F^2+\|\headmat\|_F^2)=\|\hat{\mathcal{X}}\|_*.
	\end{align*}
	The minimizers satisfy $ \|\tailemb_{:d}\|_2\|\textbf{r}_{:d}\|_2=\sqrt{|\mathcal{R}|}\|\heademb_{:d}\|_2$.
	
	Therefore, the minimizers satisfy $
	\|\heademb_{:d}\|_2\|\textbf{r}_{:d}\|_2=\sqrt{|\mathcal{R}|}\|\tailemb_{:d}\|_2
	$
	and
	$
	\|\tailemb_{:d}\|_2\|\textbf{r}_{:d}\|_2=\sqrt{|\mathcal{R}|}\|\heademb_{:d}\|_2,
	$
	$\,\forall\,d\in\{1,2,\ldots, D\}$.
\end{proof}


\section{Proof of Theorem 2}
\begin{proof}
	First, we have
	\begin{align*}
		&\sum_{l=1}^{|\mathcal{T}|}\sum_{j=1}^{|\mathcal{R}|}\|\textbf{U}\textbf{R}_j\textbf{T}_l\|_F^2+\|\textbf{V}\|_F^2 \\
		=&\sum_{d=1}^r \|\textbf{u}_{:d}\|_2^2\|\textbf{r}_{:d}\|_2^2\|\textbf{t}_{:d}\|_2^2 + |\mathcal{R}||\mathcal{T}| \|\textbf{v}_{:d}\|_2^2,\\
		&\sum_{l=1}^{|\mathcal{T}|}\sum_{j=1}^{|\mathcal{R}|}\|\textbf{R}_j\textbf{T}_l\textbf{V}^\top\|_F^2+\|\textbf{U}\|_F^2 \\
		=& \sum_{d=1}^r \|\textbf{v}_{:d}\|_2^2\|\textbf{r}_{:d}\|_2^2\|\textbf{t}_{:d}\|_2^2 + |\mathcal{R}||\mathcal{T}| \|\textbf{u}_{:d}\|_2^2.
	\end{align*}
	Suppose that $\|\hat{\mathcal{X}}\|_*$ can be attained at $D$ and $(\textbf{u}_{:d}^*,\textbf{r}_{:d}^*,\textbf{v}_{:d}^*,\textbf{t}_{:d}^*)$ for $d=1,...,D$. Let
	\begin{align*}
		\textbf{u}_{:d}' &= \alpha_u \textbf{u}_{:d}^* = \left(\sqrt{\frac{\|\textbf{v}_{:d}^*\|_2\|\textbf{r}_{:d}^*\|_2\|\textbf{t}_{:d}^*\|_2}{\sqrt{|\mathcal{R}||\mathcal{T}|}\|\textbf{u}_{:d}^*\|_2}}\right) \textbf{u}_{:d}^*\\
		\textbf{v}_{:d}' &= \alpha_v \textbf{v}_{:d}^* = \left(\sqrt{\frac{\|\textbf{u}_{:d}^*\|_2\|\textbf{r}_{:d}^*\|_2\|\textbf{t}_{:d}^*\|_2}{\sqrt{|\mathcal{R}||\mathcal{T}|}\|\textbf{v}_{:d}^*\|_2}}\right) \textbf{v}_{:d}^*\\
		\textbf{r}_{:d}' &= \alpha_r \textbf{r}_{:d}^* = \left(\sqrt{\frac{\sqrt{|\mathcal{R}||\mathcal{T}|}}{\|\textbf{r}_{:d}^*\|_2\|\textbf{t}_{:d}^*\|_2}}\right) \textbf{r}_{:d}^*\\
		\textbf{t}_{:d}' &= \alpha_t \textbf{t}_{:d}^* = \left(\sqrt{\frac{\sqrt{|\mathcal{R}||\mathcal{T}|}}{\|\textbf{r}_{:d}^*\|_2\|\textbf{t}_{:d}^*\|_2}}\right) \textbf{t}_{:d}^*.
	\end{align*}
	Notice that
	\begin{align*}
		\hat{\mathcal{X}} &= \sum_{d=1}^D\textbf{u}_{:d}'\otimes \textbf{r}_{:d}' \otimes \textbf{v}_{:d}' \otimes \textbf{t}_{:d}',\\
		\|\hat{\mathcal{X}}\|_{*} &= \sum_{d=1}^D\|\textbf{u}_{:d}'\|_2\|\textbf{r}_{:d}'\|_2\|\textbf{v}_{:d}'\|_2\|\textbf{t}_{:d}'\|_2,
	\end{align*}
	i.e., $(\textbf{u}_{:d}^*,\textbf{r}_{:d}^*,\textbf{v}_{:d}^*,\textbf{t}_{:d}^*)$ for $d=1,...,D$ also attains the nuclear $p$-norm $\|\hat{\mathcal{X}}\|_{*}$.

	It follows from the inequality of arithmetic and geometric means that
	\begin{align*}
		&\frac{1}{4}\sum_{d=1}^r (\|\textbf{u}_{:d}\|_2^2\|\textbf{r}_{:d}\|_2^2\|\textbf{t}_{:d}\|_2^2 +|\mathcal{R}||\mathcal{T}| \|\textbf{v}_{:d}\|_2^2) \\
		+&(\|\textbf{v}_{:d}\|_2^2\|\textbf{r}_{:d}\|_2^2\|\textbf{t}_{:d}\|_2^2 +|\mathcal{R}||\mathcal{T}| \|\textbf{u}_{:d}\|_2^2)\\
		\geq& \sum_{d=1}^r \sqrt{|\mathcal{R}||\mathcal{T}|} \|\textbf{u}_{:d}\|_2\|\textbf{r}_{:d}\|_2\|\textbf{v}_{:d}\|_2\|\textbf{t}_{:d}\|_2 \\
		\geq& \sum_{d=1}^D \sqrt{|\mathcal{R}||\mathcal{T}|} \|\textbf{u}_{:d}'\|_2\|\textbf{r}_{:d}'\|_2\|\textbf{v}_{:d}'\|_2\|\textbf{t}_{:d}'\|_2\\
		=&\sqrt{|\mathcal{R}||\mathcal{T}|} \|\hat{\mathcal{X}}\|_*
	\end{align*}
	The equality holds if $(\textbf{u}_{:d},\textbf{r}_{:d},\textbf{v}_{:d},\textbf{t}_{:d}) = (\textbf{u}_{:d}',\textbf{r}_{:d}',\textbf{v}_{:d}',\textbf{t}_{:d}')$, i.e.,
	\begin{align*}
		&\frac{1}{4\sqrt{|\mathcal{R}||\mathcal{T}|}}\sum_{d=1}^D (\|\textbf{u}_{:d}'\|_2^2\|\textbf{r}_{:d}'\|_2^2\|\textbf{t}_{:d}'\|_2^2 + |\mathcal{R}||\mathcal{T}|\|\textbf{v}_{:d}'\|_2^2) \\
		+&(\|\textbf{v}_{:d}'\|_2^2\|\textbf{r}_{:d}'\|_2^2\|\textbf{t}_{:d}'\|_2^2 + |\mathcal{R}||\mathcal{T}|\|\textbf{u}_{:d}'\|_2^2)\\
		=&\sum_{d=1}^D \|\textbf{u}_{:d}'\|_2\|\textbf{r}_{:d}'\|_2\|\textbf{v}_{:d}'\|_2\|\textbf{t}_{:d}'\|_2=\|\hat{\mathcal{X}}\|_*.
	\end{align*}
	Therefore, the minimization of the left side is $\|\hat{\mathcal{X}}\|_*$.
\end{proof}

\section{Proof of Theorem 3}
\begin{lemma}
	The nuclear $p$-norm is also derived by
	\begin{align}
		\| \mathcal{A} \|_{p*} = \min&\left\{
		\frac{1}{D}\sum_{i=1}^r\sum_{k=1}^D\|\textbf{u}_{k,i}\|_p^D:\right.\nonumber\\
		&\left.\mathcal{A}=\sum_{i=1}^r\textbf{u}_{1,i}\otimes \cdots \otimes \textbf{u}_{D,i},r\in\mathbb{N}
		\right\},\label{equ:nunorm}
	\end{align}
	where $\textbf{u}_{k,i} \in \mathbb{R}^{n_k}$ for $k=1, ..., D$, $i=1,...,r$, and $\otimes$ denotes the outer product.
\end{lemma}
\begin{proof}
	Suppose that the right side of Equation \eqref{equ:nunorm} attains $\| \mathcal{A} \|_{p*}$ at $R$ and $(\textbf{u}_{k,i}^*)$ for $k=1, ..., D$, $i=1,...,R$. Let
	\begin{align*}
		\textbf{u}_{k,i}'&=\alpha_{k,i}\textbf{u}^*_{k,i},\\
		\alpha_{k,i}&=\frac{\sqrt[D]{\prod_{d=1}^D\|\textbf{u}_{d,i}^*\|_p}}{\|\textbf{u}^*_{k,i}\|}.
	\end{align*}
	Notice that
	\begin{align*}
		\mathcal{A} = \sum_{i=1}^R\textbf{u}_{1,i}'\otimes \cdots \otimes \textbf{u}_{D,i}',\\
		\sum_{i=1}^R\prod_{k=1}^D\|\textbf{u}_{k,i}'\|_p = \|\mathcal{A}\|_{p*},
	\end{align*}
	i.e., $(\textbf{u}_{k,i}')$ for $k=1, ..., D$, $i=1,...,R$ also attains the nuclear $p$-norm $\|\mathcal{A}\|_{p*}$.
	
	It follows from the inequality of arithmetic and geometric means that
	\begin{align*}
		\frac{1}{D}\sum_{i=1}^r\sum_{k=1}^D\|\textbf{u}_{k,i}\|_p^D &= \sum_{i=1}^r \left( \frac{1}{D} \sum_{k=1}^D \|\textbf{u}_{k,i}\|_p^D\right) \\
		&\geq \sum_{i=1}^r \left(\prod_{k=1}^D\|\textbf{u}_{k,i}\|_p\right)\\
		&\geq \sum_{i=1}^R \prod_{k=1}^D\|\textbf{u}_{k,i}'\|_p\\
		&=\|\mathcal{A}\|_{p*}
	\end{align*}
	where $\mathcal{A}=\sum_{i=1}^r\textbf{u}_{1,i}\otimes \cdots \otimes \textbf{u}_{D,i}$. As the left side attains the minimization $\|\mathcal{A}\|_{p*}$ at $R$ and $(\textbf{u}_{k,i}')$ for $k=1, ..., D$,  $i=1,...,R$, i.e.,
	\begin{align*}
		\frac{1}{D}\sum_{i=1}^r\sum_{k=1}^D\|\textbf{u}_{k,i}'\|_p^D = \sum_{i=1}^R \prod_{k=1}^D\|\textbf{u}_{k,i}'\|_p = \|\mathcal{A}\|_{p*}.
	\end{align*}
	Therefore, the minimization of the left side is $\|\mathcal{A}\|_{p*}$.
\end{proof}


\begin{table}[ht]
	\caption{Statistics of KGC benchmark datasets.}
	\label{table:datasets}
	\centering
	\begin{tabular}{l *{3}{c}}
		\toprule
		& WN18RR & FB15k-237 & YAGO3-10 \\
		\midrule
		\#Entity & 40,943 & 14,541 & 123,182\\
		\#Relation & 11 & 237 & 37 \\
		\#Train & 86,835 & 272,115 & 1,079,040 \\
		\#Valid &3,034 &17,535 &5,000\\
		\#Test &3,134 &20,466 &5,000\\
		\bottomrule
	\end{tabular}
	
\end{table}

\begin{table}[ht]
	\caption{Statistics of temporal KGC benchmark datasets. }
	\label{table:temperal_datasets}
	\centering
	\begin{tabular}{l *{3}{c}}
		\toprule
		& ICEWS14 & ICEWS05-15 & YAGO15k \\
		\midrule
		\#Entity & 6,869 & 10,094 & 15,403\\
		\#Relation & 230 & 251 & 51 \\
		\#Timestamp & 365 & 4017 & 170\\
		\#Train & 72,826 & 368,962 & 110,441 \\
		\#Valid & 8,941 & 46,275 & 13,815\\
		\#Test & 8,963 & 46,092 & 13,800\\
		\bottomrule
	\end{tabular}
\end{table}

\begin{table*}[ht]
	\centering
	\caption{The queries in T-SNE visualizations.}
	\label{table:query_name}
	\begin{tabular}{c l}
		\toprule
		\textbf{Index} & \textbf{Query} \\
		\midrule
		1 & (political drama, \text{/media\_common/netflix\_genre/titles}, ?)\\
		2 & (Academy Award for Best Original Song, \text{/award/award\_category/winners./award/award\_honor/ceremony},?)\\
		3 & (Germany, \text{/location/location/contains},?) \\
		4 & (doctoral degree ,  \text{/education/educational\_degree/people\_with\_this\_degree./education/education/major\_field\_of\_study},?)\\
		5 & (broccoli, \text{/food/food/nutrients./food/nutrition\_fact/nutrient},?)\\
		6 & (shooting sport, \text{/olympics/olympic\_sport/athletes./olympics/olympic\_athlete\_affiliation/country},?) \\
		7 & (synthpop, \text{/music/genre/artists}, ?)\\
		8 & (Italian American, \text{/people/ethnicity/people},?)\\
		9 & (organ, \text{/music/performance\_role/track\_performances./music/track\_contribution/role}, ?)\\
		10 & (funk, \text{/music/genre/artists}, ?)\\
		\bottomrule
	\end{tabular}
\end{table*}
\section{Dataset Statistics}
For static knowledge graph completion, we consider three public static knowledge graph datasets---WN18RR \cite{wn18rr}, FB15k-237 \cite{conve}, and YAGO3-10 \cite{yago3}, which have been divided into training, validation, and testing set in previous works. The statistics of these datasets are shown in Table \ref{table:datasets}.

For temporal knowledge graph completion, we consider three public temperal knowledge graph datasets---ICEWS14 \cite{icew}, ICEWS05-15 \cite{icew}, and YAGO15k \cite{yago15k}, which have been divided into training, validation, and testing set in previous works. ICEWS14 and ICEWS05-15 are extracted from the Integrated Conflict Early Warning System (ICEWS).
YAGO15k is a modification of FB15k. The statistics of these datasets are shown in Table \ref{table:temperal_datasets}.

\section{Evaluation Metrics}
The mean reciprocal rank is the average of the reciprocal ranks of results for a sample of queries Q:
\begin{align*}
	\text{MRR}=\frac{1}{|Q|}\sum_{i=1}^{|Q|}\frac{1}{\text{rank}_i}.
\end{align*}
The Hits@N is the ratio of ranks that no more than $N$:
\begin{align*}
	\text{Hits@N}=\frac{1}{|Q|}\sum_{i=1}^{|Q|}\mathbbm{1}_{x\le N}(\text{rank}_i),
\end{align*}
where $\mathbbm{1}_{x\le N}(\text{rank}_i)=1$ if $\text{rank}_i\le N$ or otherwise $\mathbbm{1}_{x\le N}(\text{rank}_i)=0$.

\section{The optimal value of p in Static KGC}
In distance-based models, the commonly used $p$ is either $1$ or $2$. When $p=2$, DURA takes the form as the one in Equation (7) in the main text. If $p=1$, we cannot expand the squared score function of the associated DB models as in Equation (4).
Thus, the induced regularizer takes the form of $\sum_{(\head_i,r_j,\tail_k)\in\mathcal{S}}\|\heademb_i\bar{\textbf{R}}_j-\tailemb_k\|_1+\|\tailemb_k\textbf{R}_j^\top-\heademb_i\|_1$. The above regularizer with $p=1$ (Reg\_p1) does not gives an upper bound on the tensor nuclear-2 norm as in Theorem 1. Table \ref{table:cmp_results} shows that, DURA significantly outperforms Reg\_p1 on WN18RR and FB15k-237. Therefore, we choose $p=2$.

\begin{table}[ht]
	\centering
	\caption{Comparison to Reg\_p1, where "R" denotes RESCAL and "C" denotes ComplEx.  }    \label{table:cmp_results}
	\begin{tabular}{l  c c c  c c c }
		\toprule
		&\multicolumn{3}{c}{\textbf{WN18RR}}&  \multicolumn{3}{c}{\textbf{FB15k-237}} \\
		\cmidrule(lr){2-4}
		\cmidrule(lr){5-7}
		& MRR & H@1 & H@10 & MRR & H@1 & H@10 \\
		\midrule
		R-Reg\_p1   &.281 &.220 &.394 &.310 &.228 &.338\\
		C-Reg\_p1   &.409 &.393 &.439 &.316 &.229 &.487\\
		\midrule
		R-DURA    &\textbf{.498} &\textbf{.455} &.577 &.368 &\textbf{.276} &.550 \\
		C-DURA &.491 &.449 &.571 &\textbf{.371} &\textbf{.276} &\textbf{.560} \\
		\bottomrule
	\end{tabular}
\end{table}

\section{The queries in T-SNE visualization}
In Table \ref{table:query_name}, we list the ten queries used in the T-SNE visualization (Section 6.5 in the main text). Note that a query is represented as $(u, r, ?)$, where $u$ denotes the head entity and $r$ denotes the relation.
\end{document}